\documentclass[twoside,11pt]{article}

\usepackage{blindtext}

\usepackage{amsmath}
\usepackage{xcolor}
\usepackage{enumitem}
\usepackage{graphicx}
\usepackage{subcaption}
\usepackage{wrapfig}
\usepackage{algpseudocode}
\usepackage{algorithm}

\usepackage[abbrvbib]{jmlr2e}

\newcommand{\R}{\mathbb{R}}
\renewcommand{\S}{\sigma}
\renewcommand{\L}{\mathcal{L}}
\newcommand{\T}{\mathcal{T}}
\renewcommand{\a}{\alpha}
\renewcommand{\d}{\delta}

\newcommand{\op}{\texttt{opt}}
\newcommand{\opt}{\operatorname{opt}}

\newtheorem{assumption}{Assumption}

\DeclareMathOperator*{\argmin}{arg\,min}

\usepackage{lastpage}
\jmlrheading{27}{2026}{1-\pageref{LastPage}}{3/25; Revised 11/25}{1/26}{25-0549}{Addison Kristanto Julistiono, Davoud Ataee Tarzanagh, Navid Azizan} \ShortHeadings{Optimizing Attention with Mirror Descent}{Julistiono, Ataee Tarzanagh, and Azizan}

\firstpageno{1}

\begin{document}

\title{Optimizing Attention with Mirror Descent: Generalized Max-Margin Token Selection}

\renewcommand{\thefootnote}{\fnsymbol{footnote}}

\author{\name Addison Kristanto Julistiono\footnotemark[1] \email aaron25@mit.edu \\
\addr Massachusetts Institute of Technology
\AND
\name Davoud Ataee Tarzanagh\footnotemark[1] \email tarzanaq@upenn.edu \\
\addr University of Pennsylvania
\AND
\name Navid Azizan \email azizan@mit.edu \\
\addr Massachusetts Institute of Technology
}

\footnotetext[1]{Equal contribution.}

\renewcommand{\thefootnote}{\arabic{footnote}}

\editor{Mahdi Soltanolkotabi}

\maketitle

\begin{abstract}
Attention mechanisms have revolutionized several domains of artificial intelligence, such as natural language processing and computer vision, by enabling models to selectively focus on relevant parts of the input data. While recent work has characterized the optimization dynamics of gradient descent (GD) in attention-based models and the structural properties of its preferred solutions, less is known about more general optimization algorithms such as mirror descent (MD). In this paper, we investigate the convergence properties and implicit biases of a family of MD algorithms tailored for softmax attention mechanisms, with the potential function chosen as the $p$-th power of the $\ell_p$-norm. Specifically, we show that these algorithms converge in direction to a generalized hard-margin SVM with an $\ell_p$-norm objective when applied to a classification problem using a softmax attention model. Notably, our theoretical results reveal that the convergence rate is comparable to that of traditional GD in simpler models, despite the highly nonlinear and nonconvex nature of the present problem. Additionally, we delve into the joint optimization dynamics of the key-query matrix and the decoder, establishing conditions under which this complex joint optimization converges to their respective hard-margin SVM solutions. Lastly, our numerical experiments on real data demonstrate that MD algorithms improve generalization over standard GD and excel in optimal token selection.
\end{abstract}


\section{Introduction}
Attention mechanisms \citep{bahdanau2014neural} have transformed natural language processing (NLP) and large language models (LLMs). Initially developed for encoder-decoder recurrent neural networks (RNNs), attention enables the decoder to focus on relevant input segments rather than relying solely on a fixed-length hidden state. This approach became fundamental in transformers \citep{vaswani2017attention}, where attention layers—computing softmax similarities among input tokens—are the architecture's backbone. Transformers have driven rapid advancements in NLP with models like BERT \citep{devlin2018bert} and ChatGPT \citep{gpt4}, and have become the preferred architecture for generative modeling \citep{chen2021evaluating, ramesh2021zero}, computer vision \citep{dosovitskiy2021vit, radford2021learning}, and reinforcement learning \citep{driess2023palm, chen2021decision}. This has led to increased exploration of the mathematical foundations of attention's optimization. 

To understand optimization dynamics of attention mechanisms, \cite{tarzanagh2023transformers,ataee2024max} studied the \textit{implicit bias} of gradient descent (GD) in a binary classification setting with a fixed linear decoder. This bias refers to the tendency of GD to learn specific weight characteristics when multiple valid solutions exist. For example, in linear logistic regression on separable data, GD favors solutions aligned with the max-margin class separator \citep{soudry2018implicit, ji2018risk}. Similarly, \cite{tarzanagh2023transformers, ataee2024max} propose a model resembling a hard-margin Support Vector Machine (SVM)—specifically, \eqref{eqn:w-svm} with $p=2$—which maximizes the margin between optimal and non-optimal input tokens based on their softmax logits. These studies show that as training progresses, the combined key-query weights $W(k)$ of the model increasingly align with the locally optimal solution $W^\alpha_\mathrm{mm}$—the minimizer of \eqref{eqn:w-svm} with $p=2$, where $k$ denotes the training iteration. Expanding on these insights, \cite{vasudeva2024implicit} explores global directional convergence and the convergence rate of GD under specific conditions. \cite{sheen2024implicit} further extends these findings by relaxing assumptions about the convergence of regularized paths for the $(W_K, W_Q)$ parameterization of the key-query matrix, showing that gradient flow minimizes the nuclear norm of the key-query weight $W=W_K W_Q^\top$.

\subsection{Contributions}\label{sec:contr} While the aforementioned works provide insights into the implicit bias and token-selection properties of attention mechanisms, their analyses are limited to simplistic GD models. A broader understanding of more general descent algorithms, such as the mirror descent (\ref{eqn:md}) family, and their token-selection properties is missing. We address this by examining a family of \ref{eqn:md} algorithms designed for softmax attention, where the potential function is the $p$-th power of the $\ell_p$-norm, termed \ref{alg:p:agd}. This generalizes both $\ell_p$-GD \citep{azizan2018stochastic,sun2022mirror,sun2023unified} and attention GD \citep{tarzanagh2023transformers,ataee2024max}, enabling the exploration of key aspects of attention optimization via \ref{alg:p:agd}.

\noindent \textbf{Implicit Bias of \ref{alg:p:agd} for Attention Optimization}.
Building on the setting of \cite{tarzanagh2023transformers,vasudeva2024implicit,sheen2024implicit}, we consider a one-layer attention model for binary classification. Specifically, given a dataset \( (X_i,y_i,z_i)_{i=1}^n \) where \( X_i \in \mathbb{R}^{T \times d} \) represents inputs with \( T \) tokens, \( y_i \in \{\pm 1\} \) is the label, and \( z_i \in \mathbb{R}^d \) is the comparison token, we study a single-layer attention model $f(X_i,z_i) := v^\top X^\top_i \S(X_iWz_i)$, where $\S(\cdot)$ is the softmax function, $W$ is the key-query matrix, and $v$ is a linear decoder. {This single-head, one-layer setting is adopted for analytical tractability—a standard abstraction in theoretical studies of attention mechanisms \citep{tarzanagh2023transformers, ataee2024max, vasudeva2024implicit, sheen2024implicit}—while our experimental validation in Section \ref{sec:real-data-exp} demonstrates that these principles extend to practical multi-head, multi-layer architectures on real datasets.} 
Our goal is to separate a \textit{locally optimal} token $\alpha_i \in [T]$ of each input sequence \(X_i\) from the rest via an empirical risk minimization problem \eqref{eqn:erm} with a smooth decreasing loss. We extend the SVM formulation of \cite{tarzanagh2023transformers} to \eqref{eqn:w-svm}, defining a hard-margin SVM using the \(\ell_p\)-norm instead of the \(\ell_2\)-norm. The solution \( W^\alpha_\mathrm{mm} \) separates the locally optimal tokens \((\alpha_{i})_{i=1}^n\) with \textit{generalized} maximum margin (see Section \ref{sec:prelim}).

Theorem~\ref{thm-direction} provides sufficient conditions for \ref{alg:p:agd} to converge locally, in direction, to \(W_\mathrm{mm}^\alpha\). Moreover, Theorem~\ref{thm-norm} shows that \(\|W(k)\|_{p,p}\) diverges as \(k \rightarrow \infty\). These results characterize the implicit bias towards \eqref{eqn:w-svm} in separating locally optimal tokens, extending previous work to a broader class of algorithms. While Theorem~\ref{thm-direction} and the results of \cite{tarzanagh2023transformers,ataee2024max} offer insights into optimization dynamics for \(p=2\), the finite-time convergence rate of \ref{alg:p:agd} for selecting locally optimal tokens remains unexplored.

\noindent \textbf{Convergence Rate of \ref{alg:p:agd} to the Solution of \eqref{eqn:w-svm}}. Theorem~\ref{thm-rate} establishes convergence rates for \ref{alg:p:agd}, showing that the iterates $W(k)$, for large $k$, exhibit a decrease in $D_\psi\left(W_\mathrm{mm}^\a/\|W_\mathrm{mm}^\a\|_{p,p}, W(k)/\|W(k)\|_{p,p}\right)$ at an inverse poly-log rate, where $D_\psi(\cdot,\cdot)$ denotes the Bregman divergence \citep{bregman1967relaxation}; see Definition~\ref{def:breg:d}. Despite optimizing a highly nonlinear, nonconvex softmax function, we achieve a convergence rate similar to GD in linear binary classification \cite[Theorem 1.1]{ji2018risk}. Compared to the recent polynomial rate \(O(k^{-3/4})\) in \cite[Theorem 1]{vasudeva2024implicit} for optimizing attention, our rate is logarithmic and slower, but applicable to standard GD and \ref{eqn:md} for locally optimal token selection. Importantly, we do not require the near-orthogonality of tokens assumption used in \cite{vasudeva2024implicit}.

\noindent \textbf{Generalized Max-Margin Solutions and Joint Optimization of $(v,W)$}.  We investigate the joint problem over \(v\) and \(W\) under logistic loss using $\ell_p$-norm regularization path, where \eqref{eqn:erm} is solved under $\ell_p$-norm constraints, examining the solutions trajectory as these constraints relax. Since the problem is linear in \(v\), if the attention features \(\bar{X}_i = X_i^\top \S(X_i W z_i)\) are separable by their labels \(y_i\), \(v\) acts as a generalized max-margin classifier \citep{azizan2021stochastic}. Inspired by \cite{tarzanagh2023transformers,ataee2024max},  we show that under suitable geometric conditions, \(W\) and \(v\) generated by $\ell_p$-norm regularization path converge to their respective max-margin solutions (Theorem \ref{thm:joint:rp} in the appendix).

Finally, we provide extensive numerical experiments on real and synthetic data, demonstrating that \ref{eqn:md} algorithms improve generalization over standard GD, excelling in optimal token selection and suppressing non-optimal tokens.

\subsection{Additional Related Work}\label{sec:rel-work}

\textbf{Transformers Optimization.} 
Recently, the study of the mathematical foundations of attention optimization has gained significant attention
\citep{deora2023optimization, huang2023context, tian2023joma, fu2023transformers, li2024mechanics, ataee2024max, tarzanagh2023transformers, vasudeva2024implicit, sheen2024implicit, deng2023unmasking, makkuva2024attention, jeon2024information, zheng2023learn, collins2024context, chen2024provably, li2023theoretical, ildiz2024self, vasudeva2024simplicity, bao2024self, chen2024training, huang2024non, wang2024transformers, zhao2024implicit,sun2025context}. Below, we highlight works most relevant to this paper. 

Studies such as \citet{sahiner2022unraveling, ergen2022convexifying} investigate attention model optimization through convex relaxations. \citet{jelassi2022vision} demonstrate that Vision Transformers (ViTs) identify spatial patterns in binary classification via gradient methods. \citet{li2023theoretical} provide sample complexity bounds and discuss attention sparsity in SGD for ViTs.  
\citet{nguyen2024primal} introduce static primal-dual formulations for attention mechanisms, connecting self-attention to support vector regression (SVR). 
\citet{chen2024primal} propose a novel attention mechanism that optimizes self-attention in Transformers using asymmetric Kernel Singular Value Decomposition (KSVD) in the primal representation, achieving efficiency and performance improvements through low-rank approximations and regularization techniques. However, these works do not examine optimization dynamics, the role of descent algorithms, or the implications of implicit bias in training, which are the main focus of our work.  

\citet{oymak2023role} and \citet{deora2023optimization} explore optimization dynamics in prompt-tuning and multi-head attention models, respectively. \citet{tian2023scan, tian2023joma} study SGD dynamics and multi-layer Transformer training. \citet{ataee2024max, tarzanagh2023transformers} examine the implicit bias of GD in attention optimization. \citet{vasudeva2024implicit} analyze the global directional convergence and convergence rate of GD for attention optimization under specific data conditions. Furthermore, \citet{tarzanagh2023transformers} and \citet{sheen2024implicit} respectively show that the regularization path and gradient flow not only achieve minimal loss but also minimize the nuclear norm of the key-query weight \( W \). \citet{thrampoulidis2024implicit, li2024mechanics, zhao2024implicit} also study the optimization dynamics of attention mechanisms and the implicit bias of GD in next-token prediction. Our work extends these findings and those of \citet{tarzanagh2023transformers,ataee2024max}, focusing on the implicit bias of the general class of MD algorithms for attention training. {Similar to these prior works, our theoretical analysis adopts a single-head, one-layer setting—a standard abstraction that balances analytical tractability with capturing the essential nonconvex optimization structure of softmax attention. Extensions to multi-head and multi-layer attention theory, as explored in \citep{tian2023joma, collins2024context, chen2025provably}, remain an important direction for future work.}

\noindent
\textbf{Implicit Bias of First-Order Methods.} 
Significant progress has been made in understanding the implicit bias of gradient descent on separable data, particularly highlighted by \citet{soudry2018implicit, ji2018risk}. For linear predictors, \citet{nacson2019convergence, ji2021characterizing, ji2021fast} demonstrate that gradient descent methods rapidly converge to the max-margin predictor. Extending these insights to multi-layer perceptrons (MLPs), \citet{ji2020directional, lyu2019gradient, chizat2020implicit} examine the implicit bias of GD and gradient flow using exponentially-tailed classification losses, showing convergence to the Karush-Kuhn-Tucker (KKT) points of the corresponding max-margin problem, both in finite \citep{ji2020directional, lyu2019gradient} and infinite-width scenarios \citep{chizat2020implicit}. Furthermore, the implicit bias of GD for training ReLU and Leaky-ReLU networks has been investigated, particularly in the context of orthogonal data \citep{phuong2020inductive, frei2022implicit}. Additionally, the implicit bias toward rank minimization in regression settings with square loss has been explored in \citep{vardi2021implicit, arora2019implicit, li2020towards}.

Our work is also related to the implicit bias of MD \citep{gunasekar2018characterizing, azizan2018stochastic} in regression and classification, respectively. Specifically, \citet{sun2022mirror} extend the findings of \citet{gunasekar2018characterizing, azizan2018stochastic} to classification problems, developing a class of algorithms that exhibit an implicit bias toward a generalized SVM with \( \ell_p \) norms, effectively separating samples based on their labels. For a comprehensive survey, we refer to \citet{vardi2023implicit}.

\section{Preliminaries}\label{sec:prelim}
\paragraph{Notations.} Let \( N \geq 1 \) and \([N] = \{1, 2, \dots, N\}\). Vectors are denoted by lowercase letters (e.g., \( a \)), with components \( a_i \), and matrices by uppercase letters (e.g., \( A \)). The minimum and maximum of scalars \( a \) and \( b \) are \( a \wedge b \) and \( a \vee b \), respectively. For a vector \( v \in \mathbb{R}^d \), the \( p \)-norm is \( \|v\|_p = (\sum_{i=1}^d |v_i|^p)^{1/p} \). For a matrix \( M \in \mathbb{R}^{d \times d} \), the \( p,p \)-norm is \( \|M\|_{p,p} = (\sum_{i=1}^d \sum_{j=1}^d |M_{ij}|^p)^{1/p} \).   When \( p = 2 \), these are the Euclidean norm for vectors and the Frobenius norm for matrices. For any two matrices $X, Y$ of the same dimensions, we define $\langle X, Y \rangle := \text{trace}(X^\top Y)$. $\mathcal{O}$ denotes an upper bound, $\Omega$ a lower bound, and $\Theta$ both. All logarithms are natural (base $e$). Throughout, for a differentiable function $f:\mathbb{R}^{d\times d}\rightarrow\mathbb{R}$, we define $   D_{f}: \mathbb{R}^{d\times d} \times  \mathbb{R}^{d\times d} \rightarrow  \mathbb{R} $ as 
\begin{align}\label{eqn:def:df}
    D_{f}(W,V):=f(W)-f(V)-\langle\nabla f(V),W-V\rangle.
\end{align}

\paragraph{Single-head attention model.} Given input sequences \( X, Z \in \mathbb{R}^{T \times d} \) with length \( T \) and embedding dimension \( d \), the output of a single-head (cross)-attention layer is computed as:
\begin{align*}
\textnormal{softmax}(X W_Q W_K^\top Z^\top) X W_V,
\end{align*}
where \( W_Q, W_K \in \mathbb{R}^{d \times d_1} \), \( W_V \in \mathbb{R}^{d \times d_2} \) 
are trainable key, query, value matrices, respectively; \(\textnormal{softmax}(X W_Q W_K^\top Z^\top)\) is the attention map; and \(\textnormal{softmax}(\cdot): \mathbb{R}^{T \times T} \rightarrow \mathbb{R}^{T \times T}\)
denotes the row-wise softmax function applied row-wise on \(X W_Q W_K^\top Z^\top\). Similar to \cite{tarzanagh2023transformers,ataee2024max}, we reparameterize the key-query product matrix as \( W := W_Q W_K^\top \in \mathbb{R}^{d \times d} \), and subsume the value weights \( W_V \) within the prediction head \( v \in \mathbb{R}^d \). Suppose the first token of \( Z \), denoted by \( z \), is used for prediction. Then, the attention model can be formulated as
\begin{align}\label{eq-model-W}
f(X,z) = v^\top X^\top \S(XWz),
\end{align}
where $\S : \mathbb{R}^T \rightarrow \mathbb{R}^T$ is the softmax function on vectors.

\noindent
\textbf{Attention-based empirical risk minimization.} We consider a one-layer attention model \eqref{eq-model-W} for binary classification. Consider the dataset \( (X_i,y_i,z_i)_{i=1}^n \), where \( X_i \in \mathbb{R}^{T \times d} \) is the input with \( T \) tokens each of dimension \( d \), \( y_i \in \{\pm 1\} \) is the label, and \( z_i \in \mathbb{R}^d \) is the token used for comparison. We use a smooth decreasing loss function \( l: \mathbb{R} \rightarrow \mathbb{R} \) and study empirical risk minimization (ERM):
\begin{align}\label{eqn:erm}
\min_{v \in \mathbb{R}^d, W \in \mathbb{R}^{d \times d}} \quad \mathcal{L}(v,W) := \frac{1}{n} \sum_{i=1}^n l\left(y_i v^\top X_i^\top \S\left(X_i W z_i\right)\right). \tag{ERM}
\end{align}
Throughout, we will use $ \mathcal{L}(W)$ to denote the objective of \eqref{eqn:erm} with fixed $v $.

The highly nonlinear and nonconvex nature of the softmax operation makes the training problem described in \eqref{eqn:erm} a challenging nonconvex optimization task for \( W \), even with a fixed \( v \). Next, we provide an assumption on the loss function necessary to demonstrate the convergence of \ref{eqn:md}  for margin maximization within the attention mechanism. 

\begin{assumption}\label{assumption-loss}
Within any closed interval, the loss function \( l:\mathbb{R} \rightarrow \mathbb{R} \) is strictly decreasing and differentiable, and its derivative \( l' \) is bounded and Lipschitz continuous.
\end{assumption}
Assumption~\ref{assumption-loss} aligns with the assumptions on loss functions in \cite{ataee2024max, tarzanagh2023transformers}. Commonly used loss functions, such as \( l(x) = \exp(-x) \), \( l(x) = -x \), and \( l(x) = \log(1 + \exp(-x)) \), satisfy this assumption.

\noindent
\textbf{Preliminaries on mirror descent.} We review the \ref{eqn:md} algorithm \citep{original-md} for solving attention-based \eqref{eqn:erm}. Mirror descent is defined using a \textit{potential function}. We focus on differentiable and strictly convex potentials $\psi$ defined on the entire domain $\mathbb{R}^{d \times d}$. Note that in general, the potential function is a convex function of Legendre type \citep[Section 26]{rockafellar2015convex}. We call $\nabla \psi$ the \textit{mirror map}. The natural ``distance'' associated with the potential $\psi$ is given by the Bregman divergence \citep{bregman1967relaxation}.

%
\begin{definition}[Bregman Divergence]\label{def:breg:d}
For a strictly convex function $\psi: \mathbb{R}^{d \times d}  \rightarrow \mathbb{R}$, the expression $D_\psi(\cdot,\cdot)$ defined in \eqref{eqn:def:df} is called the Bregman divergence.
\end{definition}
An important example of a potential function is $\psi = \frac{1}{2}\|\cdot\|^2_F$. In this case, the Bregman divergence simplifies to \( D_\psi(W, V) = \frac{1}{2}\|W -V\|^2_F \); For more details, see \cite{bauschke2017descent}. \ref{eqn:md} with respect to the mirror map $\psi$ is a generalization of GD where the Bregman divergence is used as a measure of distance. Given a stepsize $\eta>0$, the \ref{eqn:md} algorithm is as follows:
\begin{align}\label{eqn:md}
W(k+1) \leftarrow \argmin_{W \in \mathbb{R}^{d \times d}} \left\{{\eta}^{-1}D_\psi(W, W(k)) + \left\langle\nabla \mathcal{L}(W(k)), {W} \right\rangle \right\}. \tag{\texttt{MD}}
\end{align}
Equivalently, \ref{eqn:md} can be written as $\nabla \psi(W(k+1)) = \nabla \psi(W(k)) - \eta \nabla \L(W(k))$; see \cite{bubeck2015convex,juditsky2011first}. A useful fact about the Bregman divergence is that it is non-negative and $D_\psi(W,V)=0$ if and only if $W=V$. 

\noindent
\textbf{Preliminaries on attention SVM.} Following \cite{ataee2024max,tarzanagh2023transformers}, we use the following definition of token scores.

\begin{definition}[Token Score]\label{def:token:score}
For prediction head $v \in \mathbb{R}^d$, the score of token $X_{it} $ is $\gamma_{it} = y_i v^\top X_{it}$.
\end{definition}
It is important to highlight that the score is determined solely based on the \emph{value embeddings} $v^\top X_{it}$ of the tokens. The softmax function $\S(\cdot)$ minimizes \eqref{eqn:erm} by selecting the token with the highest score \cite[Lemma 2]{tarzanagh2023transformers}. Using \eqref{def:token:score}, \cite{tarzanagh2023transformers} defines globally optimal tokens $(\op_i)_{i=1}^n$, with each $\op_i$ maximizing the score for $X_{i\op_i}$. For our \ref{eqn:md} analysis, we primarily consider locally optimal tokens, as they are more general than globally optimal ones. Locally optimal tokens \citep{ataee2024max,tarzanagh2023transformers} are characterized by having scores that surpass those of nearby tokens, we formalize the notion of nearby tokens later in Definition \ref{def:token:loptimal} on locally optimal tokens and \textit{support tokens}. Intuitively, these are the tokens that locally minimize \eqref{eqn:erm} upon selection and can be defined based on support tokens. Before presenting the mathematical notion of locally optimal tokens, we provide the formulation of the attention SVM problem. Given a set of (locally) optimal token indices $(\alpha_i)_{i=1}^n\in[T]^n$, \cite{tarzanagh2023transformers} defines the following hard-margin attention SVM problem, which aims to separate, with maximal margin, (locally) optimal tokens from the rest of the tokens for every input sequence:
\begin{equation}
    \begin{split}\label{eqn:w-2-svm}
        W^\alpha_\mathrm{mm} &:= \argmin_{W \in \mathbb{R}^{d \times d}} \|W\|_F \\
        &\textnormal{subj.~to}~~(X_{i\alpha_i} - X_{it})^\top W z_i \geq 1,~~\textnormal{for all }~t \in [T] - \{\alpha_i\}, \, i \in [n].
    \end{split}
\end{equation}
The constraint $(X_{i\alpha_i} - X_{it})^\top W z_i \geq 1$ indicates that in the softmax probability vector $\S(X_i W z_i)$, the $\alpha_i$ component has a significantly higher probability compared to the rest, and so these problems solve for a sort of probability separator that has the lowest norm.

\begin{definition}[Globally and Locally Optimal Tokens]\label{def:token:loptimal}
Consider the dataset \( (X_i, y_i, z_i)_{i=1}^n \).
 \vspace{-15pt}
\begin{enumerate}[label={\textnormal{{\textbf{\arabic*.}}}}, wide, labelindent=0pt, itemsep=-5pt]
\item  The tokens with indices $\op=(\op_i)_{i=1}^n$ are called globally optimal if they have the highest scores, given by $\op_i \in \arg\max_{t \in [T]}\gamma_{it}$.
\item  Fix token indices $(\alpha_i)_{i=1}^n$ for
which \eqref{eqn:w-2-svm} is feasible to obtain $W^\alpha_{\mathrm{mm}}$.  Let the support tokens $\mathcal{T}_i$ for the $i^{th}$ data be the set of tokens $\tau$ such that $(X_{i\alpha_i} - X_{i\tau})^\top W^\alpha_{\mathrm{mm}} z_i = 1$. The tokens with indices $(\alpha_i)_{i=1}^n$ are called locally optimal if, for all $i \in [n]$ and $\tau \in \mathcal{T}_i$, the scores per Def.~\ref{def:token:score} obey $\gamma_{i\alpha_i} > \gamma_{i\tau}$.
\end{enumerate}
\end{definition}

It is worth noting that token scoring and optimal token identification can help us understand the importance of individual tokens and their impact on the overall objective. A token score measures how much a token contributes to a prediction or classification task, while an optimal token is defined as the token with the highest relevance in the corresponding input sequence \citep{tarzanagh2023transformers}. For illustration, please refer to Figure \ref{fig:svm_token_visualization}.
\section{Implicit Bias of Mirror Descent for Optimizing Attention}

\subsection{Optimizing Attention with Fixed Head $v$ }\label{sec:w:train}
In this section, we assume the prediction head is fixed and focus on the directional convergence of \ref{eqn:md} and its token-selection property through the training of the key-query matrix \( W \). The analysis will later be expanded in Section \ref{sec:joint-optimization} to include the joint optimization of both \( v \) and \( W \).

We investigate the theoretical properties of the main algorithm of interest, namely \ref{eqn:md} with $\psi(\cdot) = \frac{1}{p} \|\cdot\|_{p,p}^p$ for $p > 1$ for training \eqref{eqn:erm} with fixed $v$. We shall call this algorithm \textit{$\ell_p$-norm AttGD} because it naturally generalizes attention training via GD to $\ell_p$ geometry, and for conciseness, we will refer to this algorithm by the shorthand \ref{alg:p:agd}. As noted by \cite{azizan2021stochastic}, this choice of mirror potential is particularly of practical interest because the mirror map $\nabla \psi$ updates become \textit{separable} in coordinates and thus can be implemented \textit{coordinate-wise} independently of other coordinates.  The update steps of the $\ell_p$-norm AttGD algorithm (\ref{alg:p:agd}) are given as follows: for all $i, j \in [d]$, and $k=0,1,\ldots,$
\begin{align}
\begin{cases}
    \left[W(k+1)\right]_{ij} \leftarrow \left| [W(k)]^+_{ij}\right|^{\frac{1}{p-1}} \cdot \textnormal{sign}\left([W(k)]^+_{ij}\right),\vspace{.3cm}\\
    [W(k)]^+_{ij} := |[W(k)]_{ij}|^{p-1} \textnormal{sign}([W(k)]_{ij}) - \eta [\nabla \L(W(k))]_{ij}. 
\end{cases}
 \tag{\texttt{$\ell_p$-AttGD}}\label{alg:p:agd}
\end{align}

The algorithm will still incur additional overhead compared to GD, but this overhead is linear in the size of the trainable parameters for both time and space. We discuss this further the Appendix \ref{alg:overhead}. Under Assumption~\ref{assumption-loss} on the loss function, we establish the following property of Algorithm~\ref{alg:p:agd}.

\begin{lemma}[$\ell_{p,p}$-Growth Bound of Attention Weights]\label{lem:norm-growth-bound}
Let Assumption \ref{assumption-loss} hold. Consider the sequence $\{W(k)\}_{k \geq 0}$ generated by Algorithm \ref{alg:p:agd} with stepsize $\eta > 0$. Then, the increment of the $\ell_{p,p}$-norm between consecutive iterations can be bounded as follows:
\[
\|W(k+1)\|_{p,p} - \|W(k)\|_{p,p} \leq \mathcal{O}\big(\|W(k)\|_{p,p}^{2-p}\big), \quad \textnormal{for all}~~~ k = 0, 1, 2, \dots~.
\]
\end{lemma}

Lemma~\ref{lem:norm-growth-bound} characterizes how the $\ell_{p,p}$-norm of attention weights evolves during training. When $p > 2$, the bound indicates sublinear growth, while for $p < 2$, it allows superlinear growth. This controlled growth property plays a fundamental role in our convergence rate in Theorem~\ref{thm-rate}, as it enables us to characterize the asymptotic behavior of the Bregman divergence between normalized iterates and the max-margin solution. Specifically, the growth rate established here directly influences the poly-logarithmic factors in the convergence rate, which varies depending on whether $p > 2$, $p = 2$, or $p < 2$.

In the following, we first identify the conditions that guarantee the convergence of \ref{alg:p:agd}. The intuition is that, for attention to exhibit implicit bias, the softmax nonlinearity should select the locally optimal token within each input sequence. \cite{tarzanagh2023transformers} shows that under certain assumptions, training an attention model using GD causes its parameters' direction to converge.

This direction can be found by solving a simpler optimization problem, such as \eqref{eqn:w-2-svm}, which selects the locally optimal token. Depending on the attention model's parameterization, the attention SVM varies slightly. In this work, we generalize \eqref{eqn:w-2-svm} using the $\ell_p$-norm as follows:

\begin{definition}[Attention SVM with $\ell_p$--norm Objective]\label{def:w:attsvm}
Given  $\{(X_i, y_i, z_i)\}_{i=1}^n$ with $y_i\in \{\pm 1\}$, $X_{i}\in\R^{T\times d}$, and token indices $(\a_i)_{i=1}^n$, $\ell_p$-based attention SVM is defined as 
\begin{align}\label{eqn:w-svm}
W^\a_\mathrm{mm}&:=\argmin_{W \in \mathbb{R}^{d \times d}}\|W\|_{p,p}     \tag{{$\ell_p$-AttSVM}}    \\
&\textnormal{subj.~to}~~(X_{i\a_i}-X_{it})^\top Wz_i\geq1,~~\textnormal{for all}~~t\in[T]-\{\a_i\},~i\in[n].    
\nonumber 
\end{align}
\end{definition}

\begin{wrapfigure}{r}{0.42\textwidth}
    \centering
    \vspace{-2cm}
    \includegraphics[width=0.41\textwidth]{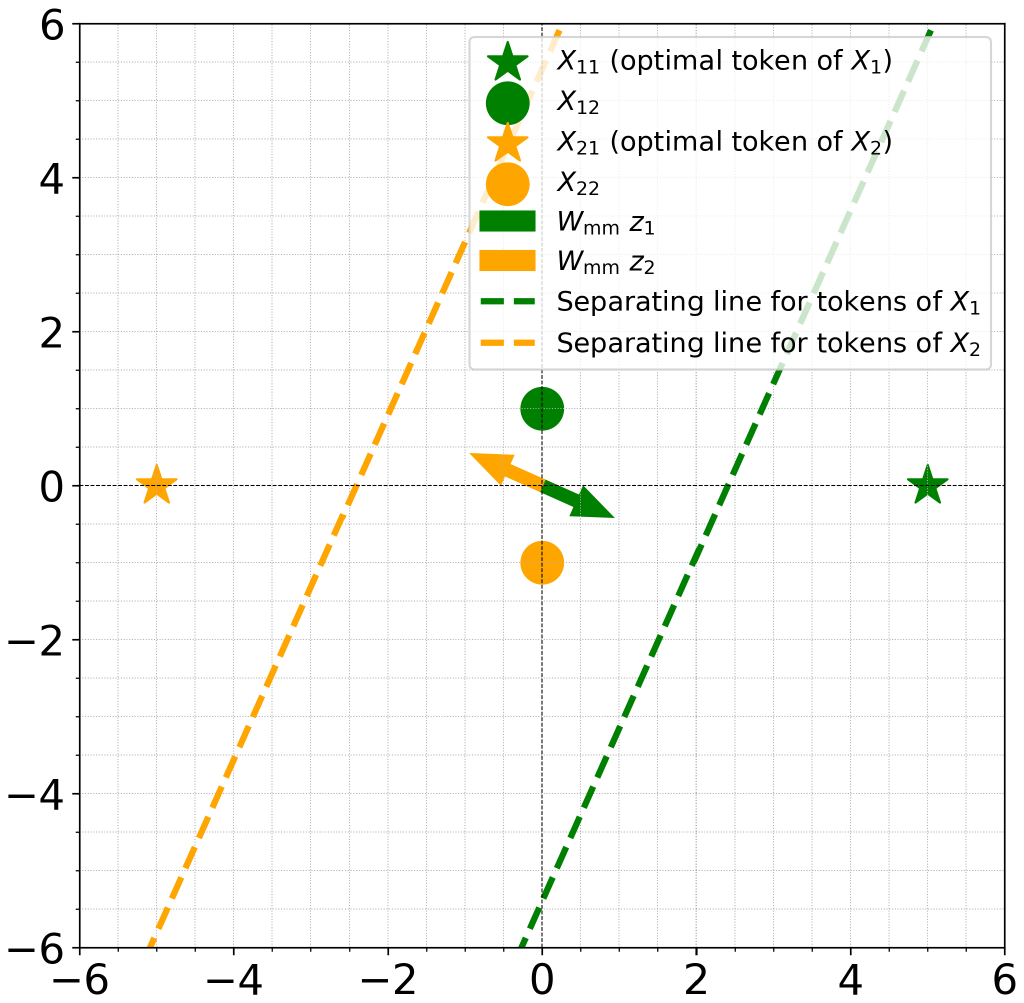}
            \vspace{-2cm}
\caption{\small{Visualization of \eqref{eqn:w-svm} for $p=3$.}}
            \vspace{-.3cm}
    \label{fig:svm_token_visualization}
\end{wrapfigure}

Problem \eqref{eqn:w-svm} is strictly convex, so it has unique solutions when feasible. Throughout this paper, we assume feasibility, which means there exists a matrix \( W \) that linearly separates the logits \( X_{i\a_i}^\top Wz_i \) from the logits \( X_{it}^\top Wz_i \) for all \( t \in [T] \setminus \{\a_i\} \) and \( i \in [n] \). 

It is worth noting that this is not a strong assumption. For example, under mild overparameterization, $d \geq \max\{T-1,n\}$, the problem is almost always feasible \cite[Theorem 1]{tarzanagh2023transformers}. Next, we assert that the solution to the \eqref{eqn:w-svm} problems determines the direction that the attention model parameters approach as the training progresses. 
%

\begin{example}\label{exm:svm_token_visualization}
Consider the matrices $X_1 = [5,~0;~0,~1]$ and $X_2 = [-5,~0;~0,~-1]$ with $y_1 = -y_2 = 1$. Let \( X_{i1} \) be the optimal token and \( X_{it} \) be the others.  Solving Problem \eqref{eqn:w-svm} with \( p=3 \) and setting \( z_i = X_{i1} \), we obtain the solution $W_\mathrm{mm}:= W^\a_\mathrm{mm} = [0.03846,~0;~-0.00769,~0]$. Figure~\ref{fig:svm_token_visualization} illustrates how the optimal tokens \( X_{11} \) and \( X_{21} \) are separated by the dashed decision boundaries. These boundaries are orthogonal to the vectors \( W_\mathrm{mm} z_i \) and indicate the hyperplanes that separate the sequences based on the optimal token in each case. 
\end{example}

\begin{theorem}[$\ell_p$--norm Regularization Path]\label{thm:rp:w}
Suppose Assumption~\ref{assumption-loss} on the loss function holds. Consider the ridge-constrained solutions \(W^{(R)}\) of \eqref{eqn:erm} defined as 
\begin{equation}\label{alg:arp}
W^{(R)} := \argmin_{W \in \mathbb{R}^{d \times d}}~ \mathcal{L}( W) \quad \textnormal{subj. to} \quad \|W\|_{p,p} \leq R.
\tag{\texttt{$\ell_p$-AttRP}}
\end{equation} 
Then, \(\lim_{R \to \infty} W^{(R)}/R= W^{\opt}_\mathrm{mm}/{\|W^{\opt}_\mathrm{mm}\|_{p,p}}\), where \(W^{\opt}_\mathrm{mm}\) is the solution of \eqref{eqn:w-svm}, with \(\alpha_i\) replaced by \(\opt_i\).
 \end{theorem}

Theorem~\ref{thm:rp:w} shows that as the radius \( R \) increases, the optimal direction \( W^{(R)} \) aligns more closely with the max-margin solution \( W^\alpha_\mathrm{mm} \). This theorem, which allows for globally optimal tokens (see Definition~\ref{def:token:loptimal}), does not require any specific initialization for the \ref{alg:arp} algorithm and demonstrates that max-margin token separation is an essential feature of the attention mechanism. 

Next, we analyze the convergence of \ref{eqn:md} applied to \eqref{eqn:erm}. Under specific initializations, the parameter's \(\ell_p\)-norm diverges to infinity, while its direction approaches the \eqref{eqn:w-svm} solution. To characterize these initializations, we define the following sets.
\begin{definition}\label{def:cone:ellp}
Given a square matrix $W\in\mathbb{R}^{d\times d}$, $\mu\in(0,1)$, and some $R>0$,
\begin{subequations}
\begin{align}
\label{eqn:def:spmu}
S_{p,\mu}(W) &:= \left\{W'\in\mathbb{R}^{d\times d} \mid D_\psi\left(\frac{W}{\|W\|_{p,p}},\frac{W'}{\|W'\|_{p,p}}\right) \leq \mu\right\}, \\
C_{p,\mu,R}(W) &:= S_\mu(W) \cap \left\{W' \mid \|W'\|_{p,p} \geq R\right\}.
\end{align}
\end{subequations}
\end{definition}

These sets contain matrices with a similar direction to a reference matrix \(W\), as captured by the inner product in \(S_\mu(W)\). For \(C_{p,\mu,R}(W)\), there is an additional constraint that the matrices must have a sufficiently high norm. We note that \(S_{p,\mu}(W)\) and \(C_{p,\mu,R}(W)\) reduce to their Euclidean variants as described in \cite{ataee2024max,tarzanagh2023transformers} when $p=2$. With this definition, we present our first theorem about the norm of the parameter increasing during training.

\begin{theorem}\label{thm-norm}
Suppose Assumption \ref{assumption-loss} holds. Let \((\alpha_i)_{i=1}^n\) be locally optimal tokens as per Definition \ref{def:token:loptimal}. Consider the sequence \(W(k)\) generated by Algorithm \ref{alg:p:agd}. For a small enough stepsize $\eta$, if \(W(0) \in C_{p,\mu,R}(W^\alpha_\mathrm{mm})\) for some dataset-dependent constants \(\mu, R > 0\), then we have $\|W(k)\|_{p,p}=\Omega(\log k)$.
\end{theorem}

\begin{remark} The condition on the stepsize $\eta$ is that it must be sufficiently small so that $\psi(\cdot)-\eta\mathcal{L}(\cdot)$ remains convex for the matrices $W$ along the path traced by the iterates $W(k)$. This applies to all theorems in this paper that require a sufficiently small stepsize $\eta$. 
\end{remark}
This theorem implies that the parameters will increase and diverge to infinity, justifying the need to characterize the convergence of their direction. 

\begin{theorem}[Convergence of \ref{alg:p:agd}]\label{thm-direction}
Suppose Assumption \ref{assumption-loss} holds. Let \((\alpha_i)_{i=1}^n\) be locally optimal tokens as per Definition \ref{def:token:loptimal}. Consider the sequence \(W(k)\) generated by Algorithm \ref{alg:p:agd}. For a small enough stepsize $\eta$, if \(W(0) \in C_{p,\mu,R}(W^\alpha_\mathrm{mm})\) for some dataset-dependent constants \(\mu > 0, R > \textnormal{exp}(2)\), then
    \[
    \lim_{k \rightarrow \infty} \frac{W(k)}{\|W(k)\|_{p,p}} = \frac{W^\alpha_\mathrm{mm} }{\|W^\alpha_\mathrm{mm} \|_{p,p}}.
    \]
\end{theorem}

These theorems show that as the parameters grow large enough and approach a locally optimal direction, they will keep moving toward that direction. {While our results build on the max-margin token selection framework of~\cite{tarzanagh2023transformers}, the extension to $\ell_p$-mirror descent is nontrivial: we establish directional convergence under non-quadratic potentials where standard Euclidean arguments fail, and provide the first explicit finite-time poly-logarithmic convergence rates for attention optimization in Theorem~\ref{thm-rate}.}

\begin{theorem}[Convergence Rate of \ref{alg:p:agd}]\label{thm-rate}
Suppose Assumption \ref{assumption-loss} holds. Let \((\alpha_i)_{i=1}^n\) be locally optimal tokens as per Definition \ref{def:token:loptimal}. Consider the sequence \(W(k)\) generated by Algorithm \ref{alg:p:agd}. For a small enough stepsize $\eta$, if \(W(0) \in C_{p,\mu,R}(W^\alpha_\mathrm{mm})\) for some \(\mu > 0, R > \textnormal{exp}(2)\), then
\begin{equation}\label{eqn:dwk:final:rate}
D_\psi\left(\frac{W^\alpha_\mathrm{mm}}{\|W^\alpha_\mathrm{mm}\|_{p,p}}, \frac{W(k)}{\|W(k)\|_{p,p}}\right) = \mathcal{O}\left( 
\begin{cases}
\frac{\log \log k}{\log k} & \textnormal{if}~~p > 2, \\
\frac{(\log \log k)^2}{\log k} & \textnormal{if}~~p = 2, \\
\frac{1}{(\log k)^{p-1}} & \textnormal{otherwise}.
\end{cases}\right).
\end{equation}
\end{theorem}

{Establishing these rates required developing new analytical tools beyond the $\ell_2$ blueprint: we prove $\ell_{p,p}$-norm growth bounds under non-quadratic potentials (Lemma~\ref{lem:norm-growth-bound}), introduce double-cone stability arguments to handle joint norm and directional constraints (Lemma~\ref{lemma-stay-in-cone}), and work throughout with Bregman divergences rather than Euclidean geometry.} 

\begin{wrapfigure}{r}{0.42\textwidth}
    \centering
    \vspace{-0.3cm}
    \includegraphics[width=0.39\textwidth]{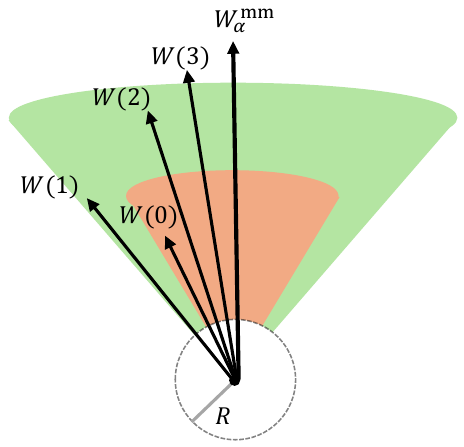}
    \vspace{-0.5cm}
    \caption{\small{Illustration of Lemma~\ref{lemma-stay-in-cone}. \(W(k)\), $\forall k>0$ are within the larger set.}}
    \label{fig:cones_and_arrows}
    \vspace{-0.2cm}
\end{wrapfigure}

Note that in the left-hand side of \eqref{eqn:dwk:final:rate}, there is a dependence on $p$ in the Bregman divergence $D_\psi$ itself as well. Despite optimizing a highly nonlinear, nonconvex softmax function, we achieve a convergence rate similar to GD in linear binary classification \cite[Theorem 1.1]{ji2018risk} (up to a $\log\log k$ factor).

Since we aim to show the parameter converges in direction to the cone center $W_\mathrm{mm}^\alpha$, we need conditions ensuring the parameters remain in the cone. {Unlike $\ell_2$ gradient descent where cone invariance follows from Pythagorean identities, the $\ell_p$ setting requires a novel double-cone construction: we must ensure iterates simultaneously remain within the directional cone $S_{p,\mu_0}(W_{\mathrm{mm}}^\alpha)$ and satisfy the norm constraint $\|W\|_{p,p}\geq R_\pi$—a strictly joint requirement not present in prior mirror descent analyses~\citep{sun2022mirror}.} We formalize this in Lemma~\ref{lemma-stay-in-cone} and prove that for any $\mu>0$ and locally optimal tokens $(\alpha_i)_{i=1}^n$ (Definition~\ref{def:token:loptimal}), there exist constants $R,\mu'>0$ depending on the dataset and $\mu$ such that if $W(0)\in C_{p,\mu',R}(W^\alpha_\mathrm{mm})$, then $W(k)\in C_{p,\mu,R}(W^\alpha_\mathrm{mm})$ for all $k$, meaning the iterates remain within a larger cone; see Figure~\ref{fig:cones_and_arrows}.

For Theorem \ref{thm-norm}, we show in  Lemma~\ref{lem:alg:md-corr-bound}  that   at any timestep $k\geq 0$, the norm of the $W$ parameter evolves in the following manner,
\begin{align*}
|W(k+1)\|_{p,p}^{p-1} &\geq\|W(k)\|_{p,p}^{p-1}+\frac{\eta}{\|W(k)\|_{p,p}}\langle-\nabla\L(W(k)),W(k)\rangle.    
\end{align*}
With the above, to prove Theorem \ref{thm-norm}, it is enough to show that \(\langle -\nabla\mathcal{L}(W(k)), W(k) \rangle\) is positive and large enough to keep the norm increasing to infinity. Specifically, in Lemma~\ref{lem:obj:grad-corr-bound-1} we show that there exist dataset-dependent constants \(R, \delta, \mu > 0\) such that for all \(W, V \in C_{p,\mu,R}(W^\alpha_\mathrm{mm} )\) with \(\|V\|_{p,p} = \|W^\alpha_\mathrm{mm} \|_{p,p}\),
\[
-\langle \nabla\mathcal{L}(W), V \rangle = \Omega\left( \textnormal{exp} \left(-\frac{\|W\|_{p,p}}{\|W^\alpha_\mathrm{mm} \|_{p,p}}\left(1+\frac{1}{2}\delta\right)\right)\right) > 0.
\]

Theorem \ref{thm-direction} is a direct consequence of Theorem \ref{thm-rate}, which extends the analysis that is done for Lemma \ref{lemma-stay-in-cone} by providing a tighter bound on how thin the cone set $C_{p,\mu,R}$ may be for later iterates.

\subsection{Training Dynamics of Mirror Descent for Joint Optimization of $W$ and $v$}\label{sec:joint-optimization}
This section explores the training dynamics of jointly optimizing the prediction head \( v \) and attention weights \( W \). Unlike Section~\ref{sec:w:train}, the main challenge here is the evolving token scores \( \gamma \), as given in Definition~\ref{def:token:score}, are influenced by the changing nature of \( v \). This requires additional technical considerations beyond those in Section~\ref{sec:w:train}, which are also addressed in this section. 

Given stepsizes $ \eta_W, \eta_v >0$,  we consider the following \textit{joint} updates for $W$ and $v$ applied to \eqref{eqn:erm}, respectively: for all $i, j \in [d]$, and $k = 0, 1, 2, \dots$,
\begin{align}
\hspace{-20pt}
\begin{cases}
    \left[W(k+1)\right]_{ij} \leftarrow \left| [W(k)]^+_{ij}\right|^{\frac{1}{p-1}} \cdot \textnormal{sign}\left([W(k)]^+_{ij}\right),\vspace{.3cm}\\
    [W(k)]^+_{ij} := |[W(k)]_{ij}|^{p-1} \textnormal{sign}([W(k)]_{ij}) - \eta_W [\nabla_W \L(W(k), v(k))]_{ij},\vspace{.3cm}\\
    [v(k+1)]_i \leftarrow \left| [v(k)]_i^+ \right|^{\frac{1}{p-1}} \cdot \textnormal{sign}([v(k)]_i^+),\vspace{.3cm}\\
    [v(k)]_i^+ := |[v(k)]_i|^{p-1} \textnormal{sign}([v(k)]_i) - \eta_v [\nabla_v \L(W(k), v(k))]_i.
\end{cases}
 \tag{\texttt{$\ell_p$-JointGD}}\label{alg:pq:agd}
\end{align}

We discuss the implicit bias and convergence for $v(k)$ below. From previous results \citep{azizan2021stochastic}, one can expect $v(k)$ to converge to the $\ell_p$-SVM solution, i.e., the max-margin classifier separating the set of samples $\{(X_{i\alpha_i}, y_i)\}_{i=1}^n$, where $X_{i\alpha_i}$ denote the (locally) optimal token for each $i \in [n]$. Consequently, we consider the following hard-margin SVM problem,
\begin{equation}
    \begin{split}\label{eqn:vp:svm}
        v_\mathrm{mm} &:= \argmin_{v\in \mathbb{R}^d} \|v\|_p \\
        &\textnormal{subj.~to}~~~y_i X_{i \alpha_i}^\top v \geq 1,~~\textnormal{for all }~ i \in [n].
    \end{split}
    \tag{$\ell_p$-SVM}    
\end{equation}

\begin{figure}[t]
     \centering
     \begin{subfigure}[b]{0.235\textwidth}
         \centering
         \includegraphics[width=\textwidth]{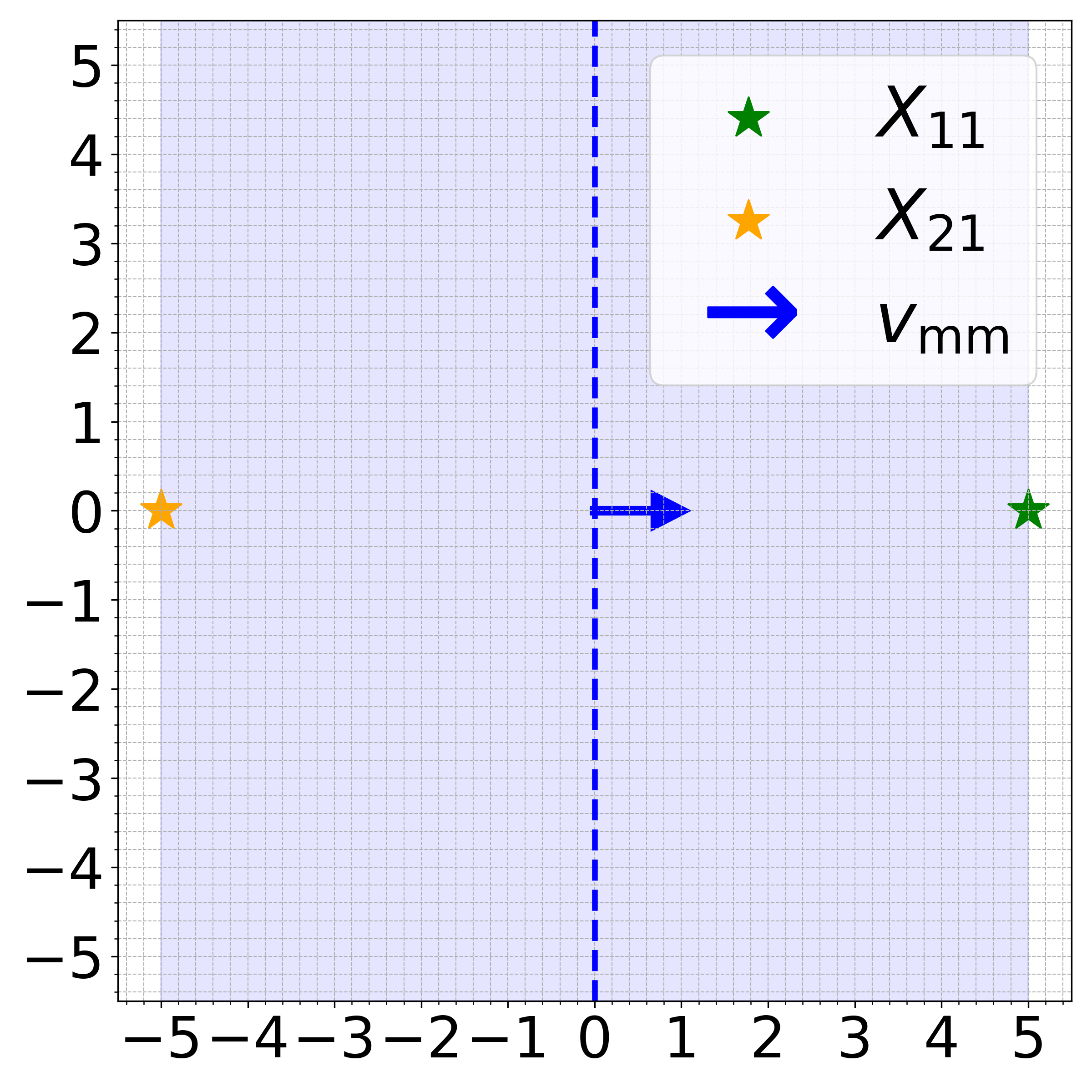}
     \end{subfigure}%
     \hspace{0.1cm}
     \begin{subfigure}[b]{0.235\textwidth}
         \centering
         \includegraphics[width=\textwidth]{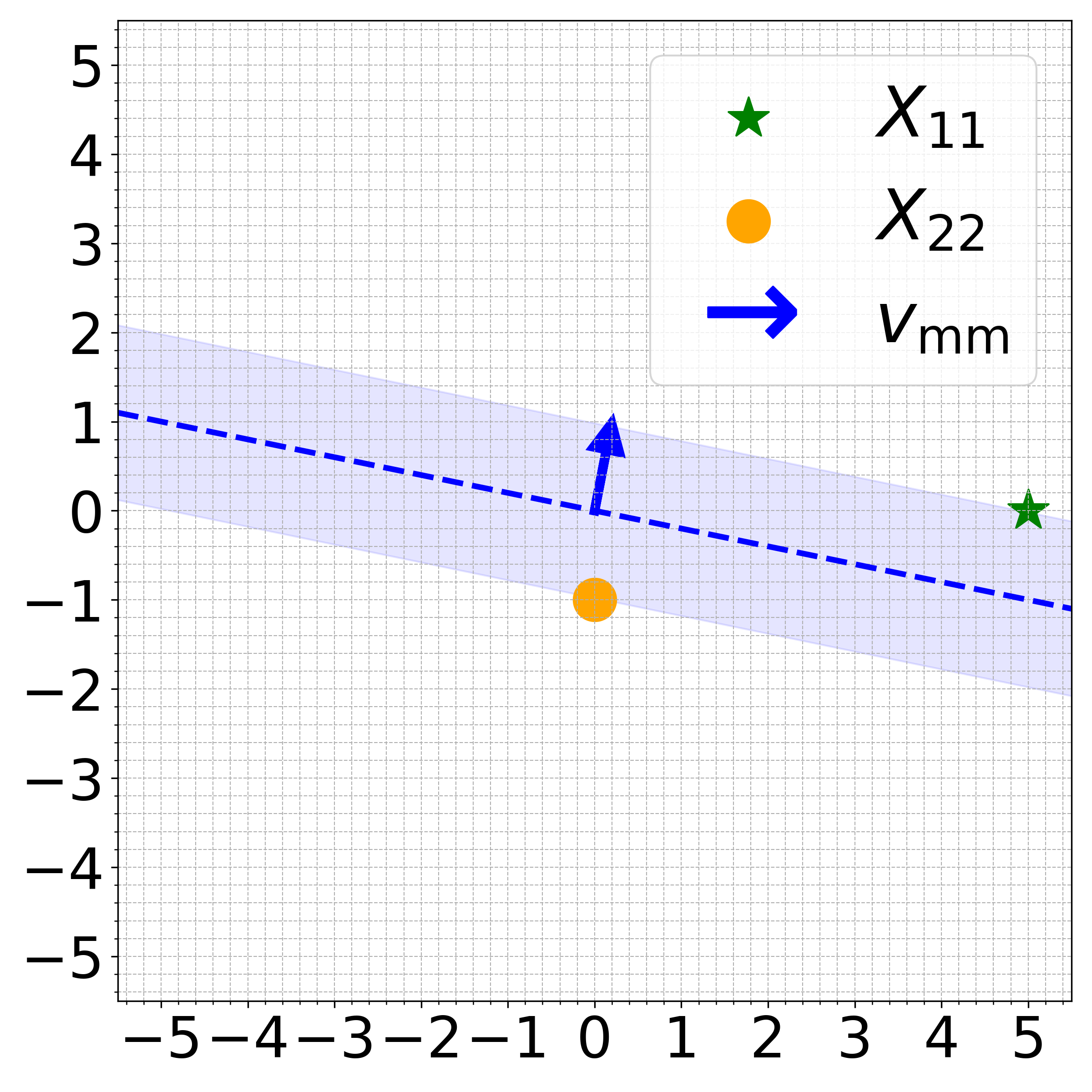}
     \end{subfigure}%
          \hspace{0.1cm}
     \begin{subfigure}[b]{0.235\textwidth}
         \centering
         \includegraphics[width=\textwidth]{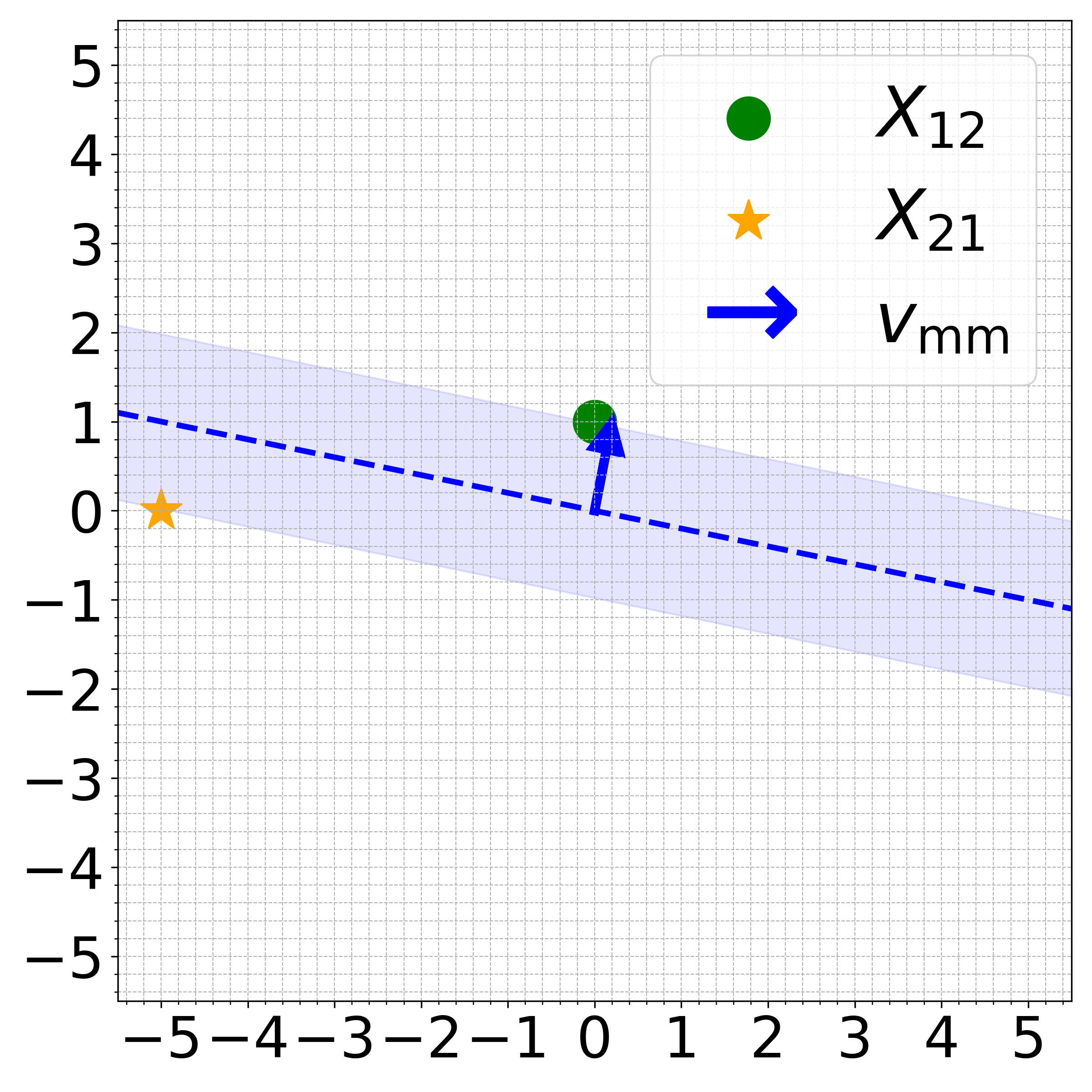}
     \end{subfigure}%
          \hspace{0.1cm}
     \begin{subfigure}[b]{0.235\textwidth}
         \centering
         \includegraphics[width=\textwidth]{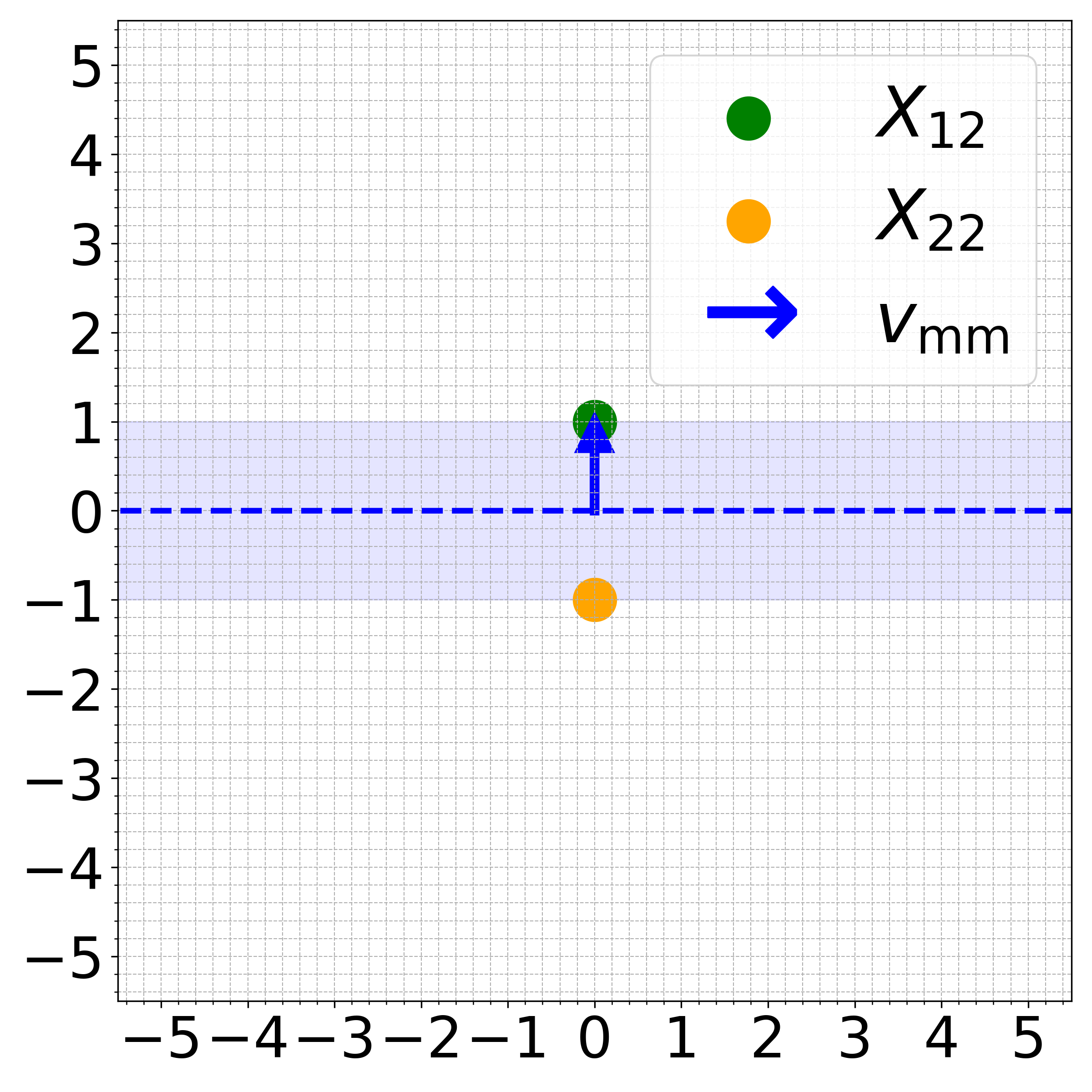}
     \end{subfigure}
\caption{Effect of token selection on margin size in \eqref{eqn:vp:svm} for Example~\ref{exm:svm_token_visualization}. The first plot shows the largest class margin with optimal tokens \(X_{11}\) and \(X_{21}\).  In subsequent plots, as different tokens are used, the class margin (light blue shaded area) decreases, reflecting suboptimal class separation.}
    \label{fig:margin;token}
    \vspace{-.25cm}
\end{figure}

In \eqref{eqn:vp:svm}, define the \emph{label margin} as \(1/\| v_\mathrm{mm} \|_{p} \). The label margin quantifies the distance between the separating hyperplane and the nearest data point in the feature space. A larger label margin indicates better generalization performance of the classifier, as it suggests that the classifier has a greater separation between classes. From \eqref{eqn:vp:svm} and Definitions~\ref{def:token:score} and \ref{def:token:loptimal}, an additional intuition by \cite{ataee2024max} behind optimal tokens is that they maximize the label margin when selected; see Figure~\ref{fig:margin;token}. Selecting the locally optimal token indices $\a = (\alpha_i)_{i=1}^n$ from each input data sequence achieves the largest label margin, meaning that including other tokens will reduce the label margin as defined in \eqref{eqn:vp:svm}. In the Appendix~\ref{sec:app:joint}, we show that \(W\) and \(v\) generated by \ref{alg:rp} converge to their respective max-margin solutions under suitable geometric conditions (Theorem~\ref{thm:joint:rp} in the appendix).

\section{Experimental Results}

\subsection{Synthetic Data Experiments}

We describe the setup of the experiments for ~\ref{alg:p:agd} and ~\ref{alg:pq:agd} and their results.

\subsubsection{\ref{alg:p:agd} Experiment}
To measure the directional distance between \( W^\mathrm{mm}_\a \) (solution of \eqref{eqn:w-svm}) and \( W(k) \) (output of \ref{alg:p:agd}), we use the directional Bregman divergence \( D_\psi(W/\|W\|_{p,p}, V/\|V\|_{p,p}) \) for \( W, V \in \mathbb{R}^{d\times d} \). We compare the \eqref{eqn:w-svm} solution with the $\ell_q$ optimization path for all $p, q \in \{1.75, 2, 3\}$ for synthetically generated data (described in detail in the Appendix). The experiment is repeated 100 times, and the average directional Bregman divergence is reported. A closer look at one sample trial is also provided.

\begin{figure}[h]
     \centering
     \begin{subfigure}[b]{0.32\textwidth}
         \centering
         \includegraphics[width=\textwidth]{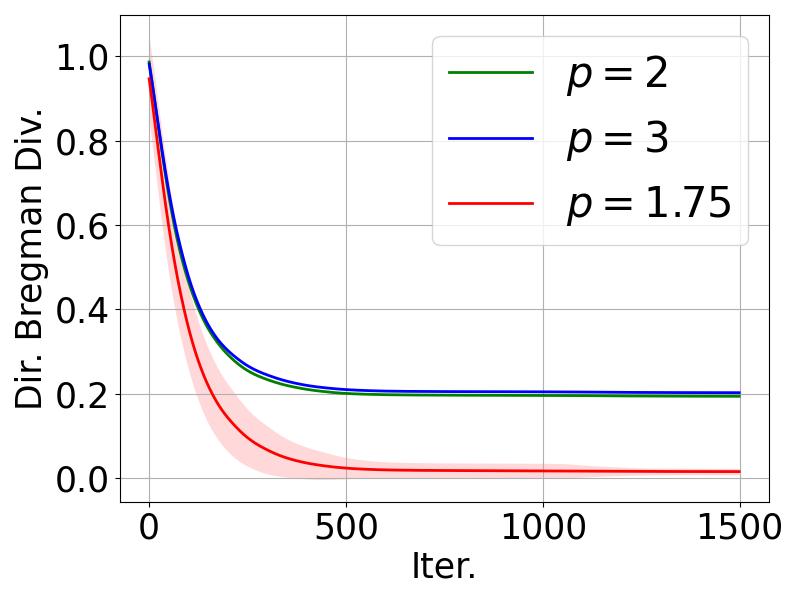}
         \caption{$\ell_{1.75}$ Dir. Bregman Div.}
         \label{fig-corr-1-75}
     \end{subfigure}
     \hfill
     \begin{subfigure}[b]{0.32\textwidth}
         \centering
         \includegraphics[width=\textwidth]{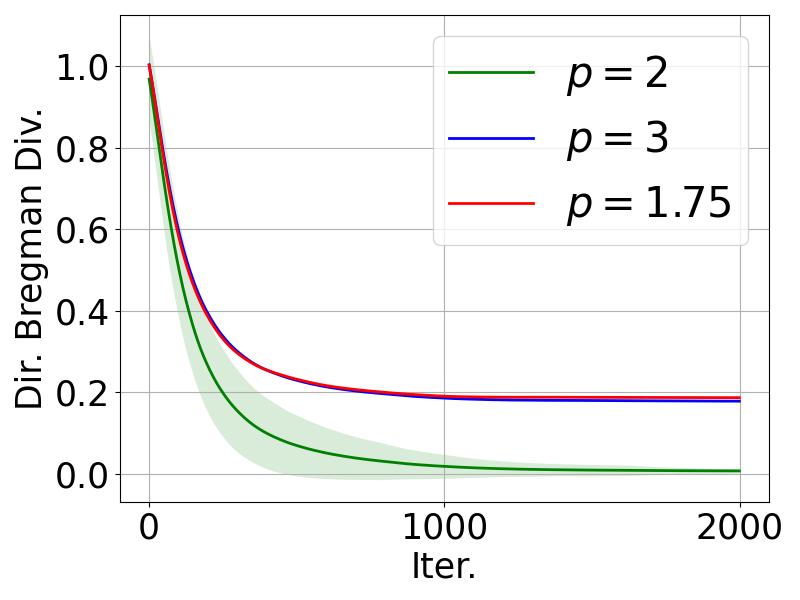}
         \caption{$\ell_2$ Directional Bregman Div.}
         \label{fig-corr-2}
     \end{subfigure}
     \hfill
     \begin{subfigure}[b]{0.32\textwidth}
         \centering
         \includegraphics[width=\textwidth]{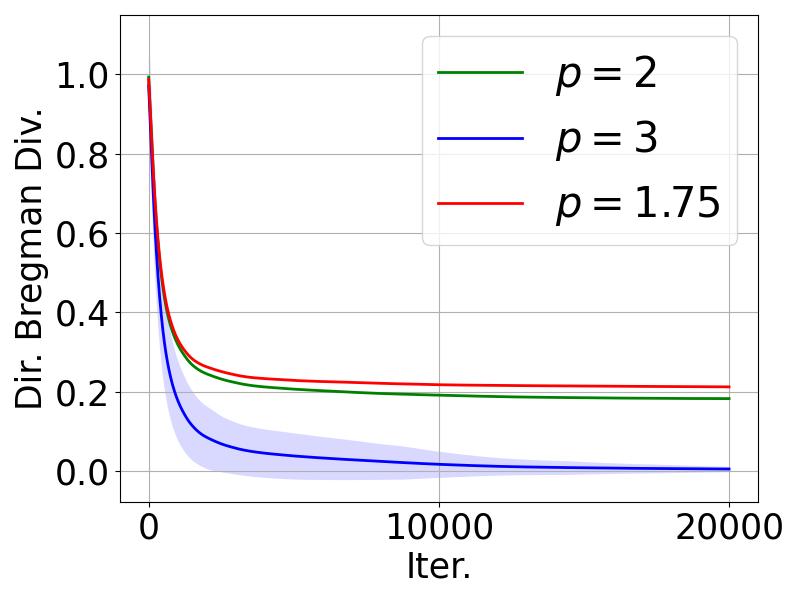}
         \caption{$\ell_3$ Directional Bregman Div.}
         \label{fig-corr-3}
     \end{subfigure}
\caption{Average directional Bregman divergence between the (a) $\ell_{1.75}$, (b) $\ell_2$, and (c) $\ell_3$ optimization paths and the \eqref{eqn:w-svm} solutions for $p=1.75, 2$, and $3$ at each training iteration from 100 trials. The shaded area represents the standard deviation of the directional Bregman divergence.  }
\label{fig-corr}
\end{figure}
\begin{figure}[t]
     \centering
     \begin{subfigure}[b]{0.3\textwidth}
         \centering
         \includegraphics[width=\textwidth]{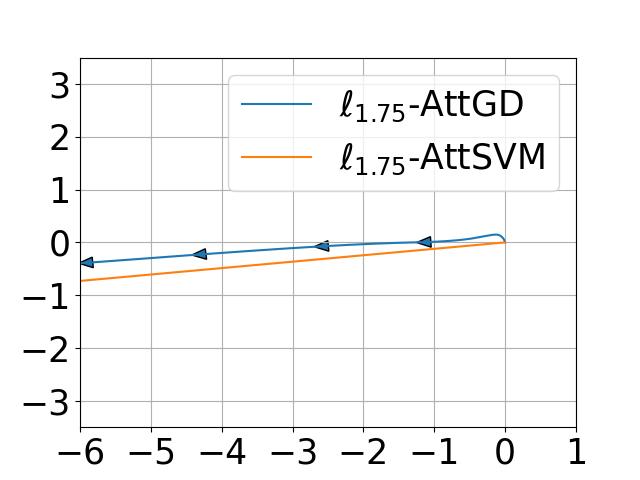}
         \caption{$\ell_{1.75}$ Optimization Path}
         \label{figsvm-inst-1-75}
     \end{subfigure}
     \hfill
     \begin{subfigure}[b]{0.3\textwidth}
         \centering
         \includegraphics[width=\textwidth]{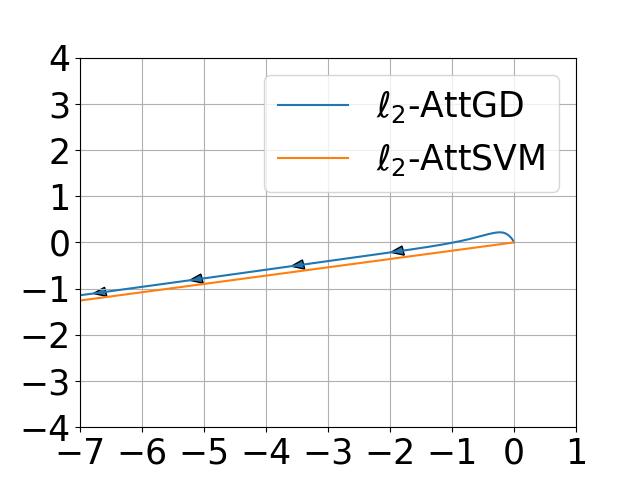}
         \caption{$\ell_2$ Optimization Path}
         \label{figsvm-inst-2}
     \end{subfigure}
     \hfill
     \begin{subfigure}[b]{0.3\textwidth}
         \centering
         \includegraphics[width=\textwidth]{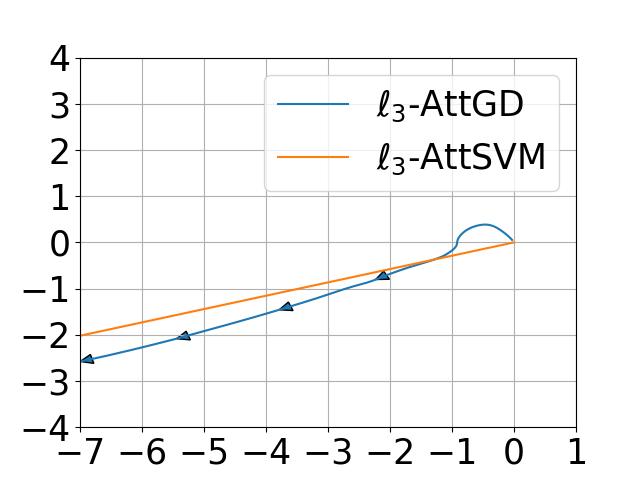}
         \caption{$\ell_3$ Optimization Path}
         \label{figsvm-inst-3}
     \end{subfigure}
\caption{Direction of change of two entries of $W$ updated by \ref{alg:p:agd} with $p=1.75$, $p=2$, and $p=3$ for one trial, shown in (a), (b), and (c). Each axis represents a different entry. The orange line shows the direction of \eqref{eqn:w-svm}.}
\vspace{-.25cm}
\label{fig-inst}
\end{figure}

Figure~\ref{fig-corr} shows the directional Bregman divergence between the \eqref{eqn:w-svm} solution and the $\ell_q$ optimization path for each pair $p, q \in \{1.75, 2, 3\}$. In Figure \ref{fig-corr-1-75}, the divergence converges to $0$ only for the \eqref{eqn:w-svm} ($p=1.75$) solution, indicating that the $\ell_{1.75}$ path does not converge to the $p=2$ or $3$ solutions. The shrinking standard deviation shows consistent behavior. Similarly, Figures \ref{fig-corr-2} and \ref{fig-corr-3} show the divergence converging to $0$ for the corresponding \eqref{eqn:w-svm} solution, demonstrating that the $\ell_p$ optimization path converges to the \eqref{eqn:w-svm} solution, with the direction of convergence changing with $p$. In addition to the directional Bregman divergence, we can also observe the convergence in direction for one trial (initialized at zero) directly by plotting how two of the entries of $W$ change during training simultaneously and plotting it on a Cartesian graph, then showing that the path it follows converges to the direction of the \eqref{eqn:w-svm} solution. As we can see in Figure \ref{fig-inst}, each of the $\ell_p$ optimization paths follows the direction of the corresponding \eqref{eqn:w-svm} solution.

\subsubsection{\ref{alg:pq:agd} Experiments} We use the data from Example ~\ref{example:p-gd} in Appendix ~\ref{app:exp:joint} to train a model using ~\ref{alg:pq:agd} for $p=1.75,2,$ and $3$.  The comparison between the iterates and the SVM solutions in Figure \ref{fig-joint-iterates} shows that the iterates of $W$ and $v$ converge to the ~\ref{eqn:w-svm} and ~\ref{eqn:vp:svm} directions, respectively, for each of $p=1.75,2,$ and $3$. These convergence are similar to Theorem \ref{thm:joint:rp}, as in both this experiment and that theorem, we get that the iterates converge to the SVM problem solutions. In addition to these iterates, we record the evolution of the average softmax probability of the optimal token, along with the average logistic probability of the model, which we define to be $1/n \sum_{i=1}^{n} 1/(1+\exp(-\gamma_i a_i))
$.

\begin{figure}[t]
     \centering
     \begin{subfigure}[b]{0.32\textwidth}
         \centering
         \includegraphics[width=\textwidth]{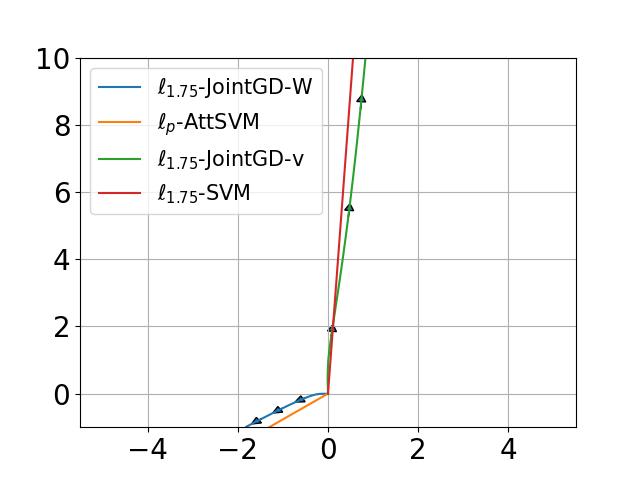}
         \caption{$\ell_{1.75}$ Optimization Paths}
         \label{figjoint-inst-1-75}
     \end{subfigure}
     \hfill
     \begin{subfigure}[b]{0.32\textwidth}
         \centering
         \includegraphics[width=\textwidth]{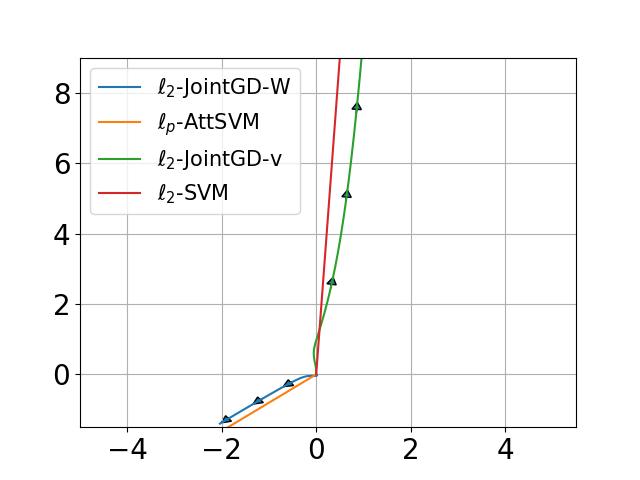}
         \caption{$\ell_2$ Optimization Paths}
         \label{figjoint-inst-2}
     \end{subfigure}
     \hfill
     \begin{subfigure}[b]{0.32\textwidth}
         \centering
         \includegraphics[width=\textwidth]{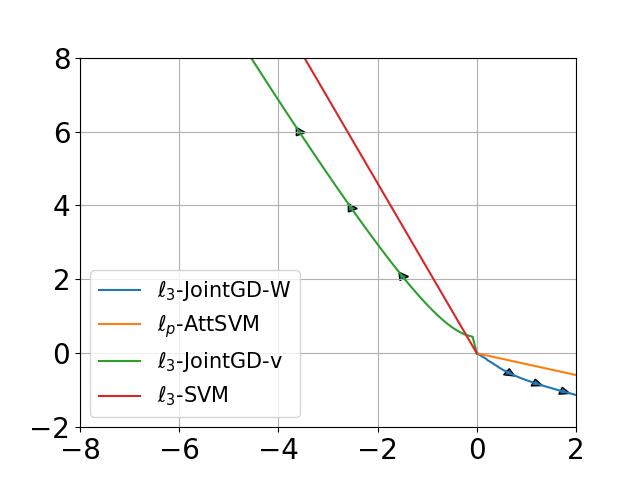}
         \caption{$\ell_3$ Optimization Paths}
         \label{figjoint-inst-3}
     \end{subfigure}
\caption{Iterates of the $W$ and $v$ parameters of the model as they are trained using ~\ref{alg:pq:agd} for $p=1.75,2,$ and $3$, along with the corresponding ~\ref{eqn:w-svm} and ~\ref{eqn:vp:svm} directions.}
        \label{fig-joint-iterates}
\end{figure}
\begin{figure}[t]
     \centering
     \begin{subfigure}[b]{0.32\textwidth}
         \centering
         \includegraphics[width=\textwidth]{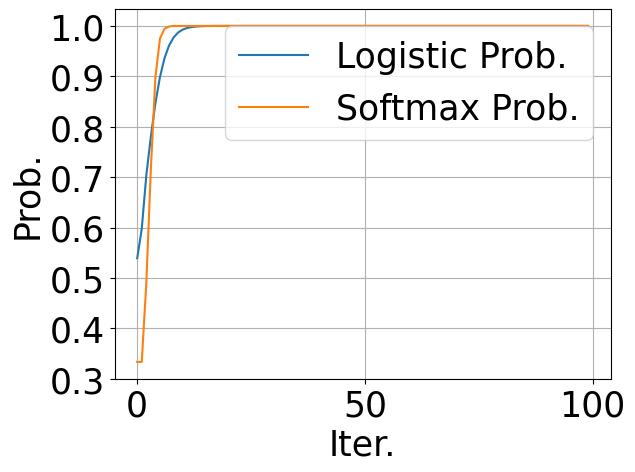}
         \caption{$\ell_{1.75}$ Probabilties}
         \label{figprob-inst-1-75}
     \end{subfigure}
     \hfill
     \begin{subfigure}[b]{0.32\textwidth}
         \centering
         \includegraphics[width=\textwidth]{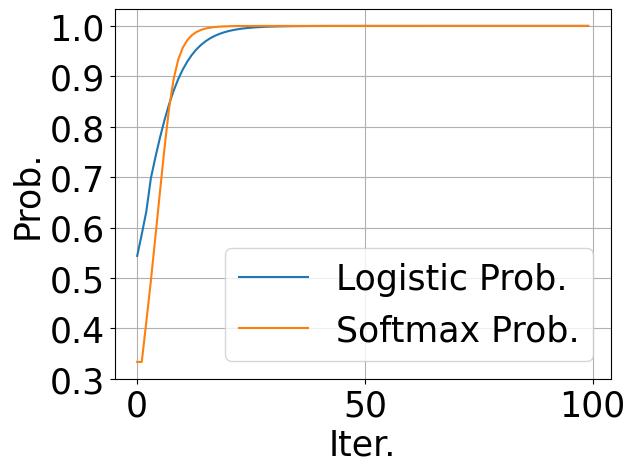}
         \caption{$\ell_2$ Probabilties}
         \label{figprob-inst-2}
     \end{subfigure}
     \hfill
     \begin{subfigure}[b]{0.32\textwidth}
         \centering
         \includegraphics[width=\textwidth]{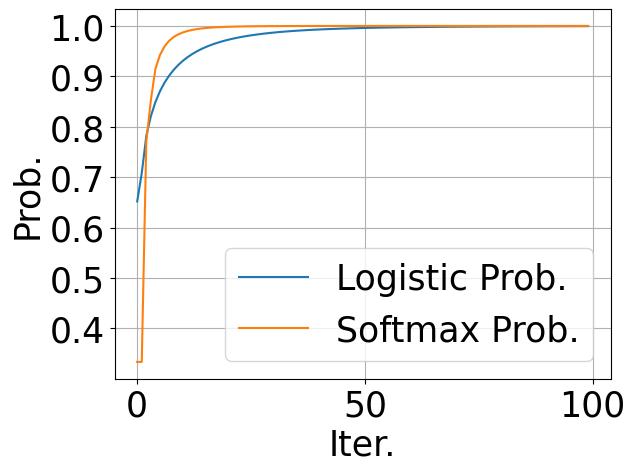}
         \caption{$\ell_3$ Probabilties}
         \label{figprob-inst-3}
     \end{subfigure}
\caption{Softmax probability evolution of the optimal token and logistic probability evolution for $p=1.75,2,$ and $3$.}
        \label{fig-prob}
\end{figure}
As shown in Figure~\ref{fig-prob}, each average softmax probability converges to $1$, indicating that the attention mechanism produces a softmax vector converging to a one-hot vector during different ~\ref{alg:pq:agd} training. Moreover, the average logistic probability also converges to $1$, indicating the model's prediction reaches $100\%$ accuracy.

\subsection{Real Data Experiments}\label{sec:real-data-exp}
This section presents evidence of improved generalization, token selection, and explainability from training an attention network with \ref{eqn:md} instead of other algorithms such as GD and Adam \citep{kingma2014adam}. Specifically, we compare the generalization, token selection, and model weight distribution of the \ref{eqn:md} and GD algorithms in training a transformer classification model on the Stanford Large Movie Review Dataset; and the generalization and model weight distribution of the \ref{eqn:md} and Adam in training a vision transformer (ViT) model \citep{dosovitskiy2020image} on CIFAR-10.

\subsubsection{Comparison with Gradient Descent}

\begin{table}[h!]
\centering
\begin{tabular}{||c|c|c|c||} 
 \hline
  &  & & \\ [-1.4ex] 
 Algorithm & Model Size 3 & Model Size 4 & Model Size 6 \\ [0.5ex] 
 \hline\hline
   &  & & \\ [-1.5ex] 
 $\ell_{1.1}$-MD & \textbf{83.47 $\pm$ 0.09\%} & \textbf{83.36 $\pm$ 0.13\%} & \textbf{83.65 $\pm$ 0.13\%} \\
   &  & & \\ [-1.1ex] 
 $\ell_{2}$-MD  & $81.66\pm0.09\%$ & $81.05\pm0.17\%$ & $82.22\pm0.13\%$ \\
   &  & & \\ [-1.1ex] 
 $\ell_{3}$-MD & $82.57\pm0.09\%$ & $82.40\pm0.12\%$ & $81.97\pm0.10\%$ \\[0.5ex] 
 \hline
\end{tabular}
\caption{Test accuracies of transformer classification models trained with $\ell_{1.1}$, $\ell_2$, and $\ell_3$-\ref{eqn:md} on the \textbf{Stanford Large Movie Review Dataset}. The model sizes refers to the number of layers in the transformer model and the number of attention heads per layer.  $\ell_{1.1}$-\ref{eqn:md} provides superior generalization performance. }
\label{table:1}
\vspace{-0.3cm}
\end{table}
We trained a transformer classification model on the Stanford Large Movie Review Dataset \citep{maas-EtAl:2011:ACL-HLT2011} using \ref{eqn:md} with $\ell_{1.1}$, $\ell_2$, and $\ell_3$ potentials. The models are similar to the encoder module in \cite{vaswani2017attention}, with the last layer being a linear classification layer on the feature representation of the first $[\texttt{CLS}]$ token. We put the details of the classification model in Appendix \ref{app:exp:stanford}. Table \ref{table:1} summarizes the resulting test accuracy of several variants of that model when trained with the three algorithms, which shows that the $\ell_{1.1}$ potential \ref{eqn:md} outperforms the other \ref{eqn:md} algorithms, including the one with the $\ell_2$ potential, which is equivalent to the GD. %
\begin{figure}[t]
    \centering
    \includegraphics[width=1\linewidth]{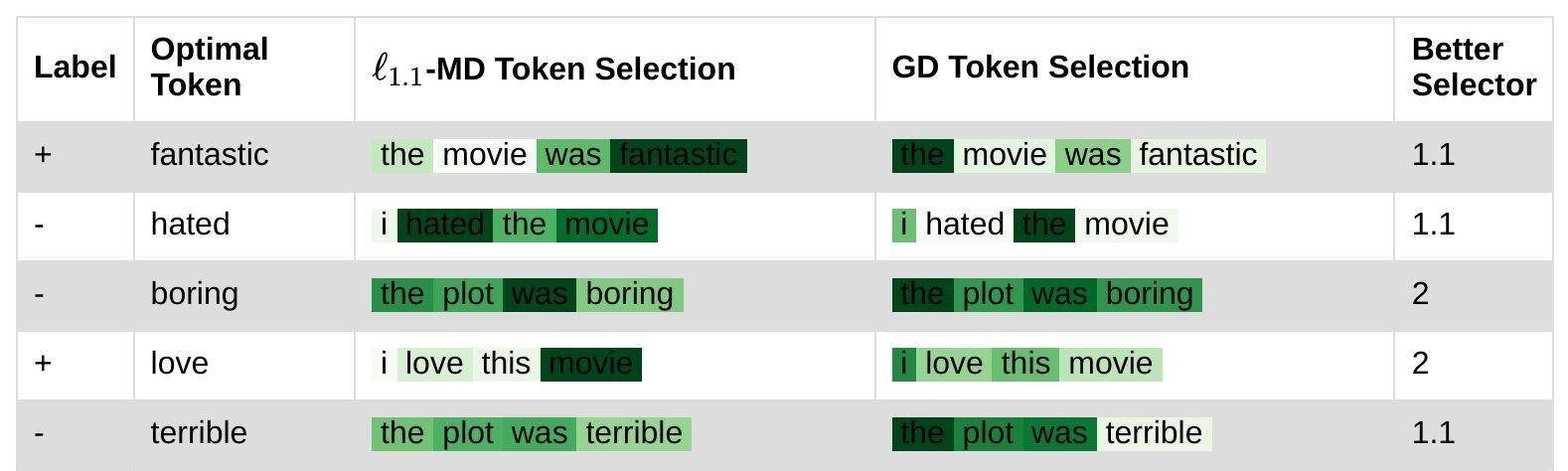}
    \caption{The attention map generated by the resulting models that were trained using $\ell_{1.1}$--\ref{eqn:md} and GD for five sample sentences. For three out of five of the sample sentences, the model trained using $\ell_{1.1}$--\ref{eqn:md} selects the optimal token better than the model trained using GD.}
    \label{img:attention_map}
\end{figure}
To investigate this further, we look at how the attention layers of the model select the tokens from simple reviews that GPT-4o generated and investigate how much the attention layer focuses on a particular token that truly determines whether the whole review was a positive one or a negative one. We chose these pivotal tokens using GPT-4o as well. We do this procedure to the model trained using $\ell_{1.1}$--\ref{eqn:md} and the GD and tabulate the full results in Appendix~\ref{app:exp:detail} (we provide five of the results in Figure \ref{img:attention_map}). We can see that the $\ell_{1.1}$--\ref{eqn:md} also outperforms the GD in token selection.

\begin{figure}[h]
\centering
\begin{tabular}{cccc}
\subfloat[$W_K$ parameters with $\ell_{1.1}$-\ref{eqn:md}]{\includegraphics[width =0.31\linewidth]{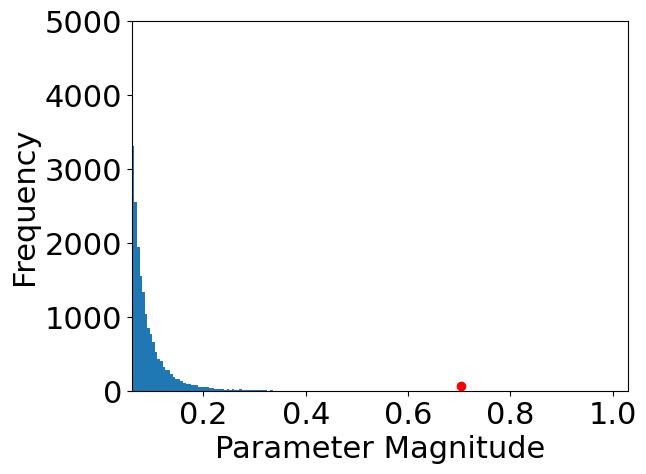}} 
\hspace{-.25cm}
&
\subfloat[$W_Q$ parameters with $\ell_{1.1}$-\ref{eqn:md}]{\includegraphics[width = 0.31\linewidth]{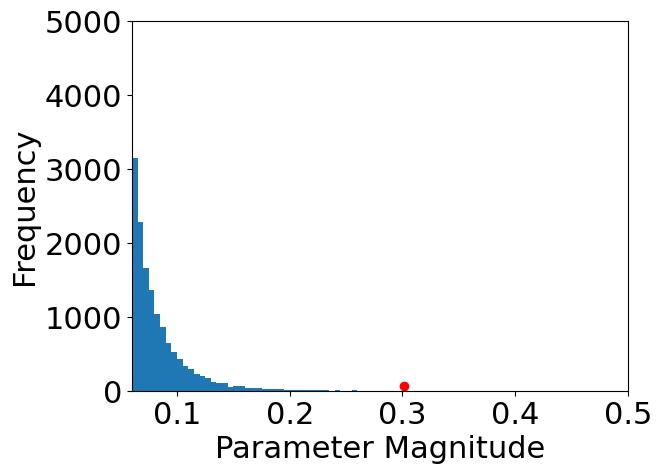}}
\hspace{-.3cm}
&
\subfloat[$W_V$ parameters with $\ell_{1.1}$-\ref{eqn:md}]{\includegraphics[width = 0.3\linewidth]{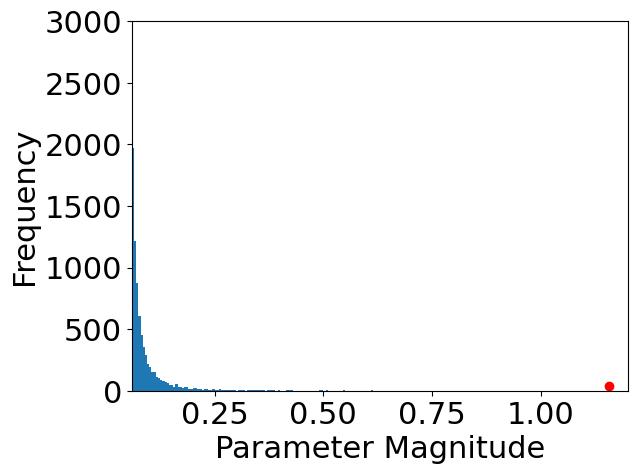}} \vspace{.3cm}\\
\subfloat[$W_K$ parameters with $\ell_{2}$-\ref{eqn:md}]{\includegraphics[width = 0.31\linewidth]{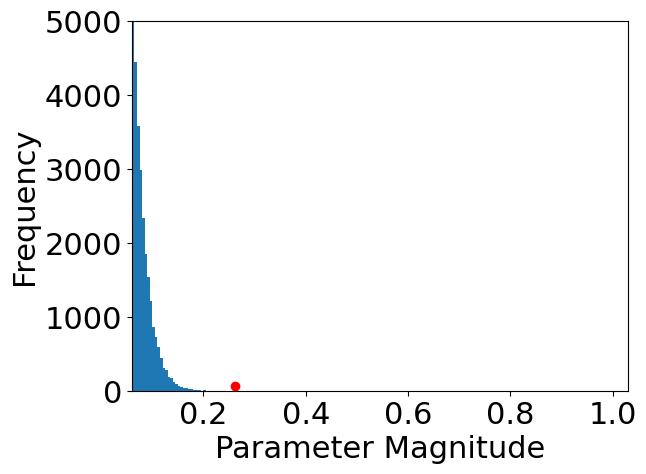}} &
\subfloat[$W_Q$ parameters with $\ell_{2}$-\ref{eqn:md}]{\includegraphics[width = 0.31\linewidth]{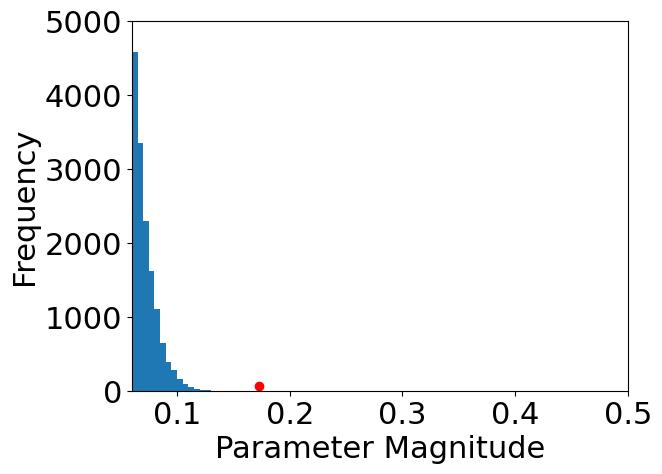}} &
\subfloat[$W_V$ parameters with $\ell_{2}$-\ref{eqn:md}]{\includegraphics[width = 0.3\linewidth]{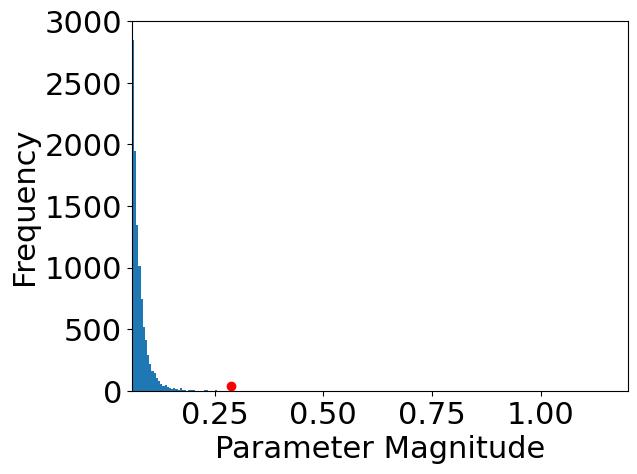}} \\
\end{tabular}
 \caption{Histogram of the absolute values of the $W_K,W_Q,$ and $W_V$ components of transformer models trained with $\ell_{1.1}$ and $\ell_{2}$-\ref{eqn:md} on the \textbf{Stanford Large Movie Review Dataset}. Only large parameters ($\geq 0.06$) are shown, with the maximum magnitude component marked by a red dot. The $\ell_{1.1}$-\ref{eqn:md} model has $18,206$ components in $W_K$, $13,964$ in $W_Q$, and $7,643$ in $W_V$ with magnitudes $\geq 0.06$, while the $\ell_2$-\ref{eqn:md} model has $27,224$ in $W_K$, $14,654$ in $W_Q$, and $10,127$ in $W_V$ with such magnitudes. These results imply that the $\ell_{1.1}$-\ref{eqn:md} algorithm yields sparser parameters and that it would have a stronger token selection ability.}
\label{figure:hist}
    \vspace{-0.3cm}
\end{figure}

Finally, we collect the training weights from the resulting models trained by $\ell_{1.1}$--\ref{eqn:md} and the GD and plot a histogram of their absolute values in Figure \ref{figure:hist}. Specifically, we take the histogram of the components of the key, query, and value matrices. The figures show that the resulting model that was trained using $\ell_{1.1}$--\ref{eqn:md} is sparser than the one trained using GD, which could hint at a potential explanation as to why $\ell_{1.1}$--\ref{eqn:md} can outperform the standard GD algorithm when it is used to train attention-based models.

\subsubsection{Comparison with Adam}\label{sec:exp:adam}
We compare the {$\ell_{1.1},\ell_{1.75},\ell_2,\ell_3$--\ref{eqn:md}, and Adam algorithms} in training a ViT model on the CIFAR-10 dataset. We specify the ViT model architecture in Appendix \ref{app:exp:adam}.

The resulting test accuracies for all algorithms are reported in Figure~\ref{experiment-adam}. Notably, all five optimization methods---spanning the $\ell_p$-MD family with $p\in\{1.1, 1.75, 2, 3\}$ and Adam---converge to approximately $82\%$ test accuracy on CIFAR-10, demonstrating that our $\ell_p$-MD framework achieves state-of-the-art performance competitive with widely-used adaptive optimizers. {The consistent performance across different values of $p$ confirms that the beneficial sparsity and interpretability properties induced by the $\ell_p$ geometry (particularly for $p$ close to 1) come at no cost to generalization capability.}

\begin{figure}[t]
\centering
\begin{tabular}{cc}
\subfloat[Test accuracy of ViT on CIFAR-10 trained with $\ell_{1.1}$, $\ell_{1.75}$, $\ell_2$, $\ell_3$-MD and Adam.\label{experiment-adam}]{\includegraphics[width=0.48\textwidth]{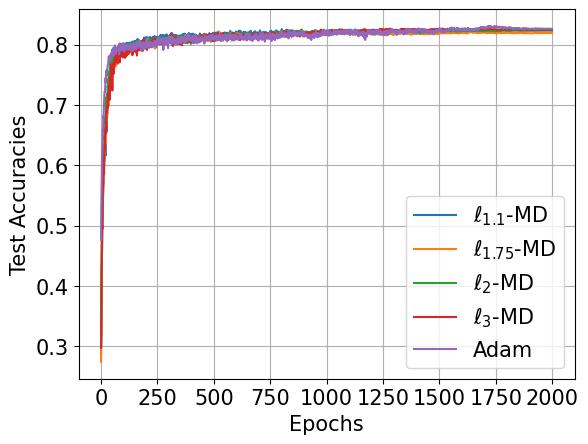}} 
&
\subfloat[Test accuracy of ViT on CIFAR-100 trained with $\ell_{1.1}$, $\ell_{1.75}$, $\ell_2$, $\ell_3$-MD and Adam.\label{fig:response-sec4-test-acc}]{\includegraphics[width=0.48\textwidth]{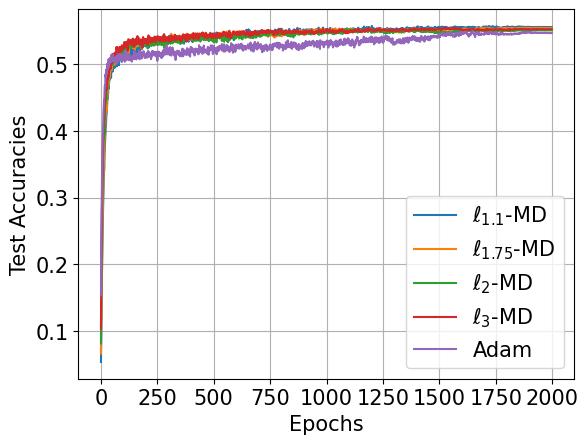}}
\end{tabular}
\caption{Test accuracy comparison across different optimizers on CIFAR-10 and -100.}
\label{fig:cifar-comparison}
\end{figure}

Furthermore, we examine the learned weight distributions by plotting the absolute values of attention layer parameters (key, query, and value matrices) from models trained with different optimizers in Figure~\ref{figure:hist-adam}. Consistent with our earlier observations in Figure~\ref{figure:hist}, $\ell_{1.1}$-MD produces markedly sparser attention weights compared to Adam, with a substantially higher concentration of near-zero parameters. {This sparsity pattern reveals that $\ell_{1.1}$-MD implicitly performs feature selection during training, concentrating representational capacity on the most discriminative attention components while driving less relevant weights toward zero. Such induced sparsity not only enhances model interpretability by identifying which attention mechanisms are critical for the task, but also enables more aggressive pruning and compression strategies without sacrificing performance---a practical advantage for deployment in resource-constrained environments.}

\begin{figure}[h]
\centering
\begin{tabular}{cccc}
\subfloat[$W_K$ parameters with $\ell_{1.1}$-\ref{eqn:md}]{\includegraphics[width =0.31\linewidth]{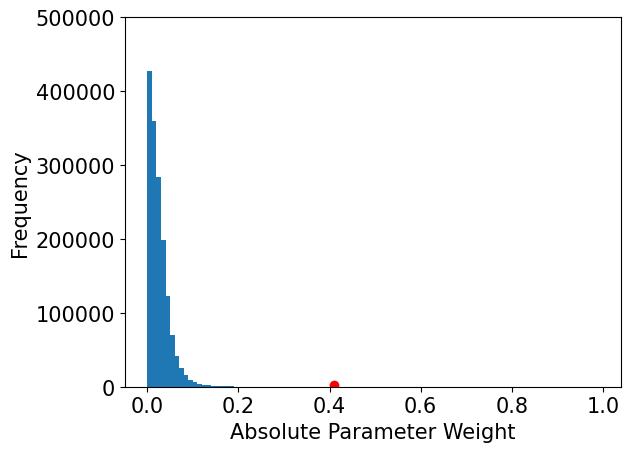}} 
\hspace{-.25cm}
&
\subfloat[$W_Q$ parameters with $\ell_{1.1}$-\ref{eqn:md}]{\includegraphics[width = 0.31\linewidth]{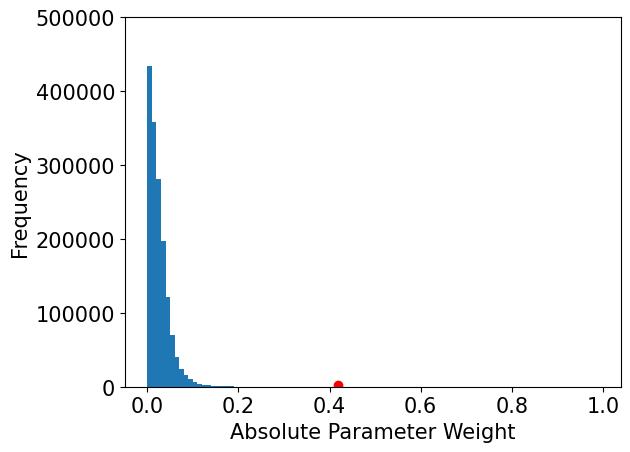}}
\hspace{-.3cm}
&
\subfloat[$W_V$ parameters with $\ell_{1.1}$-\ref{eqn:md}]{\includegraphics[width = 0.3\linewidth]{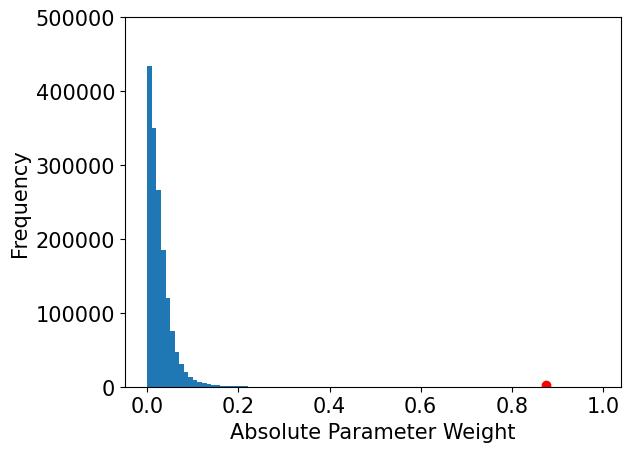}} \vspace{.3cm}\\
\subfloat[$W_K$ parameters with Adam]{\includegraphics[width = 0.31\linewidth]{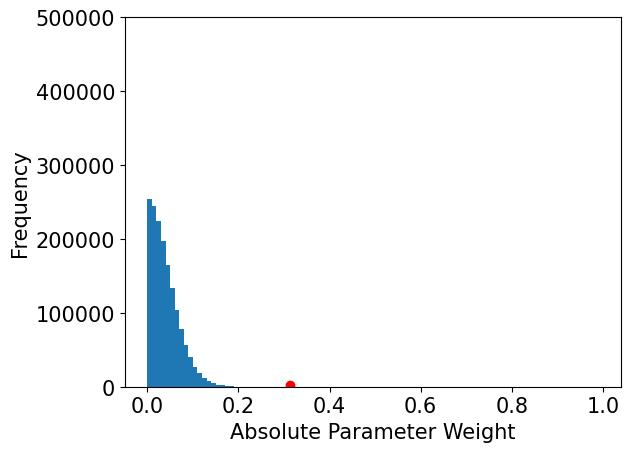}} &
\subfloat[$W_Q$ parameters with Adam]{\includegraphics[width = 0.31\linewidth]{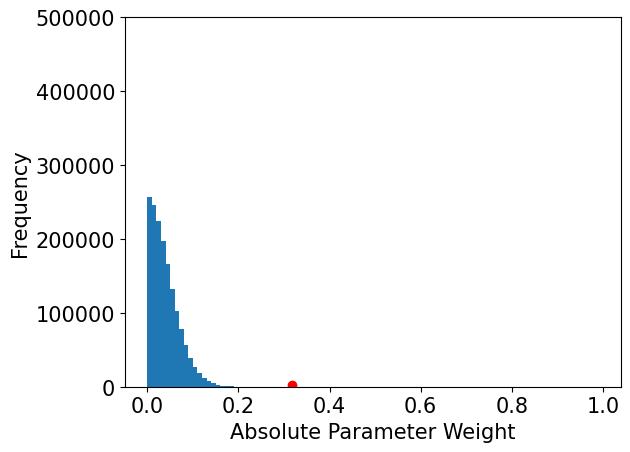}} &
\subfloat[$W_V$ parameters with Adam]{\includegraphics[width = 0.3\linewidth]{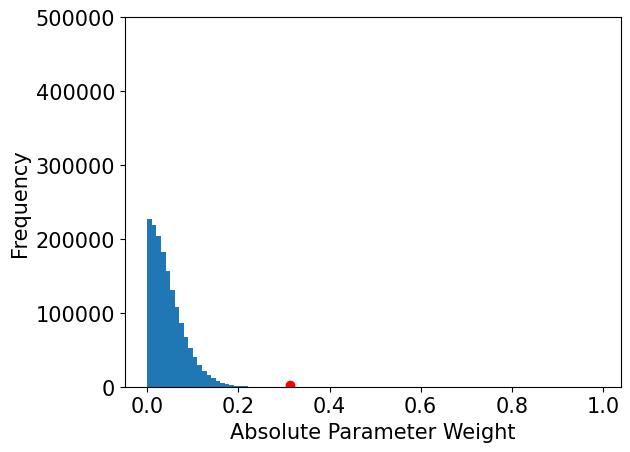}} \\
\end{tabular}
 \caption{Histogram of the absolute values of the $W_K,W_Q,$ and $W_V$ components of ViT models trained with $\ell_{1.1}$-\ref{eqn:md} and Adam on CIFAR-10. These results show that $\ell_{1.1}$-\ref{eqn:md} can be more explainable compared to Adam because it produces sparser parameters, which should induce better token selection.}
\label{figure:hist-adam}
\end{figure}

{

Similarly, we also compare these algorithms in training the same ViT architecture on CIFAR-100, a more complex dataset.

The results in Figure~\ref{fig:response-sec4-test-acc} show that our \ref{eqn:md} algorithms generally achieve test scores of around $55\%$, which are at a similar level to those of other popular training algorithms, such as SGD and Adam, on CIFAR-100, further demonstrating that our algorithms can compete with the state-of-the-art. Again, we plot histograms of the absolute values of the key, query, and value weights in the attention layers of the resulting models that were trained using Adam and $\ell_{1.1}$--\ref{eqn:md}. Figure~\ref{fig:response-sec4-hist} shows these histograms, and as we can see, the resulting model trained by $\ell_{1.1}$--\ref{eqn:md} has significantly more attention weight entries that are close to zero when compared to the model trained by Adam, making the attention weights sparser.

\begin{figure}[h]
\centering
\begin{tabular}{cccc}
\subfloat[$W_K$ parameters with $\ell_{1.1}$-\ref{eqn:md}]{\includegraphics[width =0.31\linewidth]{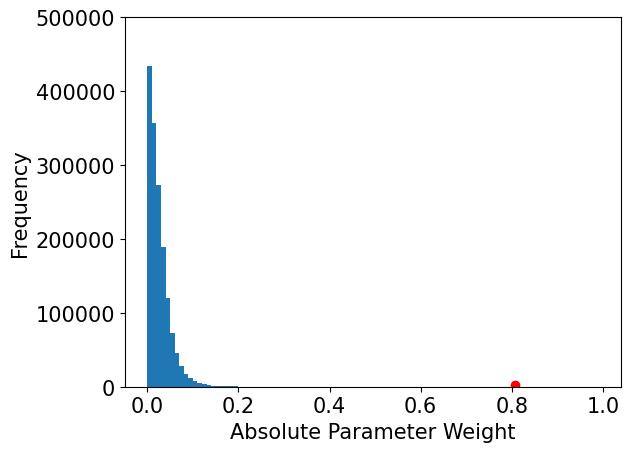}} 
\hspace{-.25cm}
&
\subfloat[$W_Q$ parameters with $\ell_{1.1}$-\ref{eqn:md}]{\includegraphics[width = 0.31\linewidth]{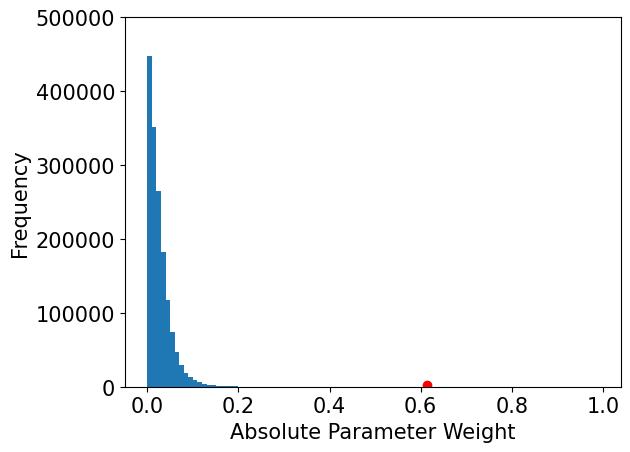}}
\hspace{-.3cm}
&
\subfloat[$W_V$ parameters with $\ell_{1.1}$-\ref{eqn:md}]{\includegraphics[width = 0.3\linewidth]{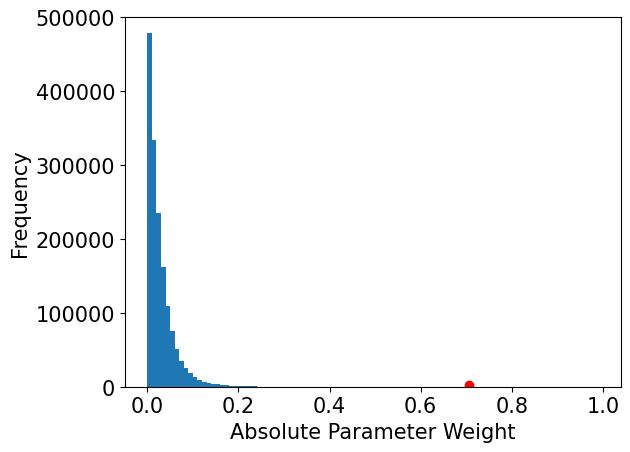}} \vspace{.3cm}\\
\subfloat[$W_K$ parameters with Adam]{\includegraphics[width = 0.31\linewidth]{Images/251013_cifar100_adam_wk_weights.jpg}} &
\subfloat[$W_Q$ parameters with Adam]{\includegraphics[width = 0.31\linewidth]{Images/251013_cifar100_adam_wq_weights.jpg}} &
\subfloat[$W_V$ parameters with Adam]{\includegraphics[width = 0.3\linewidth]{Images/251013_cifar100_adam_wv_weights.jpg}} \\
\end{tabular}
 \caption{Histogram of the absolute values of the entries in the $W_K,W_Q$, and $W_V$ attention weight matrices of the resulting models trained by $\ell_{1.1}$--\ref{eqn:md} and Adam on CIFAR-100. We can see that the weights from the model that was trained by $\ell_{1.1}$--\ref{eqn:md} is more sparse. The red dot represent the maximum absolute values of the respective weight matrices.}
\label{fig:response-sec4-hist}
\end{figure}
}

\section{Conclusion and Future Work}\label{sec:concl}
We explored the optimization dynamics and generalization performance of a family of \ref{eqn:md} 
algorithms for softmax attention mechanisms, focusing on \ref{alg:p:agd}, which generalizes GD by using the $p$-th power of the $\ell_p$-norm as the potential function. Our theoretical analysis and experiments show that \ref{alg:p:agd} converges to the solution of a generalized hard-margin SVM with an $\ell_p$-norm objective in classification tasks using a single-layer softmax attention model. This generalized SVM separates optimal from non-optimal tokens via linear constraints on token pairs. We also examined the joint problem under logistic loss with $\ell_p$-norm regularization and proved that \(W\) and \(v\) generated by $\ell_p$-norm regularization path converge to their respective generalized max-margin solutions. Finally, our numerical experiments on real data demonstrate that \ref{eqn:md} algorithms improve generalization over standard GD and excel in optimal token selection. 

{Our theoretical guarantees are established for single-head, one-layer attention models—a standard abstraction that balances analytical tractability with capturing the essential nonconvex structure of softmax attention. An important direction for future work is extending these convergence results to multi-head and multi-layer architectures, where each head solves separation problems in different subspaces and stacked layers progressively refine feature representations.}


\acks{The authors acknowledge the MIT SuperCloud and Lincoln Laboratory Supercomputing Center for providing computing resources that have contributed to the results reported within this paper. This work was supported in part by MathWorks, Jane Street, the MIT-IBM Watson AI Lab, the MIT-Amazon Science Hub, and the MIT-Google Program for Computing Innovation. The authors also thank Roey Magen for valuable feedback on the joint regularization path analysis.
}


\bibliography{jmlr_main}

\appendix



\newpage
\addtocontents{toc}{\protect\setcounter{tocdepth}{3}}
\tableofcontents

\begingroup


\section{Auxiliary Lemmas}

\subsection{Additional Notations}

Consider the following constants for the proofs, depending on the dataset \((X_i, Y_i, z_i)_{i=1}^n\), the parameter \(v\), and the locally optimal token \((\alpha_i)_{i=1}^n\):
\begin{subequations}\label{eqn:def:delta}
\begin{align}
\nonumber
\delta' &:= \frac{1}{2} \min_{i \in [n]} \min_{\tau \in \bar{\mathcal{T}}_i} \left( (X_{i \alpha_i} - X_{i \tau})^\top W^\alpha_\mathrm{mm} z_i - 1 \right) \\
&\leq \frac{1}{2} \min_{i \in [n]} \min_{t \in \mathcal{T}_i, \tau \in \bar{\mathcal{T}}_i} \left( (X_{it} - X_{i \tau})^\top W^\alpha_\mathrm{mm} z_i \right);    \\
&\delta := \min \{0.25, \delta'\}.
\end{align}
\end{subequations}
When $\bar{\mathcal{T}}_i=\emptyset$ for all $i\in[n]$ (i.e.~globally-optimal indices), we set $\delta'=\infty$ as all non-neighbor related terms will disappear. Further, recalling Definition~\ref{def:w:attsvm} and using $W^\alpha_\mathrm{mm}$---i.e., the minimizer of \eqref{eqn:w-svm}, we set
\begin{align}\label{eqn:def:A}
\nonumber
A' &:= \|W_\mathrm{mm}^\alpha\|_{p,p} \max_{i \in [n], t \in [T]} \|X_{it} z_i^\top\|_{\frac{p}{p-1},\frac{p}{p-1}};
\\
A &:= \max \{1, A'\}.
\end{align}
Recalling Definition~\ref{def:cone:ellp}, we provide the following initial radius \(\mu=\mu_0\) which will be used later in Lemma \ref{lem:obj:grad-corr-bound-1}:  
\begin{equation}\label{eqn:def:mu0}
\mu_0 := \left\{
\begin{aligned}
    &\frac{1}{p} \left( \frac{\delta}{8A} \right)^p & \textnormal{if } p \geq 2, \\
    &\frac{1}{p} \left( \frac{\delta(p-1)}{4A d^{\frac{2}{p}-1}} \right)^2 & \textnormal{otherwise}.
\end{aligned}
\right.
\end{equation}
Furthermore, define the following sums for \(W\):
\[
S_i(W) := \sum_{t \in \mathcal{T}_i} [\sigma(X_i W z_i)]_t,  \quad \textnormal{and} \quad  
Q_i(W) := \sum_{t \in \bar{\mathcal{T}}_i} [\sigma(X_i W z_i)]_t.
\]
For the samples \(i\) with non-empty supports \(\mathcal{T}_i\), let
\begin{align}
\gamma_i^{\text{gap}} := \gamma_{i \alpha_i} - \max_{t \in \mathcal{T}_i} \gamma_{it},
\quad \textnormal{and} \quad  
\bar{\gamma}_i^{\text{gap}} := \gamma_{i \alpha_i} - \min_{t \in \mathcal{T}_i} \gamma_{it}.
\end{align}
Furthermore, we define the \textit{global score gap} as
\begin{align}
\Gamma &:= \sup_{i \in [n], t, \tau \in [T]} |\gamma_{it} - \gamma_{i\tau}|.
\end{align}

\subsection{Lemma for Analyzing The $\ell_p$-Norm}

In this section of the Appendix, we provide some analysis on comparing the $\ell_p$-norm, the $\ell_p$ Bregman divergence, and the $\ell_2$-norm of matrices. Since the $\ell_2$-norm of matrices are much easier to analyze and use, like in the inner product Cauchy-Schwarz inequality, having this comparison is valuable when analyzing the \ref{alg:p:agd}.

\begin{lemma}\label{lem:norm:2p-bound}
    For any \(d \times d\) matrix \(W\), let \(w\) denote its vectorization. Then, 
\begin{align*}
\|w\|_p \in \left[d^{\frac{2}{p}-1}\|w\|_2, \|w\|_2\right]    
\end{align*}    
for \(p \geq 2\), and for \(1 < p \leq 2\), \(\|w\|_p\) is in a similar interval, with the two ends switched.
\end{lemma}

\begin{proof}
    Let $w_1,w_2,...,w_{d^2}$ be the entries of $w$. Therefore, for $p\geq2$,
    \begin{align*}
        \|w\|_p
        &=\sqrt[p]{\sum_{i=1}^{d^2}|w_i|^p}\\
        &=\sqrt[p]{\sum_{i=1}^{d^2}(|w_i|^2)^{p/2}},
    \end{align*}
    and because $\frac{p}{2}\geq 1$, we would have
    \begin{align*}
        \sqrt[p]{\sum_{i=1}^{d^2}(|w_i|^2)^{p/2}}
        &\leq\sqrt[p]{\left(\sum_{i=1}^{d^2}|w_i|^2\right)^{p/2}}\\
        &=\sqrt[p]{\|w\|_2^p}=\|w\|_2.
    \end{align*}
    Therefore, $\|w\|_p\leq\|w\|_2$ whenever $p\geq 2$. A similar argument will get us $\|w\|_p\geq\|w\|_2$ whenever $1<p\leq2$, so one end of the interval is solved for each case, now for the other end.
    
    Using the power-mean inequality, we can get that whenever $p\geq 2$,
    \[\sqrt[p]{\frac{1}{d^2}\sum_{i=1}^{d^2}|w_i|^p}\geq\sqrt{\frac{1}{d^2}\sum_{i=1}^{d^2}|w_i|^2},\]
    \[d^{-\frac{2}{p}}\|w\|_p\geq d^{-1}\|w\|_2,\]
    \[\|w\|_p\geq d^{\frac{2}{p}-1}\|w\|_2.\]
    Similarly, for $1<p\leq 2$,
    \[\|w\|_p\leq d^{\frac{2}{p}-1}\|w\|_2.\]
\end{proof}

\begin{lemma}\label{lem:norm:breg-bound}
    Let \( W_1, W_2 \in \mathbb{R}^{d \times d} \) be two matrices such that \( \|W_1\|_{p,p} = \|W_2\|_{p,p} = 1 \). Then, the following inequalities hold:
    \begin{enumerate}[label={\textnormal{{L\arabic*.}}}]
    \item For \( p \geq 2 \),
    \[
    D_\psi(W_1, W_2) \geq \frac{1}{p \times 2^p} \|W_1 - W_2\|_{p,p}^p,
    \]
    \item For \( p \in (1, 2) \),
    \[
    D_\psi(W_1, W_2) \geq \frac{(p-1)^2}{p} \|W_1 - W_2\|_{2,2}^2.
    \]
    \end{enumerate}
    Here, $D_\psi(\cdot,\cdot)$ denotes the Bregman divergence given in Definition~\ref{def:breg:d}.
\end{lemma}

\begin{proof}
    Let $W_1=(x_{ij})_{i,j\in[d]}$ and $W_2=(y_{ij})_{i,j\in[d]}$, then from Definition~\ref{def:breg:d}, we have
    \begin{align*}
        D_\psi(W_1,W_2)
        &=\frac{1}{p}\sum_{i,j\in[d]}|x_{ij}|^p-\frac{1}{p}\sum_{i,j\in[d]}|y_{ij}|^p-\sum_{i,j\in[d]}|y_{ij}|^{p-1}(x_{ij}-y_{ij})\operatorname{sign}(y_{ij})\\
        &=\sum_{i,j\in[d]}\left(\frac{1}{p}|x_{ij}|^p+\frac{p-1}{p}|y_{ij}|^p-|y_{ij}|^{p-1}|x_{ij}|\operatorname{sign}(x_{ij}y_{ij})\right).
    \end{align*}
    Therefore, it is enough to prove that whenever $x,y\in[-1,1]$, the expression
    \begin{align}\label{eq-bregman-individual}
        \frac{1}{p}|x|^p+\frac{p-1}{p}|y|^p-|x\|y|^{p-1}\operatorname{sign}(xy)
    \end{align}
    is at least $\frac{1}{p2^p}|x-y|^p$ if $p\geq2$, or is at least $\frac{(p-1)^2}{p}|x-y|^2$ if $p\in(1,2)$. We split the argument into two cases, the first is when the signs of $x$ and $y$ are the same, and the second for when they are not.

    \noindent\textbf{Case 1:} $\operatorname{sign}(xy)=1$, so both $x$ and $y$ have the same sign, WLOG both are non-negative. Let us fix the value $\Delta\in[-1,1]$ and find the minimum value of (\ref{eq-bregman-individual}) when we constraint $x$ and $y$ to be positive and $x-y=\Delta$. Therefore, that expression can be written as
    \[\frac{(y+\Delta)^p+(p-1)y^p}{p}-(y+\Delta)y^{p-1},\]
    the first derivative with respect to $y$ is
    \begin{align*}
        (y+\Delta)^{p-1}+(p-1)y^{p-1} &-y^{p-1}-(p-1)(y+\Delta)y^{p-2}\\
        &=(y+\Delta)^{p-1}-y^{p-1}-(p-1)\Delta y^{p-2}.    
    \end{align*}
    Since the function $t\mapsto t^{p-1}$ is convex for $p\geq 2$, and concave for $p\in(1,2)$, then that derivative is always non-negative when $p\geq 2$ and always negative when $p\in(1,2)$.

    \noindent\textbf{Sub-Case 1.1:} $p\geq 2$. In this subcase, (\ref{eq-bregman-individual}) reaches its minimum when $(x,y)=(\Delta,0)$ or $(0,-\Delta)$, depending on the sign of $\Delta$, plugging them in gets us the minimum, which is $\frac{1}{p}|\Delta|^p$ when $\Delta\geq0$ or $\frac{p-1}{p}|\Delta|^p$ otherwise.
    
    \noindent\textbf{Sub-Case 1.2:} $p\in(1,2)$. In this subcase, (\ref{eq-bregman-individual}) reaches its minimum when $(x,y)=(1,1-\Delta)$ if $\Delta$ is non-negative or $(1+\Delta,1)$ otherwise. When $\Delta$ is non-negative, the desired minimum is
    \begin{align*}
        \frac{1+(p-1)(1-\Delta)^p}{p}-(1-\Delta)^{p-1}
        &=\frac{1}{p}(1-(1-\Delta)^{p-1}-(p-1)\Delta(1-\Delta)^{p-1})\\
        &\geq\frac{1}{p}((p-1)\Delta-(p-1)\Delta(1-\Delta)^{p-1})\\
        &=\frac{(p-1)\Delta}{p}(1-(1-\Delta)^{p-1})\geq\frac{(p-1)^2}{p}\Delta^2.
    \end{align*}

    \noindent Combining the results from the subcases, we get that the expression in (\ref{eq-bregman-individual}) is lower-bounded by $\frac{1}{p}|x-y|^p$ when $p\geq2$, or $\frac{(p-1)^2}{p}|x-y|^2$ otherwise, which sufficiently satisfies the desired bounds for case 1.\\

    \noindent\textbf{Case 2:} $\operatorname{sign}(xy)=-1$, so $x$ and $y$ has opposite sign. The expression in (\ref{eq-bregman-individual}) can be simplified to
    \[\frac{1}{p}|x|^p+\frac{p-1}{p}|y|^p+|x||y|^{p-1},\]
    and we want to prove that it is at least $\frac{1}{p2^p}(|x|+|y|)^p$ when $p\geq 2$, or is at least $\frac{(p-1)^2}{p}(|x|+|y|)^2$ when $p\in(1,2)$. In the case that $p\geq 2$, one of $|x|$ or $|y|$ is at least $\frac{|x|+|y|}{2}$, so the above is at least $\frac{1}{p}\left(\frac{|x|+|y|}{2}\right)^p=\frac{1}{p2^p}(|x|+|y|)^p$. Otherwise,
    \begin{align*}
        \frac{1}{p}|x|^p+\frac{p-1}{p}|y|^p+|x|y|^{p-1}
        &=\frac{|x|(|x|^{p-1}+|y|^{p-1})+(p-1)|y|^{p-1}(|x|+|y|)}{p}\\
        &\geq\frac{(|x|+|y|)(|x|+(p-1)|y|^{p-1})}{p}\\
        &\geq\frac{(|x|+|y|)((p-1)|x|+(p-1)|y|)}{p}\\
        &=\frac{p-1}{p}(|x|+|y|)^2\geq\frac{(p-1)^2}{p}(|x|+|y|)^2.
    \end{align*}
    Therefore, we have proven the bound for this case.
\end{proof}

\begin{lemma}\label{lem:norm:xpyp}
For any $x\geq y\geq 0$, we have 
    \[x(x^{p-1}-y^{p-1})\geq\frac{p-1}{p}x^p-\frac{p-1}{p}y^p\geq y(x^{p-1}-y^{p-1}).\]
\end{lemma}
\begin{proof}
For any $x\geq y\geq 0$, we have
    \[\frac{d}{dx}\left(\frac{p-1}{p}x^p-\frac{p-1}{p}y^p\right)=(p-1)x^{p-1},\]
    \[\frac{d}{dx}\bigl(y(x^{p-1}-y^{p-1})\bigr)=(p-1)x^{p-2}y\leq(p-1)x^{p-1}.\]
Thus, as $x$ increases, $\frac{p-1}{p}x^p-\frac{p-1}{p}y^p$ grows at least as fast as $y(x^{p-1}-y^{p-1})$. Since both expressions are equal at $x=y$, the right inequality follows. The left inequality can be proven analogously.
\end{proof}

\begin{lemma}\label{lem:norm:xpyp-2}
    For any $x \geq y \geq 0$, if $q \geq 1$, then
    \[x^q - y^q \leq qx^{q-1}(x-y),\]
    and if $0 < q < 1$, then
    \[x^q - y^q \leq qy^{q-1}(x-y).\]
\end{lemma}
\begin{proof}
    We compute the derivatives
    \[\frac{d}{dx}(x^q - y^q) = qx^{q-1},\]
    \[\frac{d}{dx}\bigl(qx^{q-1}(x-y)\bigr) = q(q-1)x^{q-2}(x-y) + qx^{q-1},\]
    \[\frac{d}{dx}\bigl(qy^{q-1}(x-y)\bigr) = qy^{q-1}.\]
    When $q \geq 1$, we have $q - 1 \geq 0$ and $x - y \geq 0$, so
    \[\frac{d}{dx}(x^q - y^q) \leq \frac{d}{dx}\bigl(qx^{q-1}(x-y)\bigr).\]
    Since both sides equal zero when $x = y$, the first inequality follows. For $0 < q < 1$, we have $q - 1 < 0$, so $x^{q-1} \leq y^{q-1}$ for $x \geq y > 0$, which gives $qx^{q-1} \leq qy^{q-1}$. A similar argument then establishes the second inequality.
\end{proof}

\subsection{Lemma for Analyzing ERM Objective and Its Gradient}

In this section of the Appendix, we analyze the objective function. We especially want to know about its gradient and the inner product of this gradient with the matrices of the cone set, as was mentioned before in the main body of the paper. The first one bounds the loss objective,

\begin{lemma}\label{lem:obj:l-bound}
    Under Assumption~\ref{assumption-loss}, $\mathcal{L}(W)$ is bounded from above by $\mathcal{L}_{\max}$ and from below by $\mathcal{L}_{\min}$, where $\mathcal{L}(W)$ denotes the objective of \eqref{eqn:erm} with fixed $v$, and $\mathcal{L}_{\max}$ and $\mathcal{L}_{\min}$ are finite dataset-dependent constants.
\end{lemma}

\begin{proof}
    It is enough to show the same thing for each of the loss contributions of each sample, $l_i(y_iv^\top X_i^\top\sigma(X_iWz_i))$. By Assumption \ref{assumption-loss}, we simply need to show that $y_iv^\top X_i^\top\sigma(X_iWz_i)$ is bounded by dataset-dependent bounds. However, $W$ only affects the softmax, so the above expression is bounded above by $\max_{t\in[T]}\gamma_{it}$ and bounded below by $\min_{t\in[T]}\gamma_{it}$, which are dataset dependent.
\end{proof}

\begin{lemma}\label{lem:obj:grad-derive}
    If we denote \( h_i := X_i W z_i \) and \( l_i' := l'(\gamma_i^\top \sigma(h_i)) \), then
    \[
    \nabla \mathcal{L}(W) = \frac{1}{n} \sum_{i=1}^n l_i' X_i^\top (\operatorname{diag}(\sigma(h_i)) - \sigma(h_i) \sigma(h_i)^\top) \gamma_i z_i^\top,
    \]
    where  $\mathcal{L}(W)$ denotes the objective of \eqref{eqn:erm} with fixed $v$.
\end{lemma}

\begin{proof}
    We first calculate the derivatives of each term in the sum of \( \mathcal{L}(W) \). The derivative of the \( i \)-th term for the \( W_{j_1 j_2} \) component is
    \begin{align*}
        \frac{\partial}{\partial W_{j_1 j_2}} l(y_i v^\top X_i^\top \sigma(X_i W z_i))
        &= l_i' \gamma_i^\top \frac{\partial}{\partial W_{j_1 j_2}} \sigma(X_i W z_i) \\
        &= l_i' \gamma_i^\top \nabla\sigma(h_i) X_{i,:,j_1}^\top z_{ij_2} \\
        &= l_i' X_{i,:,j_1} \nabla\sigma(h_i)^\top \gamma_i z_{ij_2}.
    \end{align*}
    Therefore, the derivative for the \( j_2 \)-th row of \( W \) is
    \[l_i' X_i^\top \nabla\sigma(h_i)^\top \gamma_i z_{ij_2}.\]
    Next, the full gradient for the \( i \)-th term equals
    \[l_i' X_i^\top \nabla\sigma(h_i)^\top \gamma_i z_i^\top.\]
    To finish the proof, we calculate the derivative of \( \sigma(h_i) \). The derivative of the \( j_1 \)-th component of \( \sigma(h_i) \) with respect to \( h_{ij_2} \) is
    \begin{align*}
        \frac{\partial}{\partial h_{ij_2}} \left( \frac{e^{h_{ij_1}}}{\sum_{l=1}^T e^{h_{il}}} \right) &= \frac{e^{h_{ij_1}} 1_{j_1 = j_2}}{\sum_{l=1}^T e^{h_{il}}} - \frac{e^{h_{ij_1}} e^{h_{ij_2}}}{\left(\sum_{l=1}^T e^{h_{il}}\right)^2} \\
        &= \sigma(h_i)_{j_1} 1_{j_1 = j_2} - \sigma(h_i)_{j_1} \sigma(h_i)_{j_2}.    
    \end{align*}
    Thus, the derivative of \( \sigma(h_i) \) is a matrix in \( \mathbb{R}^{T \times T} \) defined as   
    \[\operatorname{diag}(\sigma(h_i)) - \sigma(h_i) \sigma(h_i)^\top.\]
    Therefore, the full gradient is
    \[
        \frac{1}{n} \sum_{i=1}^n l_i' X_i^\top (\operatorname{diag}(\sigma(h_i)) - \sigma(h_i) \sigma(h_i)^\top) \gamma_i z_i^\top.
    \]
\end{proof}

\begin{lemma}\label{lem:obj:grad-l-bound}
    Under Assumption \ref{assumption-loss}, $\|\nabla \L(W)\|_{p,p}$ is bounded by a dataset-dependent constant \(L\). Furthermore, the entries of $\nabla\L(W)$ is also bounded by a dataset-dependent constant.
\end{lemma}
\begin{proof}
    Using the expression in  Lemma \ref{lem:obj:grad-derive}, since $l'$ is bounded in a closed interval (Assumption \ref{assumption-loss}), the entries in $\sigma(h_i)$ is always between $0$ and $1$, and $\gamma_i^\top\sigma(h_i)$ is bounded, then the entries of $\nabla\L(W)$ is bounded by a dataset-dependent bounded, which directly implies this lemma statement.
\end{proof}

In the following lemma, we analyze the behaviors of the \eqref{eqn:w-svm} constraint \((X_{it} - X_{i\tau})^\top W z_i\) for all \(W \in S_{p,\mu_0}(W_{\mathrm{mm}}^\alpha)\) satisfying \(\|W\|_{p,p} = \|W_{\mathrm{mm}}^\alpha\|_{p,p}\), the result of which is a generalization of \cite[Equation (64)]{tarzanagh2023transformers} for a general $\ell_p$ norm.

\begin{lemma}\label{lem:obj:sftmax-prob-gap}
    Let $\a=(\a_i)_{i=1}^n$ be locally optimal tokens as per Definition \ref{def:token:loptimal}, and let \( W^\alpha_\mathrm{mm} \) be the \eqref{eqn:w-svm} solution. Let $(\T_i)_{i=1}^n$ be the index set of all  support tokens  per Definition \ref{def:token:loptimal}. Let $\bar{\T}_i=[T]-\T_i-\{\alpha_i\}$. For any $W\in S_{p,\mu_0}(W_\mathrm{mm}^\a)$ with $\mu_0$ defined in \eqref{eqn:def:mu0} and  $\|W\|_{p,p}=\|W_\mathrm{mm}^\a\|_{p,p}$, we have
    \begin{subequations}
    \begin{align}
        (X_{it}-X_{i\tau})^\top Wz_i&\geq\frac{3}{2}\d>0, \label{prob-gap-a}\\
        (X_{i\a_i}-X_{i\tau})^\top Wz_i&\geq1+\frac{3}{2}\d, \label{prob-gap-b}\\
        1+\frac{1}{2}\d\geq(X_{i\a_i}-X_{it})^\top Wz_i&\geq1-\frac{1}{2}\d,   \label{prob-gap-c}
    \end{align}
    \end{subequations}
    for all $t\in\T_i$ and $\tau\in\bar{\T}_i$.
\end{lemma}

\begin{proof}
    Let 
    \begin{equation*}
        \bar{W} := \frac{W}{\|W\|_{p,p}} \quad \textnormal{and} \quad \bar{W}_\mathrm{mm}^\alpha := \frac{W_\mathrm{mm}^\alpha}{\|W_\mathrm{mm}^\alpha\|_{p,p}}.   
    \end{equation*}
    Using Lemma \ref{lem:norm:breg-bound} and the definition of $S_{p,\mu_0}(W_\mathrm{mm}^\a)$ in \eqref{eqn:def:spmu}, when $p \geq 2$,
    \begin{align*}
        \|\bar{W} - \bar{W}_\mathrm{mm}^\alpha\|_{p,p}^p &\leq  2^p p D_\psi(\bar{W}_\mathrm{mm}^\alpha, \bar{W}) \\
        &\leq 2^p p \mu_0 \\
        &= \left(\frac{\delta}{4A}\right)^p,    
    \end{align*}
    which implies that 
    \[\|\bar{W} - \bar{W}_\mathrm{mm}^\alpha\|_{p,p} \leq \frac{\delta}{4A}.\]
    When $p \in (1, 2)$, we can also use Lemmas \ref{lem:norm:2p-bound} and \ref{lem:norm:breg-bound} to obtain
    \begin{align*}
        \|\bar{W} - \bar{W}_\mathrm{mm}^\alpha\|_{p,p} &\leq d^{\frac{2}{p} - 1} \|\bar{W} - \bar{W}_\mathrm{mm}^\alpha\|_{2,2} \\
        &\leq d^{\frac{2}{p} - 1} \frac{\sqrt{p}}{p - 1} \sqrt{D_\psi(\bar{W}_\mathrm{mm}^\alpha, \bar{W})} \\
        &\leq d^{\frac{2}{p} - 1} \frac{\sqrt{p}}{p - 1} \sqrt{\mu_0} = \frac{\delta}{4A},
    \end{align*}
    where the last inequality uses the definition of $S_{p,\mu_0}(W_\mathrm{mm}^\a)$ in \eqref{eqn:def:spmu}.

    Therefore, either way, we have 
    \[
        \|W - W_\mathrm{mm}^\alpha\|_{p,p} \leq \frac{\delta}{4A} \|W_\mathrm{mm}^\alpha\|_{p,p}.
    \]
    We will proceed to show a bound on $(X_{it_1} - X_{it_2})^\top (W - W_\mathrm{mm}^\alpha) z_i$ for any $i \in [n]$ and any token indices $t_1, t_2 \in [T]$. To do that, let us focus on the term $X_{it_1}^\top (W - W_\mathrm{mm}^\alpha) z_i$ first,
    \begin{align*}
        \left|X_{it_1}^\top (W - W_\mathrm{mm}^\alpha) z_i\right| &= \left|\langle W - W_\mathrm{mm}^\alpha, X_{it_1} z_i^\top \rangle \right| \\
        &\leq \|W - W_\mathrm{mm}^\alpha\|_{p,p}  \cdot \|X_{it_1} z_i^\top\|_{\frac{p}{p-1},\frac{p}{p-1}} \\
        &\leq \frac{\delta}{4A} \|W_\mathrm{mm}^\alpha\|_{p,p} \cdot  \|X_{it_1} z_i^\top\|_{\frac{p}{p-1},\frac{p}{p-1}} \\
        &\leq \frac{\delta}{4A} \cdot A \\
        &= \frac{\delta}{4}.
    \end{align*}
    The first inequality above uses H\"older's Inequality. We now have
    \[
        \left|(X_{it_1} - X_{it_2})^\top (W - W_\mathrm{mm}^\alpha) z_i\right| \leq \frac{1}{2} \delta.
    \]
    
    To obtain the first inequality of the lemma in \eqref{prob-gap-a}, for all \( t \in \T_i \) and \( \tau \in \bar{\T}_i \), we have
    \begin{align*}
        (X_{it} - X_{i\tau})^\top Wz_i &\geq (X_{it} - X_{i\tau})^\top W_\mathrm{mm}^\alpha z_i + (X_{it} - X_{i\tau})^\top (W - W_\mathrm{mm}^\alpha) z_i \\
        &\geq 2\delta' - \frac{1}{2}\delta \geq \frac{3}{2}\delta.
    \end{align*}

    To get the second inequality in \eqref{prob-gap-b}, for all \( \tau \in \bar{\T}_i \), we have
    \begin{align*}
        (X_{i\alpha_i} - X_{i\tau})^\top Wz_i &\geq (X_{i\alpha_i} - X_{i\tau})^\top W_\mathrm{mm}^\alpha z_i + (X_{i\alpha_i} - X_{i\tau})^\top (W - W_\mathrm{mm}^\alpha) z_i \\
        &\geq 1 + 2\delta' - \frac{1}{2}\delta \geq 1 + \frac{3}{2}\delta.
    \end{align*}

    Finally, to get the last inequality in \eqref{prob-gap-c}, for all \( t \in \T_i \), we have
    \begin{align*}
        \left|(X_{i\alpha_i} - X_{it})^\top Wz_i - 1\right| &= \left|(X_{i\alpha_i} - X_{it})^\top W_\mathrm{mm}^\alpha z_i + (X_{i\alpha_i} - X_{it})^\top (W - W_\mathrm{mm}^\alpha) z_i - 1\right| \\
        &= |(X_{i\alpha_i} - X_{it})^\top (W - W_\mathrm{mm}^\alpha) z_i| \leq \frac{1}{2}\delta,
    \end{align*}
    which implies that
    \[
        1 + \frac{1}{2}\delta \geq (X_{i\alpha_i} - X_{it})^\top Wz_i \geq 1 - \frac{1}{2}\delta.
    \]
\end{proof}

The following two lemmas aim at bounding the correlation between the gradient and the attention matrix parameter, each of which is a generalization of \cite[Lemmas 13 and 14]{tarzanagh2023transformers} for the generalized $\ell_p$ norm.

\begin{lemma}\label{lem:obj:grad-corr-bound-1}
Suppose Assumption \ref{assumption-loss} holds. Let $\alpha=(\alpha_i)_{i=1}^n$ be locally optimal tokens as per Definition \ref{def:token:loptimal}, and let \( W^\alpha_\mathrm{mm} \) be the solution to \eqref{eqn:w-svm}. There exists a dataset-dependent constant
$R_\delta = \Theta(1/\delta)$ such that for all \(W, V \in C_{p,\mu_0,R_\delta}(W^\alpha_\mathrm{mm})\) with \(\|V\|_{p,p} = \|W^\alpha_\mathrm{mm}\|_{p,p}\), $\delta$ and  $\mu_0$  defined in \eqref{eqn:def:delta} and \eqref{eqn:def:mu0}, respectively,
\begin{align*}
-\langle \nabla\mathcal{L}(W), V \rangle = \Omega\left( \textnormal{exp} \left(-\frac{\|W\|_{p,p}}{\|W^\alpha_\mathrm{mm} \|_{p,p}}\left(1+\frac{1}{2}\delta\right)\right)\right) > 0.
\end{align*}
\end{lemma}

\begin{proof}
    Let
    \begin{align*}
        h_i:=X_iWz_i, ~~\Tilde{h}_i:=X_iVz_i,~~l_i':=l'(\gamma_i^\top\sigma(h_i)),~~\textnormal{and}~~s_i=\sigma(h_i).    
    \end{align*}    
    Therefore,
    \begin{align*}
        \langle\nabla\L(W),V\rangle
        &=\frac{1}{n}\sum_{i=1}^nl_i'\langle X_i^\top(\operatorname{diag}(s_i)-s_is_i^\top)\gamma_iz_i^\top,V\rangle\\
        &=\frac{1}{n}\sum_{i=1}^nl_i'\langle (\operatorname{diag}(s_i)-s_is_i^\top)\gamma_i,X_iVz_i\rangle\\
        &=\frac{1}{n}\sum_{i=1}^nl_i'\langle (\operatorname{diag}(s_i)-s_is_i^\top)\gamma_i,\Tilde{h}_i\rangle\\
        &=\frac{1}{n}\sum_{i=1}^nl_i'\Tilde{h}_i^\top(\operatorname{diag}(s_i)-s_is_i^\top)\gamma_i,
    \end{align*}
    \begin{align}\label{lemma-correlation-lower-bound-main}
        -\langle\nabla\L(W),V\rangle=\frac{1}{n}\sum_{i=1}^n(-l_i')\Tilde{h}_i^\top(\operatorname{diag}(s_i)-s_is_i^\top)\gamma_i.
    \end{align}
    
    Just as it was in Lemma \ref{lem:obj:grad-l-bound}, $l_i'$ is bounded for any $i\in[n]$ by some bound that is dataset-dependent. Furthermore, because $l$ is decreasing, $-l'$ is always non-negative, so an easier approach is to lower-bound the following for each $i\in[n]$,
    \[\Tilde{h}_i^\top\operatorname{diag}(s_i)\gamma_i-\Tilde{h}_i^\top s_is_i^\top\gamma_i.\]
    Next, we can get for all $i\in[n]$ and $t\in[T]$  that
    \begin{align*}
        \Tilde{h}_{it}=X_{it}^\top Vz_i&=\langle X_{it}z_i^\top,V\rangle\\&\leq\|V\|_{p,p}\|X_{it}z_i^\top\|_\frac{p}{p-1}\\
        &\leq A,
    \end{align*}
    where $A$ is defined in \eqref{eqn:def:A}.
    
    Therefore, if we drop the $i$ notation and let $\a_i=1$, and use \cite[Lemma 7]{tarzanagh2023transformers},
    \[\left|\Tilde{h}_i^\top\operatorname{diag}(s_i)\gamma_i-\Tilde{h}_i^\top s_is_i^\top\gamma_i-\sum_{t=2}^T(\Tilde{h}_1-\Tilde{h}_t)s_t(\gamma_1-\gamma_t)\right|\leq 2\Gamma A(1-s_1)^2.\]
    Let us attempt to remove the non-support tokens from the sum above by bounding the sum of the term for the non-supports,
    \[\left|\sum_{t\in\bar{\T}}(\Tilde{h}_1-\Tilde{h}_t)s_t(\gamma_1-\gamma_t)\right|\leq2\max_{t\in[T]}\{|\Tilde{h}_t|\}Q(W)\Gamma\leq2AQ(W)\Gamma.\]
    Therefore,
    \[\left|\Tilde{h}_i^\top\operatorname{diag}(s_i)\gamma_i-\Tilde{h}_i^\top s_is_i^\top\gamma_i-\sum_{t\in\T}(\Tilde{h}_1-\Tilde{h}_t)s_t(\gamma_1-\gamma_t)\right|\leq 2\Gamma A((1-s_1)^2+Q(W)),\]
    which implies that     
    \[\Tilde{h}_i^\top\operatorname{diag}(s_i)\gamma_i-\Tilde{h}_i^\top s_is_i^\top\gamma_i\geq\sum_{t\in\T}(\Tilde{h}_1-\Tilde{h}_t)s_t(\gamma_1-\gamma_t)-2\Gamma A((1-s_1)^2+Q(W)).\]
    Using Lemma \ref{lem:obj:sftmax-prob-gap}, we have
    \begin{align}\label{lemma-correlation-lower-bound-sample}
        \Tilde{h}_i^\top\operatorname{diag}(s_i)\gamma_i-\Tilde{h}_i^\top s_is_i^\top\gamma_i\geq\left(1-\frac{1}{2}\d\right)\sum_{t\in\T}s_t(\gamma_1-\gamma_t)-2\Gamma A((1-s_1)^2+Q(W)).
    \end{align}
    
    To proceed, we can upper-bound $1-s_1$ and $Q(W)$. For bounding $1-s_1$, let $\tau>1$ be some index that maximizes $X_\tau^\top Wz$, so
    \begin{align*}
         1-s_1=\frac{\sum_{t=2}^Te^{X_t^\top Wz}}{\sum_{t=1}^Te^{X_t^\top Wz}}&\leq\frac{(T-1)e^{X_\tau^\top Wz}}{(T-1)e^{X_\tau^\top Wz}+e^{X_1^\top Wz}}\\
         &\leq\frac{T}{T+e^{(X_1-X_\tau)^\top Wz}}\\
         &\leq\frac{T}{T+e^{\frac{\|W\|_{p,p}}{\|W_\mathrm{mm}^\a \|_{p,p}}(1-\frac{1}{2}\d)}}\\
         &\leq\frac{T}{e^{\frac{\|W\|_{p,p}}{\|W_\mathrm{mm}^\a \|_{p,p}}(1-\frac{1}{2}\d)}},   
    \end{align*}
    with the second to last inequality using the third inequality Lemma \ref{lem:obj:sftmax-prob-gap}. 

    For ease of notation, denote 
    \begin{align}\label{eqn:Rprime}
    R' := \frac{\|W\|_{p,p}}{\|W_\mathrm{mm}^\alpha \|_{p,p}}.    
    \end{align}
     To upper bound \( Q(W) \), we use a method similar to that for \( 1 - s_1 \), but we utilize the second inequality of Lemma \ref{lem:obj:sftmax-prob-gap} instead of the first. This gives:
        \[Q(W)\leq\frac{T}{T+e^{(1+\frac{3}{2}\d)R'}}\leq\frac{T}{e^{(1+\frac{3}{2}\d)R'}}.\]
        
    Therefore, we have
    \begin{align}\label{eq-prob-bound}
    \nonumber
            2\Gamma A((1-s_1)^2+Q(W))&\leq2\Gamma A \left(\frac{T^2}{e^{(2-\d)R'}}+\frac{T}{e^{(1+\frac{3}{2}\d)R'}}\right)\\
            &\leq\frac{2\Gamma AT(T+1)}{e^{(1+\frac{3}{2}\d)R'}}.
    \end{align}
    
    Now it is time to lower-bound the sum on the right side of Equation (\ref{lemma-correlation-lower-bound-sample}). When the set of supports is empty, that sum is zero. However, if it is not empty,
    \[\sum_{t\in\T}s_t(\gamma_1-\gamma_t)\geq S(W)\gamma^{\textnormal{gap}}.\]
    
    If we let $\tau\in\T$ be the support index that minimizes $X_\tau^\top Wz$, then
    \begin{align*}
        S(W)=\frac{\sum_{t\in\T}e^{X_t^\top Wz}}{\sum_{t=1}^Te^{X_t^\top Wz}}\geq\frac{e^{X_\tau^\top Wz}}{Te^{X_1^\top Wz}}&=\frac{1}{Te^{(X_1-X_\tau)^\top Wz}}\\
        &\geq\frac{1}{Te^{(1+\frac{1}{2}\d)R'}},    
    \end{align*}
    with the last inequality coming from the third inequality of Lemma \ref{lem:obj:sftmax-prob-gap}. 

    Therefore,
    \[\sum_{t\in\T}s_t(\gamma_1-\gamma_t)\geq\frac{\gamma^{\textnormal{gap}}}{Te^{(1+\frac{1}{2}\d)R'}}>0.\]

    Using Equation (\ref{lemma-correlation-lower-bound-sample}), we get that if the support index set is empty,
    \[\Tilde{h}_i^\top\operatorname{diag}(s_i)\gamma_i-\Tilde{h}_i^\top s_is_i^\top\gamma_i\geq-\frac{2\Gamma AT(T+1)}{e^{(1+\frac{3}{2}\d)R'}},\]
    otherwise,
    \[\Tilde{h}_i^\top\operatorname{diag}(s_i)\gamma_i-\Tilde{h}_i^\top s_is_i^\top\gamma_i\geq\frac{\gamma^{\textnormal{gap}}}{Te^{(1+\frac{1}{2}\d)R'}}\left(1-\frac{1}{2}\d\right)-\frac{2\Gamma AT(T+1)}{e^{(1+\frac{3}{2}\d)R'}}.\]

    Plugging everything back into Equation (\ref{lemma-correlation-lower-bound-main}), and considering that some samples will have non-empty support index sets, we have:
    \begin{align}\label{eqn:omega:expression}
    \nonumber
    -\langle \mathcal{L}(W), V \rangle &\geq -\frac{\min_{i\in\T_i}\{\gamma^{\text{gap}}_i\}}{nTe^{(1+\frac{1}{2}\delta)R'}}\left(1-\frac{1}{2}\delta\right)\max_{i=1}^n \{l_i'\} \\
    & \quad + \frac{2\Gamma AT(T+1)}{e^{(1+\frac{3}{2}\delta)R'}} \sum_{i=1}^n l_i' = \Omega\left(e^{-(1+\frac{1}{2}\delta)R'}\right).
    \end{align}
    Let 
    \begin{align}\label{eqn:barL:def}
    \bar{L} := \frac{\sum_{i=1}^n l_i'}{\max_{i=1}^n \{l_i'\}}.
    \end{align}
    Note that using Assumption~\ref{assumption-loss}, $\bar{L}$ is positive and bounded. Hence, using \eqref{eqn:barL:def} and  \eqref{eqn:omega:expression}, the term $-\langle \mathcal{L}(W), V \rangle$ is positive when
    \[
    R' \geq \frac{1}{\delta} \log\left(\frac{2\Gamma \bar{L} A T^2 (T+1)n}{\min_{i\in\T_i}\{\gamma^{\text{gap}}_i\}\left(1 - \frac{1}{2}\delta\right)} \right),
    \]
    or equivalently, from \eqref{eqn:Rprime}, we have
    \[
    \|W\|_{p,p} \geq \frac{\|W^\alpha_{\mathrm{mm}}\|_{p,p}}{\delta} \log\left(\frac{2\Gamma \bar{L} A T^2 (T+1)}{\min_{i\in\T_i}\{\gamma^{\text{gap}}_i\}\left(1 - \frac{1}{2}\delta\right)}\right).
    \]
\end{proof}

\noindent Finally, we introduce the following lemma to help understand the correlation between the gradient of the objective and the parameter.

\begin{lemma}\label{lem:obj:grad-corr-bound-2}
Suppose Assumption \ref{assumption-loss} holds. Let $\alpha=(\alpha_i)_{i=1}^n$ be locally optimal tokens as per Definition \ref{def:token:loptimal}, let \( W^\alpha_\mathrm{mm} \) be the \eqref{eqn:w-svm} solution, and let $R_\d$ be the constant from Lemma \ref{lem:obj:grad-corr-bound-1}. For any choice of $\pi\in(0,1)$, there exists \( R_\pi \) that depends on $\pi$ defined as
\begin{equation*}
R_\pi:=\max\left\{R_\d,\Theta\left(\frac{1}{\pi\delta}\log\frac{\delta}{\pi}\right)\right\},    
\end{equation*}
such that for all $W \in C_{p,\mu_0,R_\pi}(W_\mathrm{mm}^\alpha )$,
\begin{equation*}
\left\langle\nabla\mathcal{L}(W),\frac{W}{\|W\|_{p,p}}\right\rangle\geq(1+\pi)\left\langle\nabla\mathcal{L}(W),\frac{W_\mathrm{mm}^\alpha }{\|W_\mathrm{mm}^\alpha \|_{p,p}}\right\rangle.    
\end{equation*}
\end{lemma}

\begin{proof}
    Let 
    \begin{equation}\label{eqn:nota:hs}
    \begin{split}
         h_i &:= X_iWz_i, \quad \Tilde{h}_i := X_iW_\mathrm{mm}^\alpha z_i, \quad l_i' := l'(\gamma_i^\top\sigma(h_i)),\\
        s_i &:= \sigma(h_i), \quad \bar{W} := \frac{\|W_\mathrm{mm}^\alpha \|_{p,p}W}{\|W\|_{p,p}}, \quad \text{and} \quad \bar{h}_i := X_i\bar{W}z_i.       
    \end{split}
    \end{equation}
    By decomposing $\L(W)$ into its sum and using Lemma~\ref{lem:obj:grad-derive}, the main inequality is equivalent to the following,
    \begin{align*}
        \sum_{i=1}^n(-l_i')\langle X_i^\top(\operatorname{diag}(s_i)&-s_i s_i^\top)\gamma_i z_i^\top,\bar{W}\rangle\\ & \leq (1+\pi)\sum_{i=1}^n(-l_i')\langle X_i^\top(\operatorname{diag}(s_i)-s_i s_i^\top)\gamma_i z_i^\top, W_\mathrm{mm}^\alpha \rangle, 
    \end{align*}    
    which implies that 
    \begin{align*}
        \sum_{i=1}^n(-l_i')\langle (\operatorname{diag}(s_i)&-s_i s_i^\top)\gamma_i, X_i \bar{W} z_i\rangle \\& \leq (1+\pi)\sum_{i=1}^n(-l_i')\langle (\operatorname{diag}(s_i)-s_i s_i^\top)\gamma_i, X_i W_\mathrm{mm}^\alpha z_i\rangle.
    \end{align*}
    Using \eqref{eqn:nota:hs}, we get
    \begin{align*}
        \sum_{i=1}^n(-l_i')\langle (\operatorname{diag}(s_i)-s_i s_i^\top)\gamma_i, \bar{h}_i\rangle & \leq (1+\pi)\sum_{i=1}^n(-l_i')\langle (\operatorname{diag}(s_i)-s_i s_i^\top)\gamma_i, \Tilde{h}_i\rangle, 
    \end{align*}
    which gives
    \begin{align*}
        \sum_{i=1}^n(-l_i')\bar{h}_i^\top(\operatorname{diag}(s_i)-s_i s_i^\top)\gamma_i & \leq (1+\pi)\sum_{i=1}^n(-l_i')\Tilde{h}_i^\top(\operatorname{diag}(s_i)-s_i s_i^\top)\gamma_i.
    \end{align*}
    Hence, 
    \begin{align*}
        \sum_{i=1}^n(-l_i') & \left[(1+\pi)\left(\Tilde{h}_i^\top \operatorname{diag}(s_i) \gamma_i - \Tilde{h}_i^\top s_i s_i^\top \gamma_i\right) - \left(\bar{h}_i^\top \operatorname{diag}(s_i) \gamma_i - \bar{h}_i^\top s_i s_i^\top \gamma_i\right)\right] \geq 0.
    \end{align*}
    
    Using a similar technique as the one we used to prove Lemma \ref{lem:obj:grad-corr-bound-1},
    \begin{align*}
        \Big|\Tilde{h}_i^\top\operatorname{diag}(s_i)\gamma_i-\Tilde{h}_i^\top s_is_i^\top\gamma_i&-\sum_{t\in\T_i}(\Tilde{h}_{i\a_i}-\Tilde{h}_{it})s_{it}(\gamma_{i\a_i}-\gamma_{it})\Big|\\
        &\leq 2\Gamma A((1-s_{i\a_i})^2+Q_i(W)).
    \end{align*}
    Similarly,
    \begin{align*}
      \Big|\bar{h}_i^\top\operatorname{diag}(s_i)\gamma_i-\bar{h}_i^\top s_is_i^\top\gamma_i&-\sum_{t\in\T_i}(\bar{h}_{i\a_i}-\bar{h}_{it})s_{it}(\gamma_{i\a_i}-\gamma_{it})\Big|\\
      &\leq2\Gamma A((1-s_{i\a_i})^2+Q_i(W)).  
    \end{align*}
    Therefore, it is enough to prove that
    \begin{align}
    \begin{split}
        \sum_{i=1}^n(-l_i')&\left((1+\pi)\left(\sum_{t\in\T_i}(\Tilde{h}_{i\a_i}-\Tilde{h}_{it})s_{it}(\gamma_{i\a_i}-\gamma_{it})-2\Gamma A((1-s_{i\a_i})^2+Q_i(W))\right)\right.\\
        &\left.-\left(\sum_{t\in\T_i}(\bar{h}_{i\a_i}-\bar{h}_{it})s_{it}(\gamma_{i\a_i}-\gamma_{it})+2\Gamma A((1-s_{i\a_i})^2+Q_i(W))\right)\right).
    \end{split}
    \end{align}
    Using the fact that $\pi<1$ and using Equation (\ref{eq-prob-bound}), we get another lower-bound
    \begin{align}\label{lemma-14-reformulated}
        \sum_{i=1}^n\sum_{t\in\T_i}(-l_i')(1+\pi-(\bar{h}_{i\a_i}-\bar{h}_{it}))s_{it}(\gamma_{i\a_i}-\gamma_{it})+\frac{6\Gamma AT(T+1)}{e^{(1+\frac{3}{2}\d)R'}}\sum_{i=1}^nl_i',
    \end{align}
    with $R'$ again being $\frac{\|W\|_{p,p}}{\|W_\mathrm{mm}^\a\|_{p,p}}$. Next, we analyze the softmax probability $s_{it}$, and lower and upper-bound them in terms of $R'$ and $\bar{h}_{i\a_i}-\bar{h}_{it}$. For the lower-bound,
    \begin{align*}
        s_{it}=\frac{e^{\bar{h}_{it}R'}}{\sum_{\tau\in[T]}e^{\bar{h}_{i\tau}R'}}&\geq\frac{e^{\bar{h}_{it}R'}}{Te^{\bar{h}_{i\a_i}R'}}\\
        &=\frac{1}{T}e^{-(\bar{h}_{i\a_i}-\bar{h}_{it})R'}.   
    \end{align*}
    For the upper-bound,
    \begin{align*}
        s_{it}=\frac{e^{\bar{h}_{it}R'}}{\sum_{\tau\in[T]}e^{\bar{h}_{i\tau}R'}}&\leq\frac{e^{\bar{h}_{it}R'}}{e^{\bar{h}_{i\a_i}R'}}\\
        &=e^{-(\bar{h}_{i\a_i}-\bar{h}_{it})R'}.     
    \end{align*} 
    
In both bounds, the main inequality derivation stems from the fact that $\bar{h}_{i\a_i} > \bar{h}_{i\tau}$ for all $\tau \in [T]$, which we obtain from Lemma \ref{lem:obj:sftmax-prob-gap}. Now, we analyze the left double-summation in Equation (\ref{lemma-14-reformulated}). To analyze the sum, let $\mathcal{I}$ be the subset of $[n] \times [T]$ that contains all $(i,t)$ such that $t \in \T_i$. Furthermore, let
    \begin{align*}
        \mathcal{I}_1 &:= \left\{ (i,t) \in \mathcal{I} \mid \bar{h}_{i\a_i} - \bar{h}_{it} \leq 1 \right\}, \\
        \mathcal{I}_2 &:= \left\{ (i,t) \in \mathcal{I} \mid 1 < \bar{h}_{i\a_i} - \bar{h}_{it} \leq 1 + \pi \right\}, \\
        \mathcal{I}_3 &:= \left\{ (i,t) \in \mathcal{I} \mid \bar{h}_{i\a_i} - \bar{h}_{it} > 1 + \pi \right\}.
    \end{align*}

Therefore, we can split the sum above into the sum over $\mathcal{I}_1$, $\mathcal{I}_2$, and $\mathcal{I}_3$. The set $\mathcal{I}_1$ in particular must be non-empty because $\|\bar{W}\|_{p,p} = \|W_\mathrm{mm}^\a\|_{p,p}$, meaning that one of the constraints in the \ref{eqn:w-svm} problem must either be fulfilled exactly or violated.

The sum over $\mathcal{I}_1$ must be positive and is at least
    \[-\frac{\pi}{T}\min_{i\in\mathcal{I}_1}\{\gamma_i^{gap}\}e^{-R'}\max_{i=1}^n\{l_i'\}.\]
    The sum over $\mathcal{I}_2$ must be non-negative, and the sum over $\mathcal{I}_3$ is negative can be bounded from below using Lemma \ref{lem:obj:sftmax-prob-gap}
    \[\frac{1}{2}\d\max_{i\in\mathcal{I}_3}\{\bar{\gamma}^{gap}_i\}Te^{-(1+\pi)R'}\sum_{i=1}^nl_i'.\]

Putting things together into Equation (\ref{lemma-14-reformulated}), we get that we want the following to be non-negative
\begin{align*}
-\frac{\pi}{T}\min_{i\in\mathcal{I}_1}\{\gamma_i^{gap}\}e^{-R'}\max_{i=1}^n\{l_i'\}&+\frac{1}{2}\d\max_{i\in\mathcal{I}_3}\{\bar{\gamma}^{gap}_i\}Te^{-(1+\pi)R'}\sum_{i=1}^nl_i'\\
&+6\Gamma AT(T+1)e^{-(1+\frac{3}{2}\d)R'}\sum_{i=1}^nl_i'.    
\end{align*}
This can be achieved when
\[R'\geq\frac{1}{\min\{\pi,\frac{3}{2}\d\}}\log\left(\frac{\frac{1}{2}\d\max_{i\in\mathcal{I}_3}\{\bar{\gamma}^{gap}_i\}T^2+6\Gamma AT^2(T+1)}{\pi\min_{i\in\mathcal{I}_1}\{\gamma_i^{gap}\}\max_{i=1}^n\{l_i'\}}\sum_{i=1}^nl_i'\right),\]
or equivalently,
\[\|W\|_{p,p}\geq\frac{\|W^\a_\mathrm{mm}\|_{p,p}}{\min\{\pi,\frac{3}{2}\d\}}\log\left(\frac{\frac{1}{2}\d\max_{i\in\mathcal{I}_3}\{\bar{\gamma}^{gap}_i\}T^2+6\Gamma AT^2(T+1)}{\pi\min_{i\in\mathcal{I}_1}\{\gamma_i^{gap}\}\max_{i=1}^n\{l_i'\}}\sum_{i=1}^nl_i'\right),\]
which means that such dataset dependent $R_\pi$ exists.
\end{proof}
\subsection{Lemma for Analyzing \ref{alg:p:agd}}

We introduce the lemmas for analyzing \ref{alg:p:agd}. The first we prove is Lemma \ref{lem:alg:md-corr-bound}, which describes the lower bound of the $W$ parameter at every iterate.

\begin{lemma}\label{lem:alg:md-corr-bound}
    Suppose Assumption \ref{assumption-loss} holds. For the sequence $\{W(k)\}_{k \geq 0}$ generated by \ref{alg:p:agd}, we have
    \[
        \|W(k+1)\|_{p,p}^{p-1} \geq \|W(k)\|_{p,p}^{p-1} + \frac{\eta}{\|W(k)\|_{p,p}} \langle -\nabla \mathcal{L}(W(k)), W(k) \rangle.
    \]
\end{lemma}

\begin{proof}
    With \(\psi(W) = \frac{1}{p} \|W\|_{p,p}\), the derivative \(\nabla\psi(\cdot)\) is computed as follows:
    \[
        \nabla\psi(W) = (\operatorname{sign}(W_{ij}) |W_{ij}|^{p-1})_{1 \leq i,j \leq d}.
    \]
    Thus, we have
    \[
        \langle \nabla\psi(W), W \rangle = \sum_{i,j} \operatorname{sign}(W_{ij}) |W_{ij}|^{p-1} W_{ij} = \|W\|_{p,p}^p.
    \]
    Using this fact, we take the inner product of both sides of \eqref{eq-bregman-md} with \(W(k)\):
    \[
        \langle \nabla\psi(W(k+1)), W(k) \rangle = \langle \nabla\psi(W(k)), W(k) \rangle + \eta \langle -\nabla\mathcal{L}(W(k)), W(k) \rangle,
    \]
    \begin{align}\label{eq-inner-prod-md}
        \langle \nabla\psi(W(k+1)), W(k) \rangle = \|W(k)\|_{p,p}^p + \eta \langle -\nabla\mathcal{L}(W(k)), W(k) \rangle.
    \end{align}
    The left side of the above equation is upper-bounded by
    \[
        \sum_{i,j} \operatorname{sign}(W_{ij}(k+1)) |W_{ij}(k+1)|^{p-1} W_{ij}(k) \leq \sum_{i,j} |W_{ij}(k+1)|^{p-1} |W_{ij}(k)|.
    \]
    Using Hölder's inequality:
    \begin{align*}
        \sum_{i,j} |W_{ij}(k+1)|^{p-1} |W_{ij}(k)|
        &\leq \Big(\sum_{i,j} (|W_{ij}(k+1)|^{p-1})^{\frac{p}{p-1}}\Big)^{\frac{p-1}{p}} \Big(\sum_{i,j} |W_{ij}(k)|^p\Big)^{\frac{1}{p}} \\
        &= \|W(k+1)\|_{p,p}^{p-1} \|W(k)\|_{p,p}.
    \end{align*}
    Combining this result with \eqref{eq-inner-prod-md}, we get:
    \[
        \|W(k+1)\|_{p,p}^{p-1} \geq \|W(k)\|_{p,p}^{p-1} + \frac{\eta}{\|W(k)\|_{p,p}} \langle -\nabla\mathcal{L}(W(k)), W(k) \rangle.
    \]
\end{proof}

Next, we show several tools for analyzing the algorithm further and for analyzing the Bregman divergence. The following four specifically are from \citet{sun2022mirror}, \citet{azizan2018stochastic}, and \citet{bregman1967relaxation}, and so the proofs are omitted.


\begin{lemma}\label{thm:key-iden}
For any $W \in \mathbb{R}^{d \times d}$, the following identities hold for \ref{eqn:md}:
\begin{align}
\nonumber 
D_\psi(W,W(k)) &= D_\psi(W,W(k+1)) + D_{\psi-\eta\mathcal{L}}(W(k+1),W(k)) \\
&\quad - \eta\langle\nabla\mathcal{L}(W(k)),W-W(k)\rangle - \eta\mathcal{L}(W(k)) + \eta\mathcal{L}(W(k+1)).\label{eq-bregman-md}
\end{align}
\end{lemma}

\begin{lemma}\label{lem:alg:corr-bound-1}
Suppose Assumptions \ref{assumption-loss} hold and \(\eta\) is small enough. For the sequence $\{W(k)\}_{k \geq 0}$ generated by \ref{alg:p:agd}, we have
    \begin{align}\label{eqn:lemma-10}
        \nonumber 
         \frac{p-1}{p}\|W(k+1)\|_{p,p}^p-\frac{p-1}{p}\|W(k)\|_{p,p}^p&+\eta\L(W(k+1))-\eta\L(W(k))\\
         &\leq\langle-\eta\nabla\L(W(k)),W(k)\rangle.   
    \end{align}    
\end{lemma}

\begin{lemma}\label{lem:alg:decreasing-obj}
    Suppose Assumptions \ref{assumption-loss} hold.  Consider the sequence \(W(k)\) generated by Algorithm \ref{alg:p:agd}. Given that the step size $\eta$ is sufficiently small, then the ERM objective $\L(W(k))$ is decreasing in $k$.
\end{lemma}

\begin{lemma}\citep[Bregman Divergences Cosine Law]{bregman1967relaxation}\label{lem:alg:cosine-law}
    For any $w,w',w''$ that are all vectors or matrices with the same dimensionalities, we have
    \[D_\psi(w,w')=D_\psi(w,w'')+D_\psi(w'',w')-\langle\nabla\psi(w')-\nabla\psi(w''),w-w''\rangle.\]
\end{lemma}

\noindent The following is adapted from \citet[Equation (12)]{sun2022mirror} for the case of our attention model. Our proof is quite similar, except that we use our version of the gradient correlation lemma.

\begin{lemma}\label{lem:alg:corr-bound-2}
    Suppose Assumptions \ref{assumption-loss} hold.  Consider the sequence \(W(k)\) generated by Algorithm \ref{alg:p:agd}. For any $\pi\in(0,1)$, if $W(k)\in C_{p,\mu_0,R_\pi}(W^\a_\mathrm{mm})$, with $R_\pi$ being the constant from Lemma \ref{lem:obj:grad-corr-bound-2}, then for a small enough step size $\eta$,
    \begin{align}\label{eq-divergence-update}
    \begin{split}
        \langle\nabla\psi(W(k+1))-\nabla\psi(W(k)),\bar{W}_\mathrm{mm}^\a\rangle&\geq\frac{1}{1+\pi}(\|W(k+1)\|_{p,p}^{p-1}-\|W(k)\|_{p,p}^{p-1})\\&+\frac{\eta}{\|W(k)\|_{p,p}}(\L(W(k+1))-\L(W(k))).
    \end{split}
    \end{align}
\end{lemma}

\begin{proof}
    Let $\bar{W}^\a_\mathrm{mm}=\frac{W^\a_\mathrm{mm}}{\|W^\a_\mathrm{mm}\|_{p,p}}$. Using the \ref{alg:p:agd} algorithm equation,
    \[\langle\nabla\psi(W(k+1))-\nabla\psi(W(k)),\bar{W}_\mathrm{mm}^\a\rangle=\langle-\eta\nabla\mathcal{L}(W(k)),\bar{W}^\a_\mathrm{mm}\rangle.\]
    Then, using Lemma \ref{lem:obj:grad-corr-bound-2}, we get that
    \[\langle-\eta\nabla\mathcal{L}(W(k)),\bar{W}^\a_\mathrm{mm}\rangle\geq\frac{1}{(1+\pi)\|W(k)\|_{p,p}}\langle-\eta\nabla\mathcal{L}(W(k)),W(k)\rangle,\]
    and using Lemma \ref{lem:alg:corr-bound-1}, we get that this is lower-bounded by
    \[\frac{p-1}{p(1+\pi)\|W(k)\|_{p,p}}(\|W(k+1)\|_{p,p}^p-\|W(k)\|_{p,p}^p)+\frac{\eta}{(1+\pi)\|W(k)\|_{p,p}}(\mathcal{L}(W(k+1))-\mathcal{L}(W(k))).\]
    By Lemma \ref{lem:obj:grad-corr-bound-1}, $\langle-\eta\nabla\mathcal{L}(W(k)),W(k)\rangle>0$, so by Lemma \ref{lem:alg:md-corr-bound}, $\|W(k+1)\|_{p,p}\geq\|W(k)\|_{p,p}$. Therefore, we can use Lemma \ref{lem:norm:xpyp} to get that the above is lower-bounded by
    \[\frac{1}{1+\pi}(\|W(k+1)\|_{p,p}^{p-1}-\|W(k)\|_{p,p}^{p-1})+\frac{\eta}{(1+\pi)\|W(k)\|_{p,p}}(\mathcal{L}(W(k+1))-\mathcal{L}(W(k))).\]
    From Lemma \ref{lem:alg:decreasing-obj}, we get that we can lower-bound the above further using the right hand side of (\ref{eq-divergence-update}).
\end{proof}

With all these lemmas in hand, we provide the following Lemma \ref{lemma-stay-in-cone}. The first part of the proof follows similarly to that of \citet[Theorem 13]{sun2022mirror}, but specifically using our results on the gradient and correlation bounds for the simple attention model. It is then followed by the use of our bounds on the initialization to bound the directional Bregman divergence.

\begin{lemma}\label{lemma-stay-in-cone}
    Suppose Assumptions \ref{assumption-loss} holds and that the step size $\eta$ is sufficiently small.
    For any $\mu\in(0,\mu_0]$ and any locally optimal tokens $(\a_i)_{i=1}^n$ as per Definition \ref{def:token:loptimal}, there exists constants $R_\mu$ and $\mu'\in(0,\mu]$ that depends on the dataset and $\mu$ such that if $C_1$ is the wider cone $C_{p,\mu,R_\mu}(W^\alpha_\mathrm{mm} )$ and $C_2$ is the thinner cone $C_{p,\mu',R_\mu}(W^\alpha_\mathrm{mm} )$, then if $W(0)\in C_2$, then $W(k)\in C_1$ for all positive indices $k$.
\end{lemma}


\begin{proof}
    Let $\pi$ be some positive real number that we determine later, and let $R_\pi$ be as described in Lemma \ref{lem:obj:grad-corr-bound-2}.

For the proof, we use induction with the assumption that $W(k) \in C_{p,\mu,R_\pi}(W_\mathrm{mm}^\alpha)$ for all $k = 0, \ldots, K-1$. We aim to find the correct $\mu'$ and $R_\mu$ such that $W(K) \in C_{p,\mu,R_\pi}(W_\mathrm{mm}^\alpha) $.
    
    Denote $\bar{W}(k):=\frac{W(k)}{\|W(k)\|_{p,p}}$, so
    \begin{align*}
        D_\psi(\bar{W}_\mathrm{mm}^\a,\bar{W}(k))
        &=\frac{1}{p}\|\bar{W}_\mathrm{mm}^\a\|_{p,p}-\frac{1}{p}\|\bar{W}(k)\|_{p,p}-\langle\nabla\psi(\bar{W}(k)),\bar{W}_\mathrm{mm}^\a-\bar{W}(k)\rangle\\
        &=1-\langle\nabla\psi(\bar{W}(k)),\bar{W}_\mathrm{mm}^\a\rangle.
    \end{align*}
    So now, let us analyze the term $\langle\nabla\psi(\bar{W}(K)),\bar{W}_\mathrm{mm}^\a\rangle$ using the inductive hypothesis on $k=0,1,...,K-1$. Lemma \ref{lem:alg:corr-bound-2} tells us that
    \begin{align}
    \begin{split}
        \langle\nabla\psi(W(k+1))-\nabla\psi(W(k)),\bar{W}_\mathrm{mm}^\a\rangle&\geq\frac{\|W(k+1)\|_{p,p}^{p-1}-\|W(k)\|_{p,p}^{p-1}}{(1+\pi)}\\&+\frac{\eta}{\|W(k)\|_{p,p}}(\L(W(k+1))-\L(W(k))).
    \end{split}
    \end{align}
    
    Since this is true for all $k=0,1,...,K-1$, and since $\|W(k)\|_{p,p}$ is increasing in $k$ from Lemma \ref{lem:alg:md-corr-bound} and \ref{lem:obj:grad-corr-bound-1}, we can sum all the above inequalities and get the following,
    \begin{align*}
    \langle\nabla\psi(W(K))-\nabla\psi(W(0)),\bar{W}_\mathrm{mm}^\a\rangle&\geq\frac{\|W(K)\|_{p,p}^{p-1}-\|W(0)\|_{p,p}^{p-1}}{(1+\pi)}\\
    &+\frac{\eta}{\|W(0)\|_{p,p}}(\L(W(K))-\L(W(0))).    
    \end{align*}
    Rearranging this, we get
    \begin{align*}
        \|W(K)\|_{p,p}^{p-1}-\langle\nabla\psi(W(K)),\bar{W}_\mathrm{mm}^\a\rangle
        &\leq\|W(0)\|_{p,p}^{p-1}-\langle\nabla\psi(W(0)),\bar{W}_\mathrm{mm}^\a\rangle\\&+\frac{\pi}{1+\pi}(\|W(K)\|_{p,p}^{p-1}-\|W(0)\|_{p,p}^{p-1})\\&+\frac{\eta}{\|W(0)\|_{p,p}}(\L(W(0))-\L(W(K))).
    \end{align*}
    Dividing by $\|W(K)\|_{p,p}^{p-1}$, we get
    \begin{align}\label{eq-iterate-conv-ineq}
    \begin{split}
        D_\psi(\bar{W}_\mathrm{mm}^\a,\bar{W}(K))
        &\leq\frac{\|W(0)\|_{p,p}^{p-1}}{\|W(K)\|_{p,p}^{p-1}}D_\psi(\bar{W}_\mathrm{mm}^\a,\bar{W}(0))+\frac{\pi}{1+\pi}\left(1-\frac{\|W(0)\|_{p,p}^{p-1}}{\|W(K)\|_{p,p}^{p-1}}\right)\\&+\frac{\eta}{\|W(K)\|_{p,p}^{p-1}\|W(0)\|_{p,p}}(\L(W(0))-\L(W(K)))\\
        &\leq\mu'+\pi+\frac{\eta(\L(W(0))-\L(W(K)))}{R_\mu^p}.
    \end{split}
    \end{align}
    Therefore, we can simply choose $\mu'=\frac{1}{3}\mu$, $\pi$ be any real number below $\frac{1}{3}\mu$, and have $R_\mu$ big enough so that $\frac{\eta(\L(W(0))-\L(W(K)))}{R_\mu^p}\leq\frac{1}{3}\mu$ and $R_\mu\geq R_\pi$, such $R_\mu$ exists because $\L$ is bounded.
\end{proof}

\subsection{Lemma for Analyzing Rate of Convergence}

We begin with the proof of Lemma \ref{lem:norm-growth-bound}.

\begin{proof}
If \(\|W(k+1)\|_{p,p} \leq \|W(k)\|_{p,p}\), the inequality holds trivially. Thus, we assume \(\|W(k+1)\|_{p,p} > \|W(k)\|_{p,p}\).

First, for each \(i,j \in [d]\), the update rule of Algorithm \ref{alg:p:agd} gives

\[
|W_{i,j}(k+1)| \leq \sqrt[p-1]{|W_{i,j}(k)|^{p-1} + \eta |\nabla \mathcal{L}(W(k))_{i,j}|}.
\]
Thus, we have
\[
\max_{i,j \in [d], \, k' \in \{k, k+1\}} |W_{i,j}(k')| \leq \max_{i,j \in [d]} \left( |W_{i,j}(k)|^{p-1} + \eta |\nabla \mathcal{L}(W(k))_{i,j}| \right)^{1/(p-1)}.
\]
Now, we bound the difference \(\|W(k+1)\|_{p,p} - \|W(k)\|_{p,p}\). First, by Lemma \ref{lem:norm:xpyp-2},
\[
\|W(k+1)\|_{p,p} - \|W(k)\|_{p,p} \leq p \|W(k)\|_{p,p}^{1-p} \left( \|W(k+1)\|_{p,p}^p - \|W(k)\|_{p,p}^p \right).
\]
Since \(\|W(k)\|_{p,p}^p = \sum_{i,j \in [d]} |W_{i,j}(k)|^p\), this becomes
\[
\|W(k+1)\|_{p,p} - \|W(k)\|_{p,p} \leq p \|W(k)\|_{p,p}^{1-p} \sum_{i,j \in [d]} \left( |W_{i,j}(k+1)|^p - |W_{i,j}(k)|^p \right).
\]
Next, we bound the sum using the absolute difference:
\[
\sum_{i,j \in [d]} \left( |W_{i,j}(k+1)|^p - |W_{i,j}(k)|^p \right) \leq \sum_{i,j \in [d]} \left| |W_{i,j}(k+1)|^p - |W_{i,j}(k)|^p \right|.
\]
By Lemma \ref{lem:norm:xpyp}, we further bound this as
\begin{align*}
\left| |W_{i,j}(k+1)|^p - |W_{i,j}(k)|^p \right| & \leq \frac{p}{p-1} \left| |W_{i,j}(k+1)|^{p-1} - |W_{i,j}(k)|^{p-1} \right|\\  \cdot 
 &\max \left\{ |W_{i,j}(k+1)|, |W_{i,j}(k)| \right\}.    
\end{align*}
Multiplying through by \(p \|W(k)\|_{p,p}^{1-p}\), we get
\begin{align*}
 \|W(k+1)\|_{p,p} - \|W(k)\|_{p,p} &\leq \frac{p^2 \|W(k)\|_{p,p}^{1-p}}{p-1} \sum_{i,j \in [d]} \left| |W_{i,j}(k+1)|^{p-1} - |W_{i,j}(k)|^{p-1} \right| \\
 &\cdot \max \left\{ |W_{i,j}(k+1)|, |W_{i,j}(k)| \right\}.   
\end{align*}
Using Algorithm \ref{alg:p:agd}'s update rule, we bound the difference in the sum as
\[
\left| |W_{i,j}(k+1)|^{p-1} - |W_{i,j}(k)|^{p-1} \right| \leq \eta |\nabla \mathcal{L}(W(k))_{i,j}|,
\]
so the expression becomes
\begin{align*}
    \|W(k+1)\|_{p,p} - \|W(k)\|_{p,p} &\leq \frac{p^2 \eta \|W(k)\|_{p,p}^{1-p}}{p-1} \sum_{i,j \in [d]} |\nabla \mathcal{L}(W(k))_{i,j}| \\
    &\cdot \max \left\{ |W_{i,j}(k+1)|, |W_{i,j}(k)| \right\}.
\end{align*}
Since 
\begin{align*}
 \max \{ |W_{i,j}(k+1)|, |W_{i,j}(k)| \} &\leq \max_{i,j \in [d]} \left( |W_{i,j}(k)|^{p-1} + \eta |\nabla \mathcal{L}(W(k))_{i,j}| \right)^{1/(p-1)} \\
 &=  \mathcal{O}(\|W(k)\|_{p,p}), \quad \textnormal{and}\\
 \sum_{i,j} |\nabla \mathcal{L}(W(k))_{i,j}| &= \|\nabla \mathcal{L}(W(k))\|_{1,1},
\end{align*}
we obtain
\[
\|W(k+1)\|_{p,p} - \|W(k)\|_{p,p} \leq \frac{p^2 \eta~ \mathcal{O}(\|W(k)\|_{p,p})} {(p-1) \|W(k)\|_{p,p}^{p-1}} \|\nabla \mathcal{L}(W(k))\|_{1,1}.
\]
Finally, by Lemma \ref{lem:obj:grad-l-bound}, which gives a constant bound on the gradient, we conclude
\[
\|W(k+1)\|_{p,p} - \|W(k)\|_{p,p} \leq \mathcal{O} \left( \|W(k)\|_{p,p}^{2-p} \right).
\]
\end{proof}

\begin{lemma}\label{lem:rate:bregman-asym}
Suppose Assumptions \ref{assumption-loss} holds.
    Let $R_\d$ be from Lemma \ref{lem:obj:grad-corr-bound-1}, let $c$ be from Lemma \ref{lem:alg:corr-bound-2}, let $\mu'$ and $R_\mu$ be from Lemma \ref{lemma-stay-in-cone} when $\mu=\mu_0$, and let $R:=\max\{R_\mu,R_\d,1\}$. For any $W\in\R^{d\times d}$, denote $\bar{W}:=W/\|W\|_{p,p}$. If the initialization $W(0)$ is in $C_{p,\mu',R}(W_\mathrm{mm}^\a)$, then for a sufficiently small step size $\eta$, the sequence $\{W(k)\}_{k \geq 0}$ generated by \ref{alg:p:agd} satisfies
\begin{equation}\label{eqn:dwk:rate}
D_\psi(\bar{W}^\alpha_\mathrm{mm}, \bar{W}(k)) = \left\{
\begin{aligned}
    &\mathcal{O}\left(\frac{\log\|W(k)\|_{p,p}}{\|W(k)\|_{p,p}}\right) & \textnormal{if } p > 2, \\
    &\mathcal{O}\left(\frac{(\log\|W(k)\|_{p,p})^2}{\|W(k)\|_{p,p}}\right) & \textnormal{if } p = 2, \\
    &\mathcal{O}\left(\frac{1}{\|W(k)\|_{p,p}^{p-1}}\right) & \textnormal{otherwise}.
\end{aligned}  
\right.
\end{equation}
\end{lemma}
\begin{proof}
    By setting $\pi=\min\{\frac{c\log\|W(k)\|_{p,p}}{\delta\|W(k)\|_{p,p}},1\}$, we have $\|W(k)\|_{p,p}\geq R_\pi$, so can use the result of Lemma \ref{lem:alg:corr-bound-2} on index $k$, so rearranging that result, we get
    \begin{align*}
        \|W(k+1)\|_{p,p}^{p-1}-\langle\nabla\psi(W(k+1)),\bar{W}_\mathrm{mm}^\a\rangle
        &\leq\|W(k)\|_{p,p}^{p-1}-\langle\nabla\psi(W(k)),\bar{W}_\mathrm{mm}^\a\rangle\\
        &+\frac{\pi}{1+\pi}(\|W(k+1)\|_{p,p}^{p-1}-\|W(k)\|_{p,p}^{p-1})\\
        &+\frac{\eta}{\|W(k)\|_{p,p}}(\L(W(k))-\L(W(k+1))).
    \end{align*}
    
    From Lemma \ref{lem:obj:grad-corr-bound-1} and Lemma \ref{lem:alg:md-corr-bound}, $\|W(k)\|_{p,p}$ is increasing, so focusing on the second line, we can use Lemma \ref{lem:norm:xpyp-2} and get
    \begin{align*}
        \frac{\pi}{1+\pi}(\|W(k+1)\|_{p,p}^{p-1}-\|W(k)\|_{p,p}^{p-1})
        &\leq\pi(\|W(k+1)\|_{p,p}^{p-1}-\|W(k)\|_{p,p}^{p-1})\\
        &\leq\frac{cp}{\d\|W(k)\|_{p,p}}\max\{\|W(k)\|_{p,p}^{p-2},\|W(k+1)\|_{p,p}^{p-2}\}\\&\times\log\|W(k)\|_{p,p}\\&\times(\|W(k+1)\|_{p,p}-\|W(k)\|_{p,p}).
    \end{align*}
    Let
    \begin{align*}
        \Delta(k)&=\frac{cp}{\d\|W(k)\|_{p,p}}\max\{\|W(k)\|_{p,p}^{p-2},\|W(k+1)\|_{p,p}^{p-2}\}\log\|W(k)\|_{p,p},
    \end{align*}
    so we can get that
    \begin{align*}
    \begin{split}
        \|W(K)\|_{p,p}^{p-1}-\langle\nabla\psi(W(K)),\bar{W}^\a_\mathrm{mm}\rangle
        &\leq\|W(0)\|_{p,p}^{p-1}-\langle\nabla\psi(W(0)),\bar{W}^\a_\mathrm{mm}\rangle\\&+\sum_{k=0}^{k-1}\Delta(k)(\|W(k+1)\|_{p,p}-\|W(k)\|_{p,p})\\&+\frac{\eta}{\|W(K)\|_{p,p}}(\mathcal{L}(W(0))-\mathcal{L}(W(K))),
    \end{split}
    \end{align*}
    \begin{align}
    \begin{split}\label{eq:lem-rate-div}
        \|W(K)\|_{p,p}^{p-1}D_\psi(\bar{W}^\a_\mathrm{mm},\bar{W}(K))
        &\leq\|W(0)\|_{p,p}^{p-1}D_\psi(\bar{W}^\a_\mathrm{mm},\bar{W}(0))\\&+\sum_{k=0}^{k-1}\Delta(k)(\|W(k+1)\|_{p,p}-\|W(k)\|_{p,p})\\&+\frac{\eta}{\|W(K)\|_{p,p}}(\mathcal{L}(W(0))-\mathcal{L}(W(K))).
    \end{split}
    \end{align}

    We can deduce from Lemmas \ref{lem:alg:md-corr-bound} and \ref{lem:obj:grad-corr-bound-1} that $\|W(k)\|_{p,p}$ is increasing in $k$, and with Lemma \ref{lem:norm-growth-bound}, we can upperbound the increase. Therefore, to approximate that sum above, we can use an integral approximation by "integrating" over the $\|W(k)\|_{p,p}$ terms.

    Now let us investigate the different cases for $p$. When $p>2$, we have
    \[\|W(k+1)\|_{p,p}\leq\|W(k)\|_{p,p}+\mathcal{O}(1)\text{ and}\]
    \begin{align*}
        \Delta(k)=\frac{cp}{\d\|W(k)\|_{p,p}}(\|W(k)\|_{p,p}+\mathcal{O}(1))^{p-2}\log\|W(k)\|_{p,p}.
    \end{align*}
    We can see that for any constant $C$

\begin{align*}
\frac{d}{dx}(x+C)^{p-2}\log x&=(p-2)(x+C)^{p-3}\log x+\frac{(x+C)^{p-2}}{x}\\
&=\Omega\left(\frac{(x+C)^{p-2}}{x}\log x\right)    
\end{align*}
 for all $x>0$, so we can use an integral approximation to get that
    \begin{align*}
        \sum_{k=0}^{K-1}\Delta(k)(\|W(k+1)\|_{p,p}-\|W(k)\|_{p,p})
        &=\mathcal{O}\left(\int_{\|W(0)\|_{p,p}}^{\|W(K)\|_{p,p}}\frac{cp}{\d x}(x+\mathcal{O}(1))^{p-2}\log xdx\right)\\
        &=\mathcal{O}(\|W(K)\|^{p-2}\log\|W(K)\|_{p,p}).
    \end{align*}
    
    When $p=2$, we have
    \[\|W(k+1)\|_{p,p}\leq\|W(k)\|_{p,p}+\mathcal{O}(1)\text{ and}\]
    \begin{align*}
        \Delta(k)=\frac{cp}{\d\|W(k)\|_{p,p}}\log\|W(k)\|_{p,p}.
    \end{align*}
    We can see that
    \[\frac{d}{dx}(\log x)^2=\frac{2}{x}(\log x),\]
    so we can get that using an integral approximation,
\begin{align*}
\sum_{k=0}^{K-1}\Delta(k)(\|W(k+1)\|_{p,p}-\|W(k)\|_{p,p})&=\mathcal{O}\left(\int_{\|W(0)\|_{p,p}}^{\|W(K)\|_{p,p}}\frac{cp}{\d x}\log xdx\right)\\
&=\mathcal{O}((\log\|W(K)\|_{p,p})^2).    
\end{align*}
When $p<2$, we have
    \[\|W(k+1)\|_{p,p}\leq\|W(k)\|_{p,p}+\mathcal{O}(\|W(k)\|_{p,p}^{2-p}), \quad \textnormal{and}\]
    \begin{align*}
        \Delta(k)(\|W(k+1)\|_{p,p}&-\|W(k)\|_{p,p})\\
        &=cp\d^{-1}\|W(k)\|_{p,p}^{p-3}\log\|W(k)\|_{p,p}(\|W(k+1)\|_{p,p}-\|W(k)\|_{p,p}).
    \end{align*}
    Now let $A, B,$ and $C$ be partition of the index set $\{0,1,...,k-1\}$ such that $A$ contains the indices such that $\|W(k+1)\|_{p,p}\leq\|W(k)\|_{p,p}+1$, $B$ contains the indices such that $1<\|W(k+1)\|_{p,p}-\|W(k)\|_{p,p}\leq\|W(k)\|_{p,p}^{1-p/2}$, and $C$ be the other indices. Using an integral approximation,
    \[\sum_{k\in A}\Delta(k)(\|W(k+1)\|_{p,p}-\|W(k)\|_{p,p})=\mathcal{O}(1).\]
    Furthermore,
    \[\sum_{k\in B}\Delta(k)(\|W(k+1)\|_{p,p}-\|W(k)\|_{p,p})\leq\sum_{k\in B}\mathcal{O}(\|W(k)\|^{-2+p/2}\log\|W(k)\|_{p,p}),\]
    \[\sum_{k\in C}\Delta(k)(\|W(k+1)\|_{p,p}-\|W(k)\|_{p,p})\leq\sum_{k\in C}\mathcal{O}(\|W(k)\|^{-1}\log\|W(k)\|_{p,p}).\]
    Since $\{\|W(k)\|_{p,p}:k\in B\}$ have all pair of elements being at least $1$ apart, the sum over $k\in B$ must be $\mathcal{O}(1)$. Similarly, if we order the elements of $C$ in increasing order $k_1,...,k_l$, $\|W(k_{i+1})\|_{p,p}\geq\|W(k_i)\|_{p,p}+\|W(k_i)\|_{p,p}^{1-p/2}$, so $\|W(k_i)\|_{p,p}=\Omega(i^{2/p})$. This makes the sum over $k\in C$ also be $\mathcal{O}(1)$, so
    \[\sum_{k=0}^{K-1}\Delta(k)(\|W(k+1)\|_{p,p}-\|W(k)\|_{p,p})=\mathcal{O}(1).\]

    Combining the above cases with Equation (\ref{eq:lem-rate-div}), we get that
    \[\|W(K)\|_{p,p}^{p-1}D_\psi(\bar{W}^\a_\mathrm{mm},\bar{W}(K))=\begin{cases}
        O(\|W(K)\|_{p,p}^{p-2}\log\|W(K)\|_{p,p}) & \text{if }p>2,\\
        O((\log\|W(K)\|_{p,p})^2) & \text{ if }p=2,\\
        O(1) & \text{otherwise}
    \end{cases},\]
Dividing both sides by $\|W(K)\|_{p,p}^{p-1}$ gives \eqref{eqn:dwk:rate}.
\end{proof}

\begin{lemma}\label{lem:rate:norm-asym}
    Suppose Assumption \ref{assumption-loss} holds. Let $\mu'$ be that from Lemma \ref{lemma-stay-in-cone} if $\mu=\mu_0$, and let $R$ the maximum of the $R_\mu$ from Lemma \ref{lemma-stay-in-cone} and $R_\d$ from Lemma \ref{lem:obj:grad-corr-bound-1}. Let  $\{W(k)\}_{k \geq 0}$ be the sequence generated by \ref{alg:p:agd}. If the initialization $W(0)$ is in $C_{p,\mu',R}(W_\mathrm{mm}^\a)$, then with a small enough step size $\eta$, we have the following for each $k\geq0$,
    \begin{align*}
        \|W(k)\|_{p,p}=\Omega(\log k).    
    \end{align*}
\end{lemma}

\begin{proof}
    For each $k\geq0$, Lemma \ref{lem:alg:md-corr-bound}  gives
    \[\|W(k+1)\|_{p,p}^{p-1}\geq\|W(k)\|_{p,p}^{p-1}+\frac{\eta}{\|W(k)\|_{p,p}}\langle-\nabla\L(W(k)),W(k)\rangle.\]
    Lemma \ref{lemma-stay-in-cone} gives us that $W(k)\in C_{p,\mu,R}(W_\mathrm{mm}^\a)$ for each $k\geq0$, so by Lemma \ref{lem:obj:grad-corr-bound-1},
    \[\frac{\eta}{\|W(k)\|_{p,p}}\langle-\nabla\L(W(k)),W(k)\rangle=\Omega\left(e^{-\frac{\|W(k)\|_{p,p}}{\|W_\mathrm{mm}^\a\|_{p,p}}(1+\frac{1}{2}\d)}\right),\]
    so there exists dataset dependent constants $R_1,R_2>0$ such that
    \[\frac{\eta}{\|W(k)\|_{p,p}}\langle-\nabla\L(W(k)),W(k)\rangle\geq R_1e^{-R_2\|W(k)\|_{p,p}},\]
    so for each $k\geq0$,
    \[\|W(k+1)\|_{p,p}^{p-1}\geq\|W(k)\|_{p,p}^{p-1}+R_1e^{-R_2\|W(k)\|_{p,p}}.\]
    Set $k_0=0$, and let $k_{i+1}$ be the lowest indices such that $\|W(k_{i+1})\|_{p,p}\geq\|W(k_i)\|_{p,p}+1$ for all index $i\geq0$. Therefore,
    \[k_{i+1}-k_i\leq\frac{(\|W(k_i)\|_{p,p}+1)^{p-1}-\|W(k_i)\|_{p,p}^{p-1}}{R_1e^{-R_2(\|W(k_i)\|_{p,p}+1)}}=e^{O(\|W(k_i)\|_{p,p})}.\]
    Thuse, $k_i$ grows exponentially, which implies
    \[\|W(k)\|_{p,p}=\Omega(\log k).\]
\end{proof}
\section{Proof of Theorem~\ref{thm:rp:w}}
\begin{proof}
The proof is similar to the proof of \citet[Theorem 1]{ataee2024max}. Specifically, we need to show that \( f(X) = v^\top X^\top \S(XW) \) satisfies the assumptions of \cite[Lemma 14]{ataee2024max}, where the nonlinear head is replaced by the linear term \( v \). This holds independently of the choice of algorithm or the attention SVM solution. Thus, we omit the details and refer to the proof of \citet[Theorem 1]{ataee2024max}.
\end{proof}

\section{Proof of Theorem~\ref{thm-norm}}
\begin{proof}
This is a direct implication of Lemma \ref{lem:rate:norm-asym}.
\end{proof}


\section{Proof of Theorem~\ref{thm-direction}}
\begin{proof}
This is a direct consequence of Theorem \ref{thm-rate}.    
\end{proof}

\section{Proof of Theorem~\ref{thm-rate}}
\begin{proof}
Let \(R\) be the one from Lemma \ref{lem:rate:bregman-asym}. Given \(W(0) \in C_{p,\mu,R}(W^\alpha_\mathrm{mm})\), by Lemma~\ref{lem:rate:bregman-asym}, we have

\[
D_\psi(\bar{W}^\alpha_\mathrm{mm}, \bar{W}(k)) = \left\{
\begin{aligned}
    &\mathcal{O}\left(\frac{\log\|W(k)\|_{p,p}}{\|W(k)\|_{p,p}}\right) & \text{if } p > 2, \\
    &\mathcal{O}\left(\frac{(\log\|W(k)\|_{p,p})^2}{\|W(k)\|_{p,p}}\right) & \text{if } p = 2, \\
    &\mathcal{O}\left(\frac{1}{\|W(k)\|_{p,p}^{p-1}}\right) & \text{otherwise}.
\end{aligned}
\right.
\]

From Lemma~\ref{lem:rate:norm-asym}, we know that

\[
\|W(k)\|_{p,p} = \Omega(\log k).
\]

The derivative \(\frac{d}{dx}\left(\frac{\log x}{x}\right) = \frac{1-\log x}{x^2}\) is negative when \(x > e\), so \(\frac{\log x}{x}\) is decreasing when \(x > e\). Similarly, \(\frac{(\log x)^2}{x}\) is decreasing when \(x > e^2\).

Thus when $p>2$, for a large enough $k$,
\begin{subequations}
\begin{align}\label{eqn:dpsi:rate:case:a}
D_\psi(\bar{W}^\a_\mathrm{mm},\bar{W}(k))=O\left(\frac{\log\log k}{\log k}\right).    
\end{align}
Similarly, when $p=2$, for a large enough $k$,
\begin{align}\label{eqn:dpsi:rate:case:b}
D_\psi(\bar{W}^\a_\mathrm{mm},\bar{W}(k))=O\left(\frac{(\log\log k)^2}{\log k}\right).    
\end{align}
Finally, when $1<p<2$,
\begin{align}\label{eqn:dpsi:rate:case:c}
D_\psi(\bar{W}^\a_\mathrm{mm},\bar{W}(k))=O\left(\frac{1}{(\log k)^{p-1}}\right).    
\end{align}
\end{subequations}
    
\end{proof}

\section{On the Convergence of the $\ell_p$ Regularization Path for Joint $W$ and $v$}\label{sec:app:joint}
In this section, we extend the results of Theorem \ref{thm:rp:w} to the case of joint optimization of head $v$ and attention weights $W$ using a logistic loss function.

\begin{assumption}\label{assum:label:shrink}
Let \(\Gamma, \Gamma' > 0\) denote the label margins when solving \eqref{eqn:vp:svm} with $X_{i\alpha_i}$ and its replacement with $X^\top_{i} \sigma(X_{i}Wz_i)$, for all $i \in [n]$, respectively. There exists $\nu>0$ such that for all $i \in [n]$ and $W \in \mathbb{R}^{d \times d}$,
\begin{align*}
\Gamma - \Gamma' \geq \nu \cdot (1 - s_{i\alpha_i}), \quad \textnormal{where}\quad s_{i\alpha_i}=[\S(X_{i} W z_i)]_{\alpha_i}. 
\end{align*}
\end{assumption}
Assumption \ref{assum:label:shrink} is similar to \cite{ataee2024max} and highlights that selecting optimal tokens is key to maximizing the classifier's label margin. When attention features, a weighted combination of all tokens, are used, the label margin shrinks based on how much attention is given to the optimal tokens. The term $\nu \cdot (1 - s_{i\alpha_i})$ quantifies this minimum shrinkage. If the attention mechanism fails to focus on these tokens (i.e., low $s_{i\alpha_i}$), the margin decreases, reducing generalization. This assumption implies that optimal performance is achieved when attention converges on the most important tokens, aligning with the max-margin attention SVM solution.

Similar to how we provided the characterization of convergence for the regularization path of \ref{alg:p:agd}, we offer a similar characterization here for \ref{alg:pq:agd}.

\begin{theorem}[Joint $\ell_p$--norm Regularization Path]\label{thm:joint:rp}
Consider the empirical risk minimization problem with logistic loss \( l(x) = \log(1 + \exp(-x)) \):
\begin{equation}
\begin{aligned}
(v^{(r)}, W^{(R)}) := &\argmin_{(v,W)}~\mathcal{L}(v, W) \\
\textnormal{subj. to} \quad &\|W\|_{p,p} \leq R, \\
&\|v\|_p \leq r := r(R).
\end{aligned}
\tag{\texttt{$\ell_p$-JointRP}}\label{alg:rp}
\end{equation}
Suppose there exist token indices $\alpha = (\alpha_i)_{i=1}^m$ for which $W^\alpha_{\mathrm{mm}}$ (as defined in \eqref{eqn:w-svm}) exists and Assumption~\ref{assum:label:shrink} holds for some $\Gamma, \nu > 0$. Then, for $r(R) = \exp(\Theta(R))$, as $R \to \infty$, we have
\begin{align}
  \lim_{R \to \infty}  \left(\frac{v^{(r)}}{r}, \frac{W^{(R)}}{R}\right) = \left(\frac{v_{\mathrm{mm}}}{\|v_{\mathrm{mm}}\|_{p}}, \frac{W^\alpha_{\mathrm{mm}}}{\|W^\alpha_{\mathrm{mm}}\|_{p,p}}\right),
\end{align}
where $v_{\mathrm{mm}}$ and $W^\alpha_{\mathrm{mm}}$ are the solutions to equations \eqref{eqn:vp:svm} and \eqref{eqn:w-svm}, respectively.
\end{theorem}



\begin{proof}Without loss of generality, we set \(\alpha_i = 1\) for all \(i \in [n]\), and we use \(W_\mathrm{mm}\) instead of \(W_\mathrm{mm}^\alpha\). Suppose the claim is incorrect, meaning either \(W^{(R)} / R\) or \(v^{(r)} / r\) fails to converge as \(R\) and \(r\) grow. Define 
\begin{align}\label{notation:thm:joint1}
\nonumber 
\Xi &= \frac{1}{\|\bar{W}_\mathrm{mm}\|_{p,p}},  \qquad \Gamma = \frac{1}{\|v_\mathrm{mm}\|_p},  \\
\bar{W}_\mathrm{mm} &:= R \Xi W_\mathrm{mm},  \qquad \bar{v}_\mathrm{mm} := r \Gamma v_\mathrm{mm}.
\end{align}
Our strategy is to show that \((\bar{v}_\mathrm{mm}, \bar{W}_\mathrm{mm})\) is a strictly better solution compared to \((v^{(r)}, W^{(R)})\) for large \(R\) and \(r\), leading to a contradiction. 

\noindent$\bullet$ \textbf{Case 1: $W^{(R)}/R$ does not converge to $\bar{W}_\mathrm{mm}/R$}. 
In this case, there exists $\delta,\gamma=\gamma(\delta)>0$ such that we can find arbitrarily large $R$ with 
\[
\|W^{(R)}/R-\bar{W}_\mathrm{mm}/R\|\geq \delta
\]
and the margin induced by $W^{(R)}/R$ is at most \(\Xi(1-\gamma)\). 

We bound the amount of non-optimality \(q^*_i\) of \(\bar{W}_\mathrm{mm}\):
\begin{align*}
q^*_i := \frac{\sum_{t \neq \alpha_i} \exp(X_{it}^\top \bar{W}_\mathrm{mm} z_i)}{\sum_{t \in [T]} \exp(X_{it}^\top \bar{W}_\mathrm{mm} z_i)} &\leq \frac{\sum_{t \neq \alpha_i} \exp(X_{it}^\top \bar{W}_\mathrm{mm} z_i)}{\exp(X_{i\alpha_i}^\top \bar{W}_\mathrm{mm} z_i)}\\
&\leq T \exp(-\Xi R).    
\end{align*}
Thus, 
\begin{subequations}\label{joint:qbound}
\begin{equation}\label{joint:qstarmax}
   q^*_{\max} := \max_{i \in [n]} q^*_i \leq T \exp(-\Xi R).
\end{equation}
Next, assume without loss of generality that the first margin constraint is \(\gamma\)-violated by \(W^{(R)}\), meaning 
\[
 \min_{t \neq \alpha_1}~(X_{1\alpha_1} - X_{1t})^\top W^{(R)} z_1\leq \Xi R (1 - \gamma).   
\]
Denoting the amount of non-optimality of the first input of $W^{(R)}$ as \(\hat{q}_1\), we find
\begin{align*}
\hat{q}_1 := \frac{\sum_{t \neq \alpha_1} \exp(X_{1t}^\top W^{(R)}z_1)}{\sum_{t \in [T]} \exp(X_{1t}^\top W^{(R)}z_1)} &\geq \frac{1}{T} \frac{\sum_{t \neq \alpha_1} \exp(X_{1t}^\top W^{(R)}z_1)}{\exp(X_{1\alpha_1}^\top W^{(R)}z_1)} \\
&\geq T^{-1} \exp(-(1 - \gamma) R \Xi).    
\end{align*}
This implies that 
\begin{equation}\label{joint:qhatmax}
\hat{q}_{\max} := \max_{i \in [n]} q^*_i \geq T^{-1} \exp(-\Xi R(1-\gamma)).
\end{equation}
We similarly have 
\begin{equation}\label{joint:qmax}
q^*_{\max} \geq T^{-1} \exp(-\Xi R).
\end{equation}
\end{subequations}
Thus, \eqref{joint:qbound} gives the following relationship between the upper and lower bounds on non-optimality:
\begin{align}\label{smax difff joint}
\nonumber
 -(1 - \gamma) \Xi R - \log T &\leq \log(\hat{q}_{\max}), \\
    -\Xi R - \log T &\leq \log(q^*_{\max}) \leq -\Xi R + \log T.
\end{align}
In other words, \(\bar{W}_\mathrm{mm}\) has exponentially less non-optimality compared to \(W^{(R)}\) as \(R\) grows. To proceed, we need to upper and lower bound the logistic loss of \((\bar{v}_\mathrm{mm}, \bar{W}_\mathrm{mm})\) and \((v^{(r)}, W^{(R)})\) respectively, to establish a contradiction.

\noindent$\bullet$ \textbf{Sub-Case 1.1:  Upper bound for $\mathcal{L}(\bar{v}_\mathrm{mm}, \bar{W}_\mathrm{mm})$}. 
We now bound the logistic loss for the limiting solution. Set \(\bar{r}_i = X_i^\top \S(X_i \bar{W}_\mathrm{mm} z_i)\). If \(\| \bar{r}_i - X_{i1} \|_{p} \leq \epsilon_i\), then \(v_\mathrm{mm}\) satisfies the SVM constraints on \(\bar{r}_i\) with \(Y_i \cdot \bar{r}_i^\top v_\mathrm{mm} \geq 1 - \epsilon_i / \Gamma\). Setting \(\epsilon_{\max} = \sup_{i \in [n]} \epsilon_i\), \(v_\mathrm{mm}\) achieves a label-margin of \(\Gamma - \epsilon_{\max}\) on the dataset \((Y_i, \bar{r}_i)_{i \in [n]}\). Let \(M = \sup_{i \in [n], t, \tau \in [T]} \| X_{it} - X_{i\tau} \|_p\). Recalling \eqref{smax difff joint}, the worst-case perturbation is
\[
\epsilon_{\max} \leq M \exp(-\Xi R + \log T) = M T \exp(-\Xi R).
\]
This implies the upper bound on the logistic loss:
\begin{align}
\nonumber
\mathcal{L}(\bar{v}_\mathrm{mm}, \bar{W}_\mathrm{mm}) &\leq \max_{i \in [n]} \log(1 + \exp(-Y_i \bar{r}_i^\top \bar{v}_\mathrm{mm})) \\
\nonumber
&\leq \max_{i \in [n]} \exp(-Y_i \bar{r}_i^\top \bar{v}_\mathrm{mm}) \\
\nonumber
&\leq \exp(-r \Gamma + r \epsilon_{\max}) \\
&\leq e^{r M T \exp(-\Xi R)} e^{-r \Gamma}. \label{logistic upper}
\end{align}

\noindent$\bullet$ \textbf{Sub-Case 1.2:  Lower bound for $\mathcal{L}(v^{(r)}, W^{(R)})$}. 
We now bound the logistic loss for the finite solution. Set \(\bar{r}_i = X_i^\top \sigma(X_i W^{(R)}z_i)\). Using Assumption \ref{assum:label:shrink}, solving \eqref{eqn:vp:svm} on \((y_i, \bar{r}_i)_{i \in [n]}\) achieves at most \(\Gamma - \nu e^{-(1-\gamma)\Xi R} / T\) margin. Consequently, we have:

\begin{align*}
\mathcal{L}(v^{(r)}, W^{(R)})& \geq \frac{1}{n} \max_{i \in [n]} \log(1 + \exp(-Y_i \bar{r}_i^\top v^{(r)})) \\
&\geq \left(\frac{1}{2n} \max_{i \in [n]} \exp(-Y_i \bar{r}_i^\top v^{(r)}) \right)\wedge \log 2 \\
&\geq  \left(\frac{1}{2n} \exp(-r (\Gamma - \nu e^{-(1-\gamma)\Xi R} / T)) \right)\wedge \log 2 \\
&\geq \left( \frac{1}{2n} e^{r (\nu / T) \exp(-(1 - \gamma) \Xi R)} e^{-r \Gamma} \right) \wedge \log 2. 
\end{align*}

Observe that this lower bound dominates the upper bound from \eqref{logistic upper} when \(R\) is large, specifically when:
\begin{align*}
    \frac{1}{2n} \exp\left( r(\nu/T) \exp\left(-(1-\gamma)\Xi R\right) \right) &\geq \exp\left( rMT \exp(-\Xi R) \right), \\
    -\log(2n) + r(\nu/T) \exp\left(-(1-\gamma)\Xi R\right) &\geq rMT \exp(-\Xi R), \\
    r \exp\left(-(1-\gamma)\Xi R\right) \left(\frac{\nu}{T} - MT \exp(-\gamma\Xi R) \right) &\geq \log(2n).
\end{align*}
As $R$ gets larger,
\[r\gtrsim\frac{T\log(2n)}{\nu}\exp((1-\gamma)\Xi R)=\exp(\Theta(R)).\]
Thus, we obtain the desired contradiction since such a large \(R\) and \(r=r(R)\) would cause $\L(v^{(r)},W^{(R)})$ to not be optimal when \(W^{(R)} / R \not\rightarrow \bar{W}_\mathrm{mm}\). Therefore, \(W^{(R)} / R\) must converge to \( \bar{W}_\mathrm{mm}/R \).

\noindent$\bullet$ \textbf{Case 2: Suppose $v^{(r)}/r$ does not converge.} In this case, there exists $\delta > 0$ such that we can find arbitrarily large $r$ obeying $\text{dist}(v^{(r)}/r, \bar{v}_\mathrm{mm}/r) \geq \delta$. If $\text{dist}(W^{(R)}/R, \Xi W_\mathrm{mm}) \not\rightarrow 0$, then "Case 1" applies. Otherwise, we have $\text{dist}(W^{(R)}/R, \Xi W_\mathrm{mm}) \rightarrow 0$, thus we can assume $\text{dist}(W^{(R)}/R, \Xi W_\mathrm{mm}) \leq \epsilon$ for an arbitrary choice of $\epsilon > 0$.

On the other hand, due to the strong convexity of \eqref{eqn:w-svm}, for some $\gamma := \gamma(\delta) > 0$, $v^{(r)}$ achieves a margin of at most $(1 - \gamma) \Gamma r$ on the dataset $(Y_i, X_{i1})_{i \in [n]}$, where $X_{i1}$ denotes the optimal token for each $i \in [n]$.
Additionally, since $\text{dist}(W^{(R)}/R, \Xi W_\mathrm{mm}) \leq \epsilon$, $W^{(R)}$ strictly separates all optimal tokens (for small enough $\epsilon > 0$) and $\hat{q}_{\max} := f(\epsilon) \rightarrow 0$ as $R \rightarrow \infty$. Note that  \( f(\epsilon) \) quantifies the non-optimality of \( W^{(R)} \) compared to \( W_\mathrm{mm} \); as \( \epsilon \to 0 \), meaning \( W^{(R)} / R \) converges to \( \Xi W_\mathrm{mm} / R \), \( f(\epsilon) \to 0 \). Consequently, setting $r_i = X_i^\top \sigma(X_i  W^{(R)}z_i)$, for sufficiently large $R > 0$ and setting $M = \sup_{i \in [n], t \in [T]} \|X_{it}\|$, we have that
\begin{align}
\nonumber
\min_{i \in [n]}~~y_i (v^{(r)})^\top r_i & \leq \min_{i \in [n]} ~~y_i (v^{(r)})^\top X_{i1} + \sup_{i \in [n]} |(v^{(r)})^\top (X_{it} - X_{i1})| \\
\nonumber
& \leq (1 - \gamma) \Gamma r + M f(\epsilon) r \\
& \leq (1 - \gamma/2) \Gamma r.
\end{align}
This in turn implies that logistic loss is lower bounded by

\[
\mathcal{L}(v^{(r)}, W^{(R)}) \geq  \left(\frac{1}{2n} e^{\gamma \Gamma r / 2} e^{-\Gamma r} \right)\wedge \log 2.
\]

Now, using \eqref{logistic upper}, this exponentially dominates the upper bound of $(\bar{W}_\mathrm{mm}, \bar{v}_\mathrm{mm})$ whenever $r M T \exp(-\Xi R) < r \gamma \Gamma / 2-\log(2n)$, completing the proof, which is satisfied for some $r=\exp(\Theta(R))$.
\end{proof}

\begin{remark} 
The exponential relationship $r = \exp(\Theta(R))$ manifests in two complementary ways. Structurally, it emerges from the composition in \eqref{eqn:erm}: the inner attention mechanism (controlled by $W$ with norm $R$) generates exponentially sharp distributions through softmax, which the outer classifier (controlled by $v$ with norm $r$) must effectively discriminate. Analytically, this scaling is precisely quantified when examining the logistic loss landscape: when $W^{(R)}/R$ deviates from the max-margin solution by $\gamma$, the attention weights become exponentially less optimal as $\exp(-(1-\gamma)\Xi R)$, requiring $r \gtrsim \frac{T\log(2n)}{\nu}\exp((1-\gamma)\Xi R)$ to ensure the optimal solution dominates. This reveals a fundamental aspect of attention-based learning: the classifier's capacity to discriminate ($r$) must scale exponentially with the attention mechanism's capacity to focus ($R$) to guarantee convergence to the max-margin solution.
\end{remark}

\section{Implementation Details and Additional Experiments}\label{app:exp:detail}

The experiments were run on an Intel i7 core and a single V100 GPU using the pytorch and huggingface libraries. They should be runnable on any generic laptop.

\subsection{Computational Overhead of Algorithm}\label{alg:overhead}

When compared to the standard gradient descent algorithm, the mirror descent based algorithms, such as the $\ell_p$-MD, \ref{alg:p:agd}, and \ref{alg:pq:agd} have additional computational overhead. {
To provide a concrete comparison, we analyze the operation counts for updating a parameter matrix $W(k)$ with $D$ entries under three algorithms: GD, Adam, and $\ell_p$-MD. We assume the gradient $\nabla\mathcal{L}(W(k))$ is provided as an oracle, since all three algorithms require gradient computation, and focus solely on the update-rule overhead. 

For Adam, the algorithm maintains first and second momentum estimates $M(k)$ and $V(k)$, with update rules
\[
M(k+1)\leftarrow \beta_1M(k)+(1-\beta_1)\nabla\mathcal{L}(W(k)), \quad
V(k+1)\leftarrow \beta_2V(k)+(1-\beta_2)\nabla\mathcal{L}(W(k))^2,
\]
\[
W(k+1)\leftarrow W(k)-\eta\cdot\frac{M(k+1)/(1-\beta_1^k)}{\sqrt{V(k+1)/(1-\beta_2^k)}+\varepsilon}.
\]
Updating $M(k+1)$ requires two entry-wise multiplications and one addition; updating $V(k+1)$ requires two multiplications, one squaring, and one addition; updating $W(k+1)$ requires three divisions, one subtraction, one addition, one multiplication, and one square root. In total, Adam requires fourteen entry-wise operations per parameter.

For GD, the update is simply
\[
W(k+1)\leftarrow W(k)-\eta\nabla\mathcal{L}(W(k)),
\]
requiring only one multiplication and one subtraction, i.e., two entry-wise operations per parameter.

For $\ell_p$-MD, the update follows Algorithm~\ref{alg:p:agd}. Computing $[W(k)]^+_{ij}$ requires one entry-wise sign operation, one power, one absolute value, two multiplications, and one subtraction. The final computation requires one power, one absolute value, one sign, and one multiplication. In total, $\ell_p$-MD requires ten entry-wise operations per parameter.

All of these operations scale linearly with the number of parameters $D$. Thus, despite introducing richer $\ell_p$ geometry, $\ell_p$-MD remains strictly more efficient than Adam by four operations per parameter, while the overhead relative to GD is modest, linear in the number of parameters, and easily offset by the induced sparsity benefits. This analysis demonstrates that our algorithm achieves a favorable complexity-accuracy tradeoff: its computational overhead is provably negligible and in fact lighter than Adam, while yielding sparsity and interpretability benefits that neither GD nor Adam can provide.
}




\subsection{\ref{alg:p:agd} Experiment}
The dataset $(X_i,Y_i,z_i)_{i=1}^n$ is generated randomly: $X_i$ and $z_i$ are sampled from the unit sphere, and $Y_i$ is uniformly sampled from $\{\pm1\}$. Additionally, $v$ is randomly selected from the unit sphere. We use $n=6$ samples, $T=8$ tokens per sample, and $d=10$ dimensions per token, fulfilling the overparameterized condition for the \ref{eqn:w-svm} problem to be almost always feasible.

\begin{figure}[t]
     \centering
     \begin{subfigure}[b]{0.32\textwidth}
         \centering
         \includegraphics[width=\textwidth]{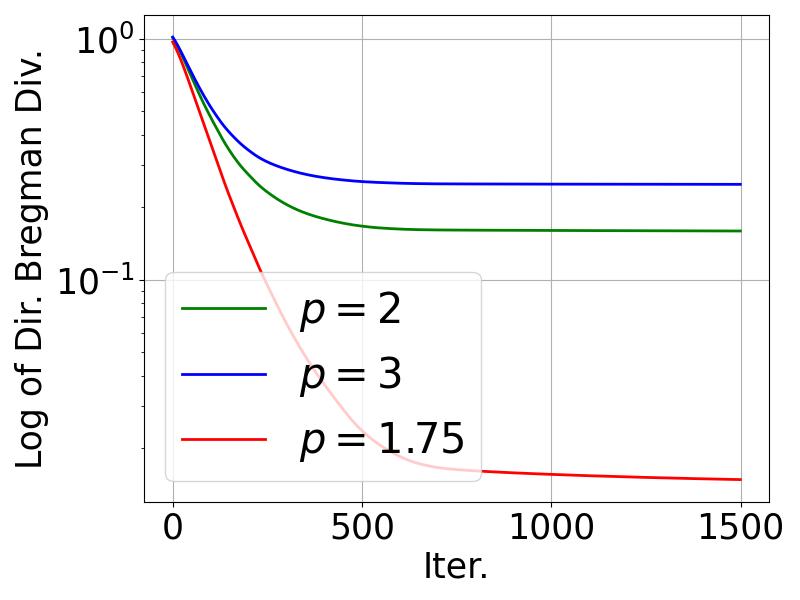}
         \caption{$\ell_{1.75}$ Log Divergence}
         \label{figjoint-inst-log-1-75}
     \end{subfigure}
     \hfill
     \begin{subfigure}[b]{0.32\textwidth}
         \centering
         \includegraphics[width=\textwidth]{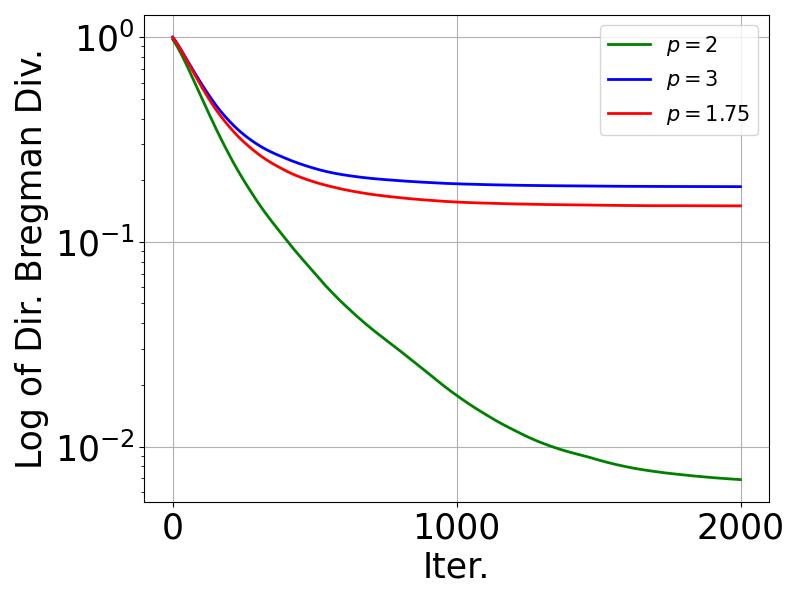}
         \caption{$\ell_2$ Log Divergence}
         \label{figjoint-inst-log-2}
     \end{subfigure}
     \hfill
     \begin{subfigure}[b]{0.32\textwidth}
         \centering
         \includegraphics[width=\textwidth]{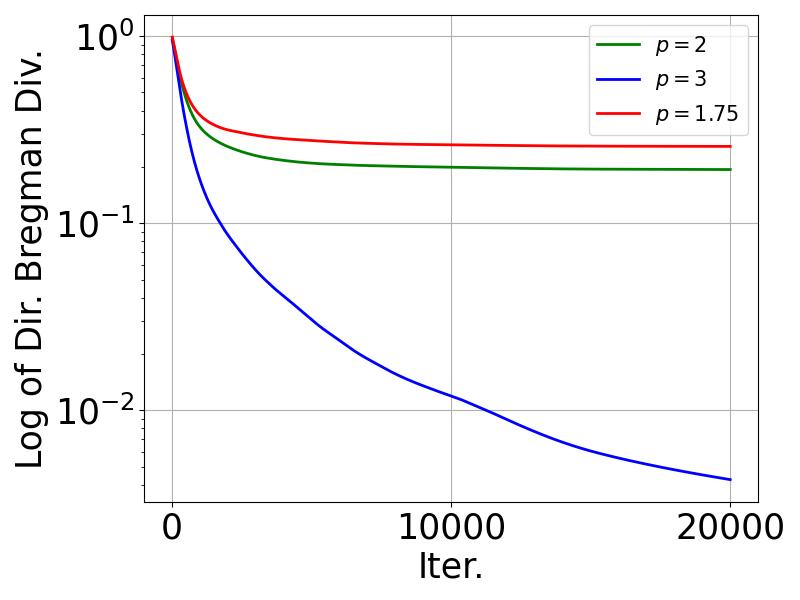}
         \caption{$\ell_3$ Log Divergence}
         \label{figjoint-inst-log-3}
     \end{subfigure}
\caption{Graph of the log of directional Bregman divergence between the (a) $\ell_{1.75}$, (b) $\ell_2$, and (c) $\ell_3$ optimization paths and the \eqref{eqn:w-svm}. This shows how the divergence behaves when viewed from the log-space, adding an extra detail to Figure \ref{fig-corr}}
\end{figure}

The model is trained with parameters initialized near the origin, using~\ref{alg:p:agd} with $p=1.75, 2$, and $3$, and a learning rate of $0.1$. Training lasted for $1,500$ epochs for $p=1.75$, $2,000$ epochs for $p=2$, and $20,000$ epochs for $p=3$. Gradients are normalized to accelerate convergence without altering results significantly. We refer to the parameter histories as the $\ell_{1.75}, \ell_2,$ and $\ell_3$ optimization paths. We compute the chosen tokens $(\alpha_i)_{i=1}^n$ for the \eqref{eqn:w-svm} problem by selecting the token with the highest softmax probability for each sample. This process is repeated for $p=1.75$, $2$, and $3$.

\subsection{~\ref{alg:pq:agd} Experiment}\label{app:exp:joint}
We use the following example dataset for the experiment on joint optimization.
\begin{example}\label{example:p-gd}
Let \(n=2\), \(T=3\), \(d=2\). Let \(y_1=1\), \(y_2=-1\). Let:
\begin{align}
  X_1 = \begin{pmatrix}
X_{11} \\ X_{12} \\ X_{13}
\end{pmatrix} =
\begin{pmatrix}
-5.4 & 2.4 \\
2.8 & 4.2 \\
2.6 & -0.2
\end{pmatrix}
,\quad \textnormal{and} \quad 
X_2 = \begin{pmatrix}
X_{21} \\ X_{22} \\ X_{23}
\end{pmatrix} =
\begin{pmatrix}
0.8 & -4.4 \\
-2.2 & -0.8 \\
1.8 & 0.2
\end{pmatrix}
.
\end{align}
Let \(z_1=X_{11}\), \(z_2=X_{21}\).

\end{example}
We use learning rates $0.1$ and we trained the model for $1,500$ epochs for when $p=1.75$, $2,000$ epochs for $p=2$, and $20,000$ epochs for $p=3$. As it was done in the previous experiment, we obtain the parameter history for each $p$, and compute the optimal token for the ~\eqref{eqn:w-svm} and ~\ref{eqn:vp:svm} problems.

\subsection{Architecture Details for Stanford Large Movie Review Classification}\label{app:exp:stanford}

The model architecture used for semantic analysis on the Stanford Large Movie Review dataset follows 
the transformer encoder described in \cite{vaswani2017attention}, with a linear classifier as the 
final layer.

The embedding layer has trainable token embeddings $E$ and position encodings $P$. The model's 
vocabulary size is $30{,}522$, with a maximum token length of $512$ and an embedding dimension of $384$. 
Hence, $E \in \mathbb{R}^{30522 \times 384}$ and $P \in \mathbb{R}^{512 \times 384}$. If a token $t$ 
is in position $i$, its embedding is $X_i = E_t + P_i$, where $E_t$ and $P_i$ denote the $t^\text{th}$ 
and $i^\text{th}$ rows of $E$ and $P$, respectively.

Next, the token features are passed through the encoding blocks, each consisting of a multi-head 
self-attention layer $\text{MultiHead}$, two layer-normalization layers ($\text{LayerNorm}_1$ and 
$\text{LayerNorm}_2$), and a Multilayer Perceptron (MLP). If the sequence of input token features 
is $X_1, \ldots, X_T$, and we denote $\text{MultiHead}(X_1, \ldots, X_T)_i$ as the $i^\text{th}$ 
token feature from the multi-head self-attention, then the output of the encoding block for the 
$i^\text{th}$ token is
\[
\text{LayerNorm}_2 \bigl( X_i' + \text{MLP}(X_i') \bigr),
\]
where
\[
X_i' = \text{LayerNorm}_1 \bigl( X_i + \text{MultiHead}(X_1, \ldots, X_T)_i \bigr).
\]
We apply a dropout rate of $0.2$ for regularization during training.

The $\text{MultiHead}$ attention is a variant of single-head attention that horizontally stacks multiple 
single-head attentions within the same layer. A single-head attention is equivalent to 
\eqref{eq-model-W}, but with the vector $z$ replaced by the matrix $X^\top$, and the vector $v$ 
replaced by the matrix $V$.

We experimented with 3 encoding blocks (each having 3 attention heads), 4 encoding blocks (with 
4 heads), and 6 encoding blocks (with 6 heads). Finally, we pass the feature vector of the first 
token from the last encoding layer to a linear classifier.

\subsection{Additional Experiment with Adam}\label{app:exp:adam}

The ViT architecture that we trained to compare $\ell_{1.1}$--\ref{eqn:md} and Adam used a patch size of $4$ on the input image, $512$ dimensional token feature, $6$ layers of attention blocks with $8$ attention heads per attention layer, and a one-layer linear classification layer for the final prediction head on the final [CLS] patch token feature.

The embedding layer of the architecture follows the work of \cite{dosovitskiy2020image}, where the embedding layer learns the [CLS] token embedding, a linear map for embedding each image patch, and a positional embedding for each possible position on the image. The details of how the attention layers are implemented are similar to that of the architecture used for the Stanford Large Movie dataset, just that the MLP sublayer is now a two-layer GeLU network, and that we apply the layer normalization before the multihead attention and the MLP instead of after the residual connection. Furthermore, we apply a dropout of $0.1$ when training the network.

\clearpage

\subsection{Addendum to the Attention Map Results}
%
\includegraphics[width=.9\linewidth]{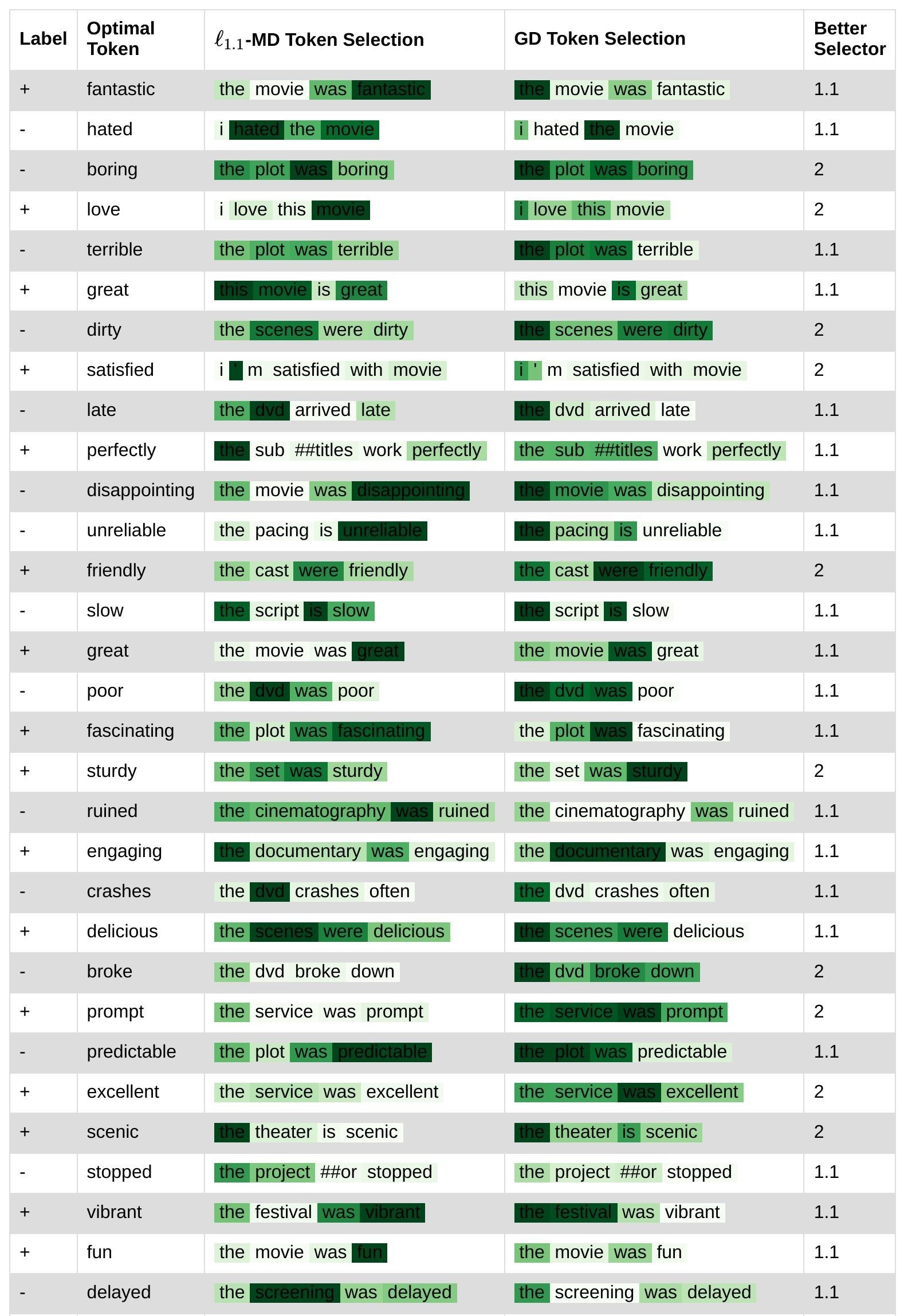}
\newpage
\begin{figure}
    \centering
    \includegraphics[width=1\linewidth]{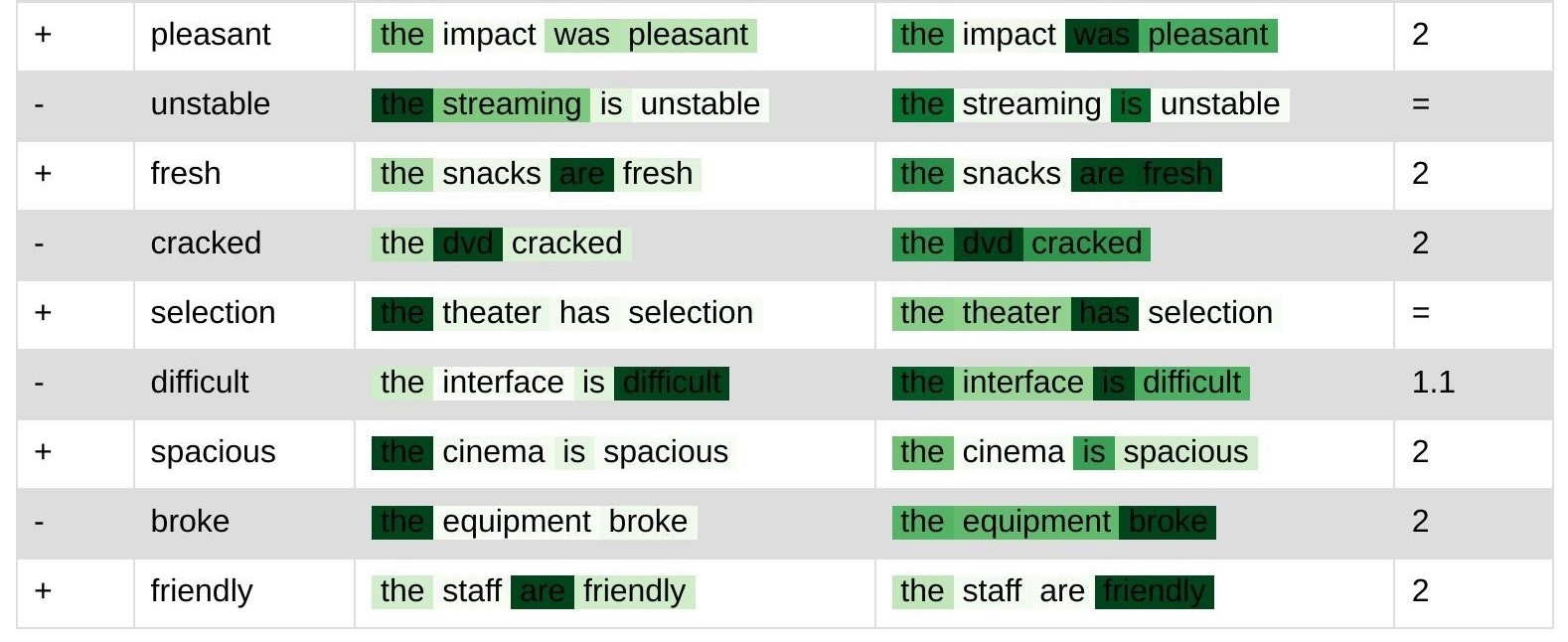}
    \caption{The full attention map table that shows that $\ell_{1.1}$-MD provides strictly more attention to the pivotal token compared to $\ell_2$-MD, or equivalently GD, for $22$ of the sample sentences. Out of the other $18$ sentences, for $16$ of which, GD strictly outperforms $\ell_{1.1}$-MD, while for the other $2$, the two algorithms are equally as good.}
\end{figure}

{
\clearpage
\newpage
\subsection{Additional Results on Attention Weight Sparsity and Runtime Performance} \label{app:add:sparsity}

\begin{figure}[h]
\centering
\begin{tabular}{cccc}
\subfloat[$W_K$ parameters with $\ell_{1.75}$-\ref{eqn:md}]{\includegraphics[width =0.31\linewidth]{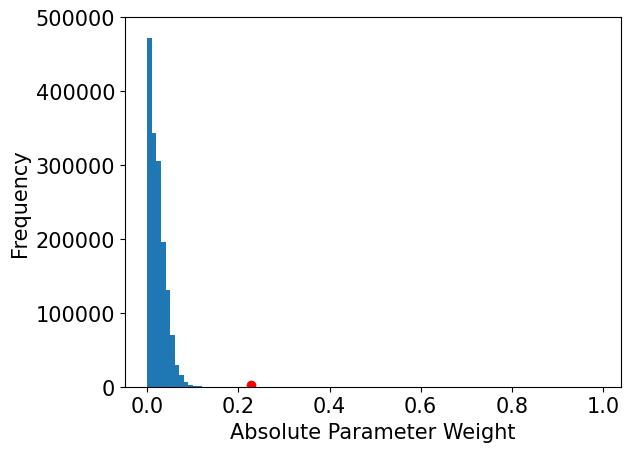}} 
\hspace{-.25cm}
&
\subfloat[$W_Q$ parameters with $\ell_{1.75}$-\ref{eqn:md}]{\includegraphics[width = 0.31\linewidth]{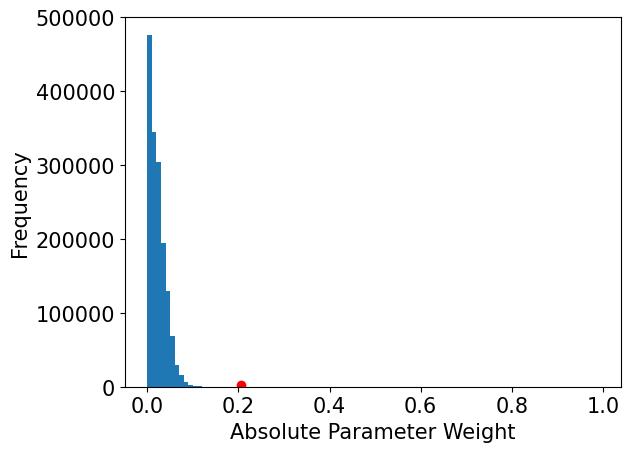}}
\hspace{-.3cm}
&
\subfloat[$W_V$ parameters with $\ell_{1.75}$-\ref{eqn:md}]{\includegraphics[width = 0.3\linewidth]{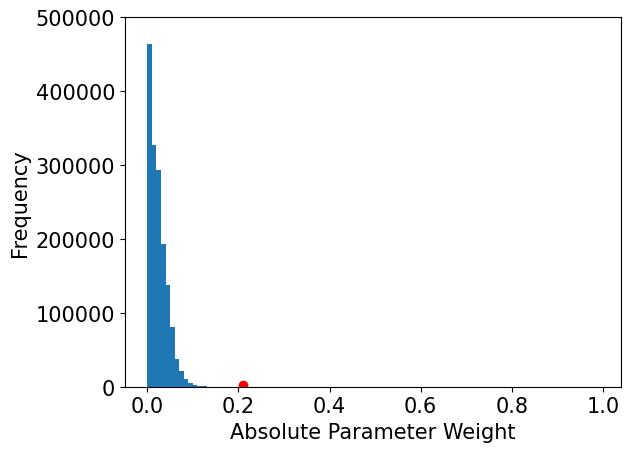}} \vspace{.3cm}\\
\subfloat[$W_K$ parameters with $\ell_2$-\ref{eqn:md}]{\includegraphics[width = 0.31\linewidth]{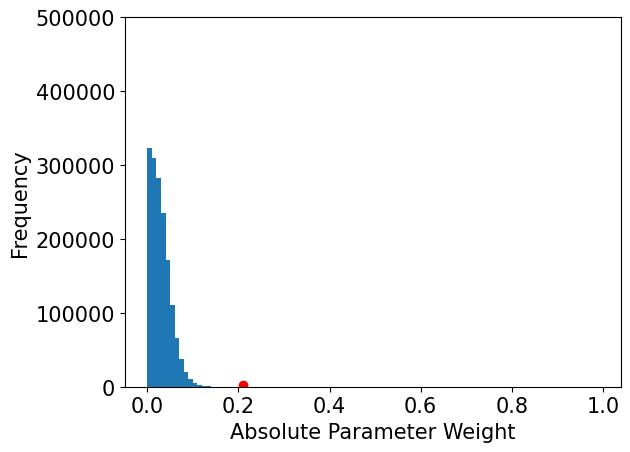}} &
\subfloat[$W_Q$ parameters with $\ell_2$-\ref{eqn:md}]{\includegraphics[width = 0.31\linewidth]{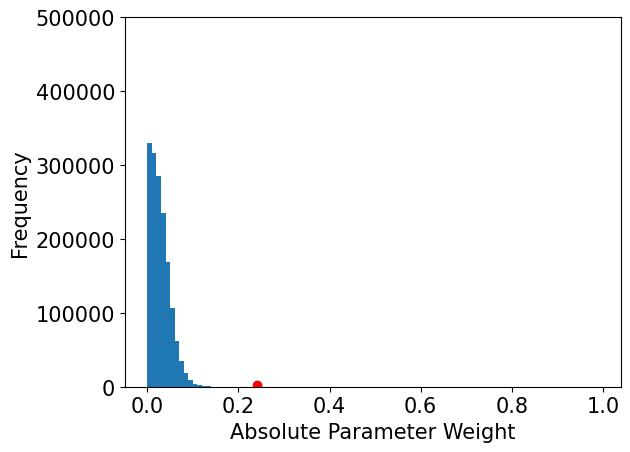}} &
\subfloat[$W_V$ parameters with $\ell_2$-\ref{eqn:md}]{\includegraphics[width = 0.3\linewidth]{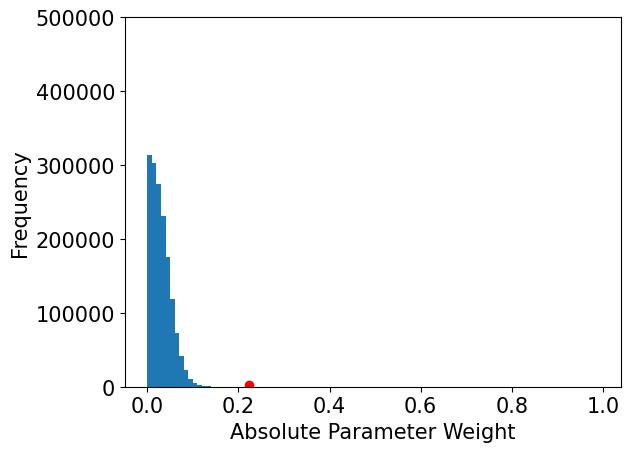}} \vspace{.3cm}\\
\subfloat[$W_K$ parameters with $\ell_3$-\ref{eqn:md}]{\includegraphics[width = 0.31\linewidth]{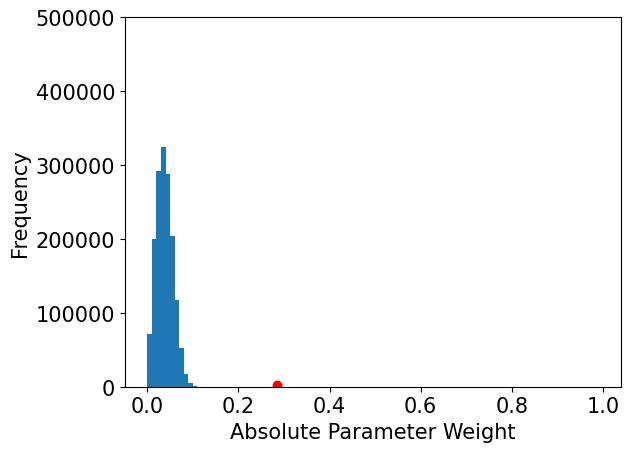}} &
\subfloat[$W_Q$ parameters with $\ell_3$-\ref{eqn:md}]{\includegraphics[width = 0.31\linewidth]{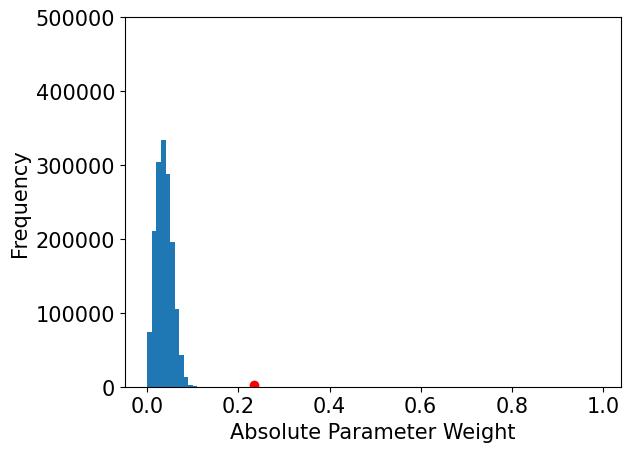}} &
\subfloat[$W_V$ parameters with $\ell_3$-\ref{eqn:md}]{\includegraphics[width = 0.3\linewidth]{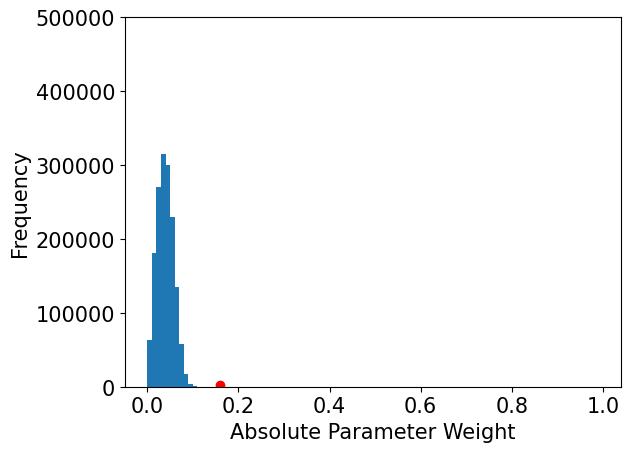}} \\
\end{tabular}
 \caption{Histogram of the absolute values of the entries in the $W_K,W_Q$, and $W_V$ attention weight matrices of the resulting models trained by $\ell_{1.75},\ell_2$, and $\ell_3$--\ref{eqn:md} on CIFAR-10.}
\label{fig:response-response-ii-hist-cifar10}
\end{figure}

\begin{figure}[h]
\centering
\begin{tabular}{cccc}
\subfloat[$W_K$ parameters with $\ell_{1.75}$-\ref{eqn:md}]{\includegraphics[width =0.31\linewidth]{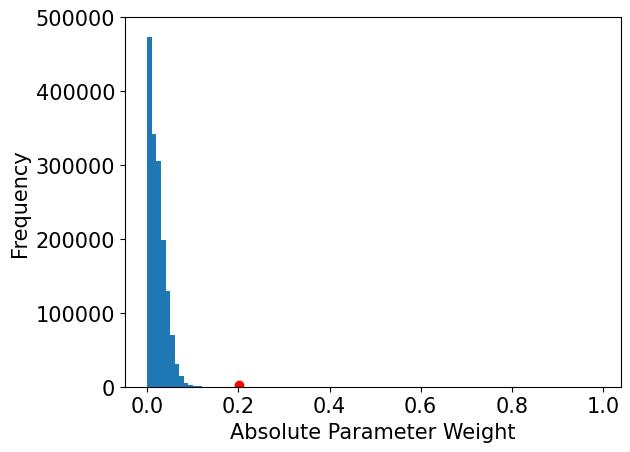}} 
\hspace{-.25cm}
&
\subfloat[$W_Q$ parameters with $\ell_{1.75}$-\ref{eqn:md}]{\includegraphics[width = 0.31\linewidth]{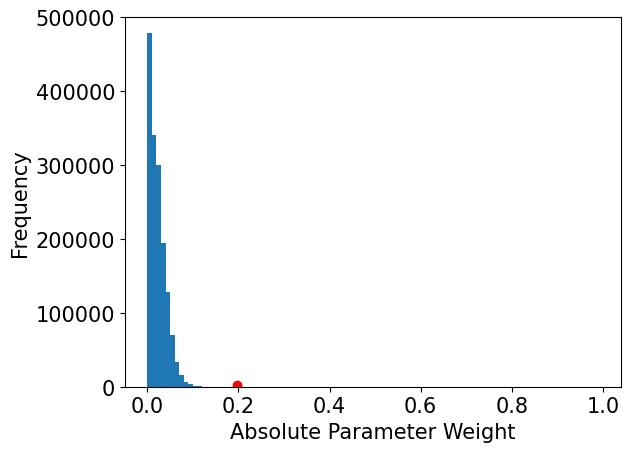}}
\hspace{-.3cm}
&
\subfloat[$W_V$ parameters with $\ell_{1.75}$-\ref{eqn:md}]{\includegraphics[width = 0.3\linewidth]{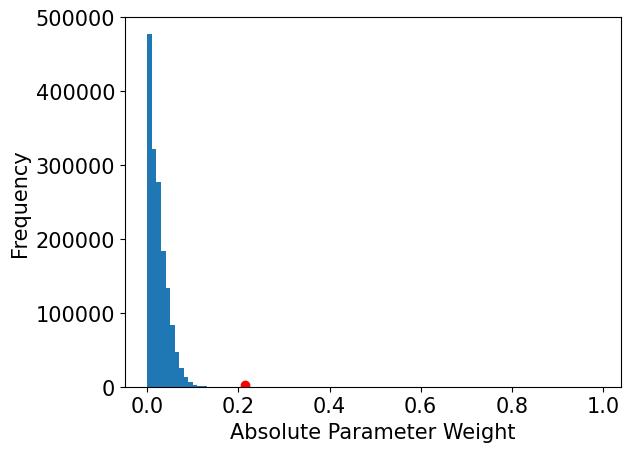}} \vspace{.3cm}\\
\subfloat[$W_K$ parameters with $\ell_2$-\ref{eqn:md}]{\includegraphics[width = 0.31\linewidth]{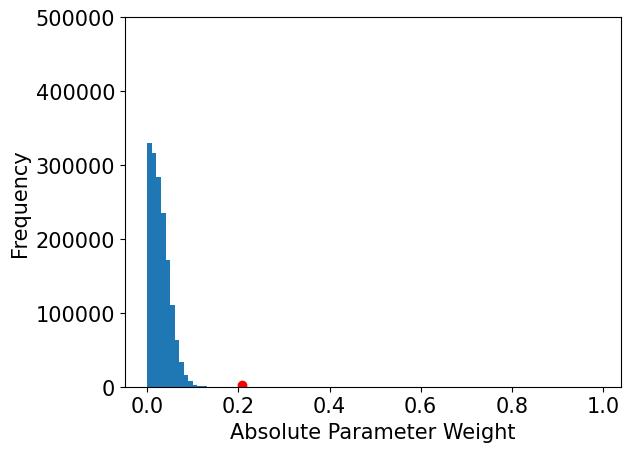}} &
\subfloat[$W_Q$ parameters with $\ell_2$-\ref{eqn:md}]{\includegraphics[width = 0.31\linewidth]{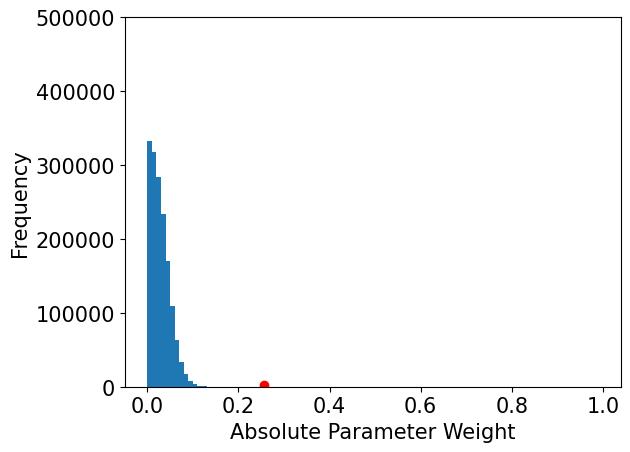}} &
\subfloat[$W_V$ parameters with $\ell_2$-\ref{eqn:md}]{\includegraphics[width = 0.3\linewidth]{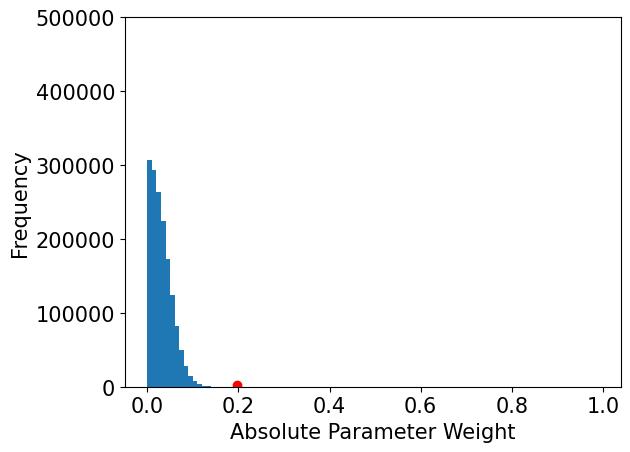}} \vspace{.3cm}\\
\subfloat[$W_K$ parameters with $\ell_3$-\ref{eqn:md}]{\includegraphics[width = 0.31\linewidth]{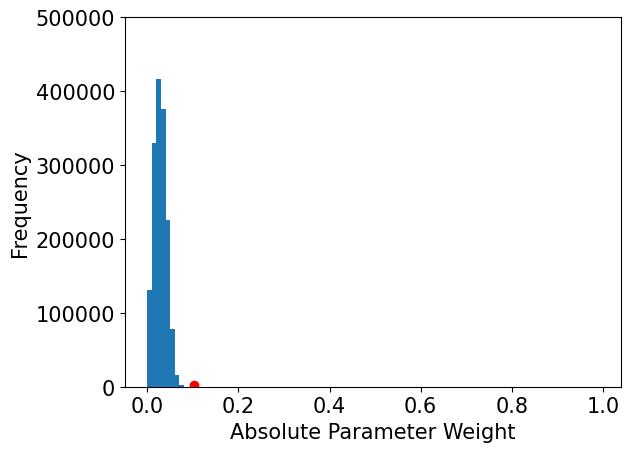}} &
\subfloat[$W_Q$ parameters with $\ell_3$-\ref{eqn:md}]{\includegraphics[width = 0.31\linewidth]{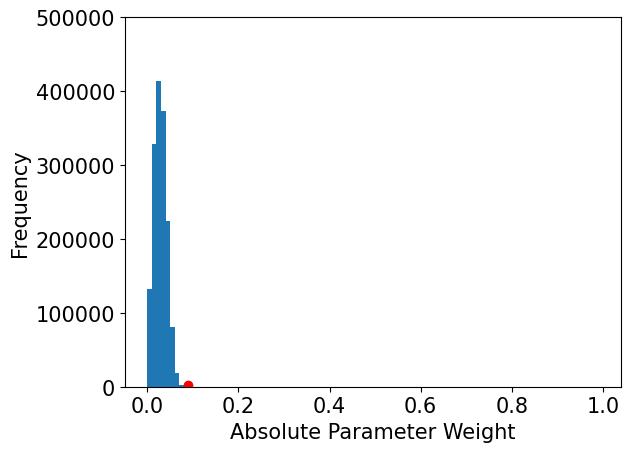}} &
\subfloat[$W_V$ parameters with $\ell_3$-\ref{eqn:md}]{\includegraphics[width = 0.3\linewidth]{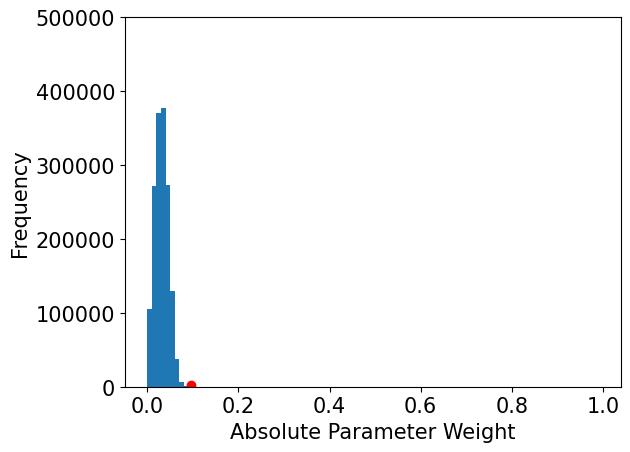}} \\
\end{tabular}
 \caption{Histogram of the absolute values of the entries in the $W_K,W_Q$, and $W_V$ attention weight matrices of the resulting models trained by $\ell_{1.75},\ell_2$, and $\ell_3$--\ref{eqn:md} on CIFAR-100.}
\label{fig:response-response-ii-hist-cifar100}
\end{figure}

Similar to the histogram plots for $\ell_{1.1}$--\ref{eqn:md} and Adam in Figures \ref{figure:hist-adam} and \ref{fig:response-sec4-hist} of Section 4.2.2, we plot the histograms for $\ell_{1.75},\ell_2,$ and $\ell_3$--\ref{eqn:md} below. As we can see in Figure \ref{fig:response-response-ii-hist-cifar10} and \ref{fig:response-response-ii-hist-cifar100}, the $\ell_{1.75}$--\ref{eqn:md} histogram shows a sparser weight distribution than that of Adam and $\ell_2$--\ref{eqn:md}, while the $\ell_3$--\ref{eqn:md} histogram shows a denser distribution.


\clearpage
\newpage
\begin{figure}[h]
     \centering
     \begin{subfigure}[b]{0.49\textwidth}
         \centering
         \includegraphics[width=\textwidth]{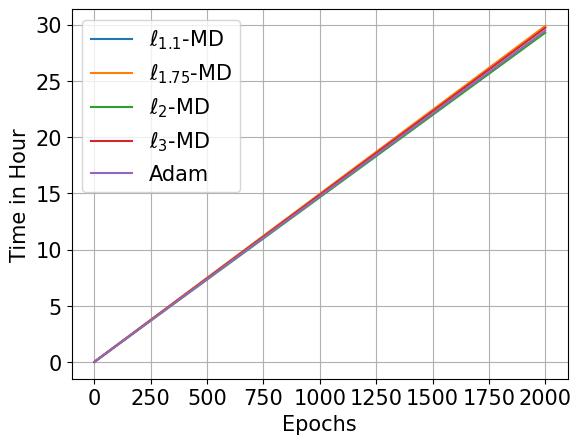}
         \caption{CIFAR-10 Training Time in Hour}
         \label{fig:response-ii-time-taken:cifar-10}
     \end{subfigure}
     \hfill
     \begin{subfigure}[b]{0.49\textwidth}
         \centering
         \includegraphics[width=\textwidth]{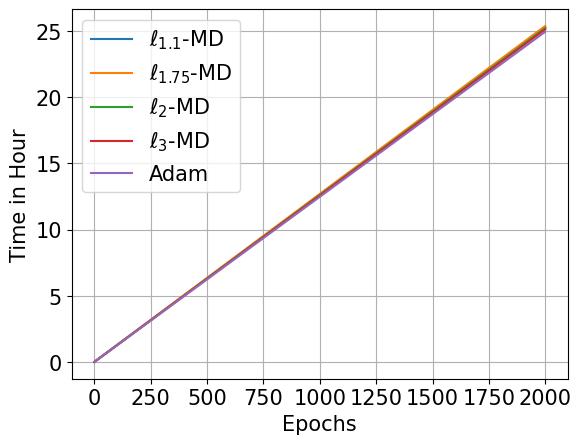}
         \caption{CIFAR-100 Training Time in Hour}
         \label{fig:response-ii-time-taken:cifar-100}
     \end{subfigure}
    \caption{The cumulative time taken to train ViT on CIFAR-10 (left) and CIFAR-100 (right) using $\ell_p$-\ref{eqn:md} for $p=1.1, 1.75, 2,$ and $3$ and Adam algorithms across the 2000 epochs. On the left subgraph, it took approximately 29 hours for each of the five optimizers to fully train the model on the CIFAR-10 dataset, and on the right subgraph, it took approximately 25 hours for each of the five optimizers to fully train the model on the CIFAR-100 dataset.}
    \label{fig:response-ii-time-taken}
\end{figure}

Next, we also take the cumulative time taken to perform each training steps for the experiments in Section 4.2.2. Figure \ref{fig:response-ii-time-taken} shows that for training CIFAR-10 and CIFAR-100, the differences in the training time between the five different training algorithms are very marginal. Therefore, our algorithms do not have any significant slowdown in the training speed.

}

\endgroup

\vskip 0.2in

\end{document}